\documentclass{article}
\usepackage[UKenglish]{babel}
\usepackage{csquotes}
\usepackage{adjustbox}

\usepackage{xparse}

\usepackage{luatex85}   

\usepackage{microtype}
\usepackage[shortlabels,inline]{enumitem}
\usepackage{graphicx}
\usepackage[table]{xcolor}
\usepackage{caption}
\usepackage{subcaption}
\usepackage{booktabs} 
\usepackage{multirow}
\newcolumntype{L}[1]{>{\raggedright\let\newline\\\arraybackslash\hspace{0pt}}m{#1}}
\newcolumntype{C}[1]{>{\centering\let\newline\\\arraybackslash\hspace{0pt}}m{#1}}
\newcolumntype{R}[1]{>{\raggedleft\let\newline\\\arraybackslash\hspace{0pt}}m{#1}}

\usepackage{tikz}
\usetikzlibrary{patterns}
\usetikzlibrary{patterns.meta}
\usetikzlibrary{shapes.misc}
\usepackage{pgfplots}
\pgfplotsset{compat=1.18}
\usepgfplotslibrary{groupplots}
\usepgfplotslibrary{fillbetween}

\usetikzlibrary{external}

\usepackage{algorithm}
\usepackage{algorithmic}

\usepackage{hyperref}

\renewcommand{\algorithmicrequire}{\textbf{Input:}}
\renewcommand{\algorithmicensure}{\textbf{Output:}}

 \usepackage[accepted]{icml2026}

\usepackage{mathtools}
\usepackage{amsmath}
\usepackage{amssymb}
\usepackage{amsthm}
\usepackage{bm}  

\renewcommand{\algorithmiccomment}[1]{\textit{ // #1}}

\definecolor{bettergreen}{RGB}{0,127,93}  
\definecolor{betterred}{RGB}{165,71,0}  
\definecolor{betterorange}{RGB}{178,121,0}  

\definecolor{CatA}{RGB}{27,158,119}
\definecolor{CatB}{RGB}{217,95,2}
\definecolor{CatC}{RGB}{117,112,179}
\definecolor{CatD}{RGB}{231,41,138}
\definecolor{CatE}{RGB}{102,166,30}
\definecolor{CatF}{RGB}{230,171,2}
\definecolor{CatG}{RGB}{166,118,29}
\definecolor{CatH}{RGB}{102,102,102}
\definecolor{CatAOne}{rgb}{0.0117, 0.1914, 0.5938}
\definecolor{CatATwo}{rgb}{0.125, 0.3906, 0.6797}
\definecolor{CatAThree}{rgb}{0.1953, 0.5703, 0.7578}
\definecolor{CatAFour}{rgb}{0.3633, 0.7539, 0.8242}
\definecolor{CatBFour}{rgb}{0.7812, 0.7031, 0.332}
\definecolor{CatBThree}{rgb}{0.6875, 0.5, 0.168}
\definecolor{CatBTwo}{rgb}{0.5977, 0.3203, 0.082}
\definecolor{CatBOne}{rgb}{0.4923, 0.0908, 0.0001}

\NewDocumentCommand{\naive}{}{na\"ive}

\NewDocumentCommand{\AlgorithmName}{m}{\textrm{\textsc{#1}}}
\NewDocumentCommand{\OurAlgo}{}{\AlgorithmName{VeriSHAP}}
\NewDocumentCommand{\CROWN}{}{\AlgorithmName{CROWN}}
\NewDocumentCommand{\CROWNIBP}{}{\AlgorithmName{CROWN}-\AlgorithmName{IBP}}
\NewDocumentCommand{\AlphaCROWN}{}{\(\alpha\)-\AlgorithmName{CROWN}}
\NewDocumentCommand{\IBP}{}{\AlgorithmName{IBP}}
\NewDocumentCommand{\LBP}{}{\AlgorithmName{LBP}}
\NewDocumentCommand{\SelectMaxDiam}{}{\AlgorithmName{MaxDiam}}
\NewDocumentCommand{\SelectMinDiam}{}{\AlgorithmName{MinDiam}}

\NewDocumentCommand{\SplitInOrder}{}{\AlgorithmName{InOrder}}
\NewDocumentCommand{\SplitStrongBranching}{}{\AlgorithmName{StrongBranching}}
\NewDocumentCommand{\SplitSmartBranchingIBP}{}{\AlgorithmName{IBPStrongBranching}}
\NewDocumentCommand{\SplitSmears}{}{\AlgorithmName{Smears}}
\NewDocumentCommand{\ExactSHAP}{}{\AlgorithmName{ExactSHAP}}
\NewDocumentCommand{\KernelSHAP}{}{\AlgorithmName{KernelSHAP}}

\NewDocumentCommand{\LeverageSHAP}{}{\AlgorithmName{LeverageSHAP}}
\NewDocumentCommand{\LinearMSR}{}{\AlgorithmName{LinearMSR}}
\NewDocumentCommand{\TreeMSR}{}{\AlgorithmName{TreeMSR}}
\NewDocumentCommand{\AdultDataset}{}{\texttt{Adult}}
\NewDocumentCommand{\ObesityDataset}{}{\texttt{Obesity}}
\NewDocumentCommand{\GermanDataset}{}{\texttt{German}}
\NewDocumentCommand{\MushroomDataset}{}{\texttt{Mushroom}}
\NewDocumentCommand{\DefaultDataset}{}{\texttt{Default}}
\NewDocumentCommand{\HepatitisCDataset}{}{\texttt{HepatitisC}}
\NewDocumentCommand{\AutomobileDataset}{}{\texttt{Auto}}
\NewDocumentCommand{\SteelDataset}{}{\texttt{Steel}}
\NewDocumentCommand{\AnnealingDataset}{}{\texttt{Annealing}}
\NewDocumentCommand{\SonarDataset}{}{\texttt{Sonar}}

\NewDocumentCommand{\BreastCancerDataset}{}{\texttt{BreastCancer}}
\NewDocumentCommand{\LungCancerDataset}{}{\texttt{LungCancer}}

\NewDocumentCommand{\OnlineNewsDataset}{}{\texttt{Online}~\texttt{News}}
\NewDocumentCommand{\MNISTDataset}{o}{\texttt{MNIST}\IfNoValueF{#1}{~\texttt{#1x#1}}}
\NewDocumentCommand{\FashionMNISTDataset}{}{\texttt{FashionMNIST}}

\NewDocumentCommand{\CIFARTenDataset}{}{\texttt{CIFAR10}}
\NewDocumentCommand{\GTSRBDataset}{}{\texttt{GTSRB}}
\newcommand\transp[1]{{
    #1^{\mkern-1.5mu\mathsf{T}}
}}
\newcommand{\overbar}[1]{\mkern 4mu\overline{\mkern-4mu#1\mkern-4mu}\mkern 4mu}
\newcommand{\Reals}{\mathbb{R}}

\newcommand{\RealsNonNeg}{\mathbb{R}_{\geq{} 0}}

\newcommand{\Nats}{\mathbb{N}}

\newcommand{\Bools}{{\{0,1\}}}

\renewcommand{\equiv}{\Leftrightarrow}
\DeclareMathOperator*{\argmax}{arg\,max}
\DeclareMathOperator*{\argmin}{arg\,min}

\renewcommand{\vec}[1]{\mathbf{#1}}

\newcommand{\mat}[1]{\mathbf{#1}}

\newcommand{\powerset}[1]{\mathcal{P}({#1})}
\newcommand{\upto}[1]{[{#1}]}
\NewDocumentCommand{\shapval}{O{i}}{\varphi_{#1}}
\newcommand{\val}{v}
\NewDocumentCommand{\otherfeatures}{O{i}}{\mathcal{S}_{#1}}
\NewDocumentCommand{\assign}{m m m o}{{({#1}_{#2};{#3}_{\Bar{#2}})}_{\IfNoValueF{#4}{#4}}}
\NewDocumentCommand{\NN}{}{f}
\NewDocumentCommand{\layer}{m}{f^{(#1)}}
\NewDocumentCommand{\x}{o}{\vec{x}_{\IfNoValueF{#1}{#1}}}
\NewDocumentCommand{\baseline}{o}{\vec{y}_{\IfNoValueF{#1}{#1}}}
\NewDocumentCommand{\bgdata}{o}{\baseline[#1]}
\NewDocumentCommand{\lb}{m}{\underline{#1}}
\NewDocumentCommand{\ub}{m}{\overline{#1}}
\NewDocumentCommand{\coaliw}{}{\lambda}
\NewDocumentCommand{\contrib}{}{\Delta_i}
\NewDocumentCommand{\maskcontrib}{}{\boldsymbol{\Delta}_i}
\NewDocumentCommand{\maskval}{}{\mathbf{v}}
\NewDocumentCommand{\branches}{o}{\mathsf{B}^{\IfNoValueF{#1}{(#1)}}}

\NewDocumentCommand{\branch}{o o}{\mathcal{B}_{\IfNoValueF{#1}{#1}}^{\IfNoValueF{#2}{(#2)}}}
\NewDocumentCommand{\includeset}{o o}{\mathcal{I}_{\IfNoValueF{#1}{#1}}^{\IfNoValueF{#2}{(#2)}}}
\NewDocumentCommand{\excludeset}{o o}{\mathcal{E}_{\IfNoValueF{#1}{#1}}^{\IfNoValueF{#2}{(#2)}}}
\NewDocumentCommand{\sumcoaliw}{m}{\Lambda_{#1}}

\NewDocumentCommand{\branchsvlb}{o}{\lb{p}_{\IfNoValueF{#1}{#1}}}
\NewDocumentCommand{\branchsvub}{o}{\ub{p}_{\IfNoValueF{#1}{#1}}}
\NewDocumentCommand{\mask}{o}{\vec{m}_{\IfNoValueF{#1}{#1}}}
\NewDocumentCommand{\contribgrad}{o}{\nabla_{\IfNoValueF{#1}{#1}}\contrib}
\NewDocumentCommand{\splittable}{}{\mathcal{J}}
\NewDocumentCommand{\z}{}{\vec{z}}%
\NewDocumentCommand{\lbz}{}{\lb{\z}}%
\NewDocumentCommand{\ubz}{}{\ub{\z}}%
\NewDocumentCommand{\branchone}{o}{\mathcal{B}'_{\IfNoValueF{#1}{#1}}}
\NewDocumentCommand{\branchtwo}{o}{\mathcal{B}''_{\IfNoValueF{#1}{#1}}}
\NewDocumentCommand{\includeone}{o}{\mathcal{I}'_{\IfNoValueF{#1}{#1}}}
\NewDocumentCommand{\excludeone}{o}{\mathcal{E}'_{\IfNoValueF{#1}{#1}}}
\NewDocumentCommand{\includetwo}{o}{\mathcal{I}''_{\IfNoValueF{#1}{#1}}}
\NewDocumentCommand{\excludetwo}{o}{\mathcal{E}''_{\IfNoValueF{#1}{#1}}}

\newcommand{\AlgoBP}{\AlgorithmName{BoundPropagation}}
\newcommand{\AlgoSelect}{\AlgorithmName{Select}}
\newcommand{\AlgoSplit}{\AlgorithmName{Split}}
\newcommand{\AlgoRefine}{\AlgorithmName{Refine}}
\newcommand{\AlgoAssemble}{\AlgorithmName{Assemble}}
\newcommand{\maskmat}{\mat{M}}
\newcommand{\sumcoaliws}{\boldsymbol{\Lambda}}

\tikzset{%
    dot diameter/.store in=\dot@diameter,
    dot diameter=3pt,
    dot spacing/.store in=\dot@spacing,
    dot spacing=7.5pt,
    dots/.style={
        line width=\dot@diameter,
        line cap=round,
        dash pattern=on 0pt off \dot@spacing
    },
    long dashes/.style={
        dash pattern=on 8pt off 4pt,
        dash phase=2pt,
    },
    tiny dashes/.style={
        dash pattern=on 2pt off 1pt,
    },
    FirstSHAPLine/.style = {
        CatA, very thick, mark=none
    },
    FirstSHAPArea/.style = {
        CatA, 
        opacity=0.7
    },
    FirstTimeMarker/.style = {
        rounded rectangle, 
        draw=CatA, fill=CatA!60!white, text=black,
        inner sep=2pt,
        font=\footnotesize,
    },
    SecondSHAP/.style = {
        CatB, solid, line width=2.5pt
    },
    SecondSHAPLine/.style = {
        CatB, very thick, mark=none
    },
    SecondSHAPArea/.style = {
        CatB, 
        pattern={Lines[angle=45,distance={6pt/sqrt(2)},line width=3pt]}, 
        pattern color=CatB, 
        opacity=0.7
    },
    ThirdSHAP/.style = {
        CatE, long dashes, line width=3pt
    },
    ThirdSHAPLine/.style = {
        CatE, very thick, mark=none
    },
    ThirdSHAPArea/.style = {
        CatE, 
        pattern={Lines[angle=0,distance={5pt/sqrt(2)},line width=2pt]}, 
        pattern color=CatE, 
        opacity=0.7
    },
    FourthSHAP/.style = {
        CatC, dots, line width=4.5pt, mark=none
    },
    FourthSHAPLine/.style = {
        CatC, very thick, mark=none
    },
    FourthSHAPArea/.style = {
        CatC, 
        pattern={Dots[angle=15,distance={6pt/sqrt(2)},radius=1.9pt]}, 
        pattern color=CatC, 
        opacity=0.7
    },
    FifthSHAP/.style = {
        CatD, tiny dashes, line width=4.5pt, mark=none
    },
    FifthSHAPLine/.style = {
        CatD, very thick, mark=none
    },
    FifthSHAPArea/.style = {
        CatD, 
        pattern={Hatch[angle=45,distance={8pt/sqrt(2)},line width=2pt]}, 
        pattern color=CatD, 
        opacity=0.7
    },
}%

\newcommand{\legendcolorbox}[1]{{%
  \begin{tikzpicture}%
    \draw[fill,#1,very thin] (0cm,-0.1cm) rectangle (0.2cm,0.1cm);%
  \end{tikzpicture}%
}}
\newcommand{\legendcolorline}[1]{{%
  \scalebox{.75}{\begin{tikzpicture}%
    \draw[draw,#1] (-.25cm,.2cm) -- (.25cm,0cm);%
  \end{tikzpicture}}%
}}
\newcommand{\legendcolorlineThick}[1]{{%
  \scalebox{.75}{\begin{tikzpicture}%
    \draw[draw,line width=2.5pt,#1] (-.25cm,-.2cm) -- (.25cm,0cm);%
  \end{tikzpicture}}%
}}
\newcommand{\legendsplitbox}[1]{{%
  \begin{tikzpicture}[baseline={([yshift={-2.1mm}]current bounding box.north)}]%
    \draw[#1, fill=#1, fill opacity=.6, very thick, rounded corners=1, 
    ] (0cm,-0.1cm) rectangle (0.2cm,0.1cm);
  \end{tikzpicture}%
}}
\newcommand{\legendmesh}[1]{{%
  \begin{tikzpicture}[
    baseline={([yshift={-2.1mm}]current bounding box.north)}, scale=0.145,
    trim left=.5mm,
  ]%
    \begin{axis}[
      width=4cm,
      axis lines=none,
      xtick=\empty,ytick=\empty,ztick=\empty,
      clip=false,
    ]
      \addplot3 [mesh, ultra thin, #1] coordinates {
        (0.0,0.0,0.7) (0.0,0.5,0.75) (0.0,1.0,1.0)

        (0.5,0.0,0.25) (0.5,0.5,0.2) (0.5,1.0,0.75)

        (1.0,0.0,0.0) (1.0,0.5,0.25) (1.0,1.0,0.7)
      };
    \end{axis}
  \end{tikzpicture}%
}}

\usepackage[capitalize,noabbrev]{cleveref}
\tikzifexternalizing{%
    \icmlshowauthorstrue
}{}

\theoremstyle{plain}
\newtheorem{theorem}{Theorem}[section]
\newtheorem{proposition}[theorem]{Proposition}

\newtheorem{corollary}[theorem]{Corollary}
\theoremstyle{definition}

\newtheorem{example}[theorem]{Example}
\theoremstyle{remark}

\newtheoremstyle{repeated}
    {\topsep}{\topsep}              
    {\itshape}                      
    {}                              
    {}                              
    {.}                             
    {\topsep}                       
    {\thmnote{\bfseries#3} \textnormal{(Restated)}}
\theoremstyle{repeated}
\newtheorem{repeatedtheory}{}

\icmltitlerunning{Verified SHAP: Provable Bounds for Exact Shapley Values of Neural Networks}

\begin{document}

\twocolumn[
  \icmltitle{Verified SHAP: \\ Provable Bounds for Exact Shapley Values of Neural Networks}



  \icmlsetsymbol{equal}{*}

\tikzifexternalizing{}{%
  \begin{icmlauthorlist}
    \icmlauthor{David Boetius}{equal,kn}
    \icmlauthor{Shahaf Bassan}{equal,huj}
    \icmlauthor{Guy Katz}{huj}
    \icmlauthor{Stefan Leue}{kn}
    \icmlauthor{Tobias Sutter}{stg}
  \end{icmlauthorlist}
}%

  \icmlaffiliation{kn}{University of Konstanz, Konstanz, Germany}
  \icmlaffiliation{huj}{Hebrew University of Jerusalem, Jerusalem, Israel}
  \icmlaffiliation{stg}{University of St.Gallen, St.Gallen, Switzerland}

  \icmlcorrespondingauthor{David Boetius}{david.boetius@uni-konstanz.de}
  \icmlcorrespondingauthor{Shahaf Bassan}{shahaf.bassan@gmail.com}

  \icmlkeywords{Machine Learning, ICML, SHAP, XAI, Explainable AI, Neural Network, Deep Learning, Branch and Bound, Neural Network Verification}

  \vskip 0.3in
]



\printAffiliationsAndNotice{\icmlEqualContribution}

\begin{abstract}
 Shapley additive explanations (SHAP) are widely recognised as computationally intractable for neural networks, since they induce an exponential search space over the input features. In this work, we take a first step towards scaling exact SHAP computation to larger search spaces by introducing an algorithm that leverages recent advances in neural network verification to compute arbitrarily tight exact lower and upper bounds on SHAP values for neural networks, ultimately recovering the exact SHAP values. We demonstrate that our approach scales to orders of magnitude larger search spaces than state-of-the-art exact methods. This provides an important first step towards exact SHAP computation and establishes a principled cornerstone for evaluating statistical approximation methods on larger search spaces.

\end{abstract}

\section{Introduction}

Shapley additive explanations (SHAP) is a widely used post-hoc explainability method for attributing a machine learning (ML) model's predictions to its input features. However, a central limitation of SHAP is its high computational complexity. While \emph{exact} SHAP values can be computed efficiently for certain model classes, such as tree-based~\citep{lundberg2020explainable,mitchell2020gputreeshap,yang2021fast} and additive models~\citep{enoueninstashap}, they unfortunately become prohibitively expensive to compute for neural networks, where explanations are often most needed.
Concretely, exact SHAP computation induces an \emph{exponentially} large search space over feature subsets, rendering exact computation impractical for most neural networks.


Consequently, the literature primarily focuses on \emph{estimating} SHAP values of neural networks, using methods such as \KernelSHAP{}~\citep{lundberg2017unified,covert2021improving}, \LeverageSHAP{}~\citep{musco2025leverageshap}, \TreeMSR{}~\citep{witter2025regression}, \AlgorithmName{FastSHAP}~\citep{jethani2021fastshap}, and \AlgorithmName{DeepSHAP}~\citep{Chen2021}. 
However, these methods share two fundamental limitations: 
\vspace{-.05cm}
\begin{enumerate}[\itshape(i)]
\item by construction, they only provide \emph{approximations} of the true SHAP values, which can be inaccurate in challenging settings such as highly non-linear models or scenarios with strong feature interactions, and
\item they lack a principled \emph{evaluation} framework: since computing \enquote{ground-truth} exact SHAP values for neural networks is prohibitively expensive, estimators are evaluated on toy-sized neural networks or ML models where exact SHAP is feasible---both of which are settings that may not generalise to real-world neural networks.
\end{enumerate}

\paragraph{Our Contributions.} 
In this work, we present \emph{Verified SHAP} (\OurAlgo{}), the first algorithm to leverage recent advances in \emph{neural network verification}~\citep{wang2021beta,zhouscalable,brix2024fifth,kotha2023provably,ferraricomplete,zhou2025clipandverify} to compute \emph{exact} SHAP values and \emph{provable bounds} on exact SHAP values for neural networks, scaling to significantly larger search spaces than previously tractable by exact methods. 
Our contribution serves three complementary purposes: 
\vspace{-.05cm}
\begin{enumerate}[\itshape(i)] 
    \item it represents a first step towards scaling the exact computation of SHAP values for neural networks to larger search spaces, with the promise that continued progress in neural network verification will enable our approach to scale to yet larger search spaces; 
    \item it allows computing provable SHAP lower and upper bounds up to \emph{arbitrary precision}, enabling fast, yet meaningful and trustworthy explanatory insights; and
    \item it provides a principled \enquote{\emph{ground-truth}} explanation baseline for evaluating statistical approximation methods on larger search spaces, serving as an important cornerstone for SHAP evaluation in the explainable AI literature.
\end{enumerate}


\paragraph{From Neural Network Verification to Exact SHAP.} 
The problem of automatically verifying that a neural network satisfies specifications, such as adversarial robustness, has seen rapid progress in recent years, driven largely by branch-and-bound verification methods~\citep{bunel2020branch,wang2021beta}. 
Recent works~\citep[e.g.,][]{wu2023verix, bassanexplaining, izza2024distance} adapt these techniques to compute certain forms of explanations, but instead of targeting the widely used SHAP framework, they focus on robustness-based explanations that are well-suited for applying existing neural network verifiers.
In contrast, SHAP values present significant challenges for applying neural network verifiers, as SHAP is 
\begin{enumerate*}[\itshape(i)] 
  \item inherently \emph{discrete};
  \item \emph{probabilistic}, since SHAP is defined via expectations; and 
  \item involves a summation over an \emph{exponential} number of feature sets.
\end{enumerate*}
In this work, we present the first algorithmic framework that overcomes these challenges and enables computing exact SHAP values using techniques from neural network verification.

\paragraph{Our Algorithm.} 
Our \OurAlgo{} algorithm is based on an incremental branch-and-bound neural network verification approach.
\OurAlgo{} recursively partitions the feature space, computing upper and lower bounds on the SHAP value within each resulting region. 
The central challenges lie in deriving sufficiently tight bounds and in designing efficient splitting strategies. 
Conceptually, our approach can be viewed as exploiting a decomposition of the neural network's input space into near-linear regions, within which the model's output---and consequently the SHAP value---can be tightly bounded. 
Owing to the recursive structure of the branch-and-bound procedure, our algorithm can produce arbitrarily tight upper and lower bounds on the SHAP values of a model~$\NN$.
Eventually, it recovers the exact SHAP values.

\paragraph{Empirical Evaluation.} 
We conduct an extensive evaluation of our framework on both tabular and vision benchmarks, demonstrating that it scales to substantially larger search spaces than current state-of-the-art exact SHAP approaches. 
Furthermore, we compare our exact SHAP bounds with a broad range of statistical SHAP estimators, showing that our algorithm complements these estimators when high accuracy and precision are required.
Finally, we use \OurAlgo{} to evaluate SHAP estimators on neural networks for larger search spaces than were previously feasible.

Overall, our results highlight the promise of \OurAlgo{} as both a first step towards exact SHAP computation for neural networks and a central benchmark for evaluating SHAP estimators. 
Due to space constraints, we only include brief proof sketches for our theoretical results in the main paper and defer the complete proofs to \cref{appendix:proofs}.

\section{Preliminaries}
This section formally defines SHAP values and introduces bound propagation as used for neural network verification.

\paragraph{Notation.}
Let~\(S\) be a set and~\(n \in \Nats\).
We define~\(\upto{n} := \{1, \ldots, n\}\), denote the cardinality of~\(S\) as~\(|S|\), the power set of~\(S\) as~\(\powerset{S}\), and the all-zero and all-one vectors as~\(\vec{0}_n, \vec{1}_n \in \Reals^n\), respectively.
Calligraphic upper-case letters like~\(\mathcal{S} \subseteq \powerset{S}\) denote sets of subsets, while sans-serif upper-case letters like~\(\mathsf{S} \subseteq \powerset{\powerset{S}}\) denote families of sets of subsets.
For~\(\x, \z \in \Reals^n\) and~\(S \subseteq \upto{n}\), we define~\(\assign{\x}{S}{\z} := \vec{u}\), where~\(\vec{u}_i = \x[i]\) if~\(i \in S\) and~\(\vec{u}_i = \z_i\) if~\(i \in \Bar{S} = \upto{n} \setminus S\) for~\(i \in \upto{n}\).
Given~\(\lb{\x}, \ub{\x} \in \Reals^n\) with~\(\lb{\x} \leq \ub{\x}\),~\([\lb{\x}, \ub{\x}]\) denotes the hyper-box~\(\{\x \in \Reals^n \mid \lb{\x} \leq \x \leq \ub{\x}\}\).

\subsection{Shapley Additive Explanations (SHAP)}\label{sec:bg-shap}

Shapley Additive Explanations (SHAP) are a popular game-theoretic explainability technique for attributing an ML model's output to input features using the classic Shapley value framework. Given a black-box model~$\NN: \Reals^n\to\Reals^m$ (in our case, a neural network), and some specific input~$\x\in\Reals^n$, we wish to locally explain the prediction~$\NN(\x)$. Formally, the SHAP value attribution of feature~$i\in \upto{n}$ is
\begin{align}
    \shapval &:= \sum_{S \in \otherfeatures} \frac{|S|!(n - |S| - 1)!}{n!} (\val(S \cup \{i\}) - \val(S)) \nonumber \\
    &= \sum_{S \in \otherfeatures} \coaliw(|S|) \contrib(S),\label{eqn:shapley}
\end{align}
where~\(\otherfeatures := \powerset{[n] \setminus \{i\}}\),~\(\coaliw(k) := \frac{1}{n}{n-1 \choose k}^{-1}\) is the \emph{weight} of a \emph{coalition}~\(S \in \otherfeatures\) of size~\(|S| = k \in \{0, \ldots, n-1\}\),~$\val(S)$ is the \emph{value function}, measuring the contribution of~$S$, and
\begin{equation}
    \contrib(S) := \val(S \cup \{i\}) - \val(S)\label{eqn:contrib}
\end{equation}
is the \emph{marginal contribution} of~\(i\) to the coalition~\(S\). 
A central question that arises when using SHAP values is the definition of the value function~$\val(S)$~\citep{sundararajan20many}. 
In this work, we use the standard \emph{marginal} value function~$\val(S):=\mathbb{E}_{\z\sim\mathcal{D}}[\NN(\x[S];\z_{\Bar{S}})]$, where~$(\x[S];\z_{\Bar{S}})$ denotes a vector where the features in~$S$ are taken from the vector~$\x$ and those in~$\Bar{S} = \upto{n} \setminus S$ are taken from the vector~$\z$. 
Moreover,~$\mathcal{D}$ denotes the data distribution, which is typically a marginalisation over a background dataset sampled from the training dataset. 

Computing~\(\shapval\) exactly is~\#P-hard~\citep{van2022tractability}, a class viewed as even \enquote{harder} than NP-hard~\citep{arora2009computational}. However, this does not determine whether \(\shapval\) may sometimes be efficiently computable in practice.



\subsection{Neural Network Verification}\label{sec:bg-nnver}

Neural network verification seeks to automatically prove input-output properties of neural networks over computationally hard (e.g., NP-hard) problems. Given a neural network~$\NN:\Reals^n\to\Reals^m$, an input specification~$\psi_\mathrm{in}$ (e.g., interval bounds on input features), and an output specification~$\psi_\mathrm{out}(\NN(\x))$, a neural network verifier formally proves whether the output specification holds for all inputs satisfying the input specification.

\paragraph{Bound Propagation.}
Bound propagation is a key ingredient in state-of-the-art neural network verification methods. 
It operates by deriving lower and upper bounds~\(\lb{\NN}, \ub{\NN}\) on the output of a neural network~\(\NN\) over an input hyper-box~\(X := [\lb{\x}, \ub{\x}]\). 
These bounds can then be used to prove whether the output specification~$\psi_{out}(\NN(\x))$ holds for all~$\x\in X$.
Formally, the bounds satisfy~\(\lb{\NN} \leq \NN(\x) \leq \ub{\NN}\),~\(\forall \x \in X\).
The bound propagation methods we consider here also satisfy~\(\lb{\NN} = \ub{\NN}\) if~\(\lb{\x} = \ub{\x}\)~\citep{moore2009introduction,boetius2025solving}.
%

Bound propagation relies on the fact that neural networks can be defined as a composition of more fundamental functions~\(\NN = \layer{K} \circ \cdots \circ \layer{1}\), where~\(\layer{1}, \ldots, \layer{K}\) are the layers of the network.
The input bounds~$X$ are propagated layer by layer through the network, where each step obtains an outer approximation of the layer's function image. 
This is done until the output layer, where~\([\lb{\NN}, \ub{\NN}]\) is obtained.


\paragraph{Interval Bound Propagation.}
One of the simplest bound propagation techniques is interval bound propagation (\IBP)~\citep{moore2009introduction}, also known as interval arithmetic, which is based on a set of rules for computing bounds for each layer type.
For example, if~\(\layer{1}(\x) = \mat{W}\x + \vec{b}\),
\begin{align*}
  \lbz &= \max(0, \mat{W})\lb{\x} + \min(0, \mat{W})\ub{\x} + \vec{b} \leq \layer{1}(\x) \\
  \ubz &= \max(0, \mat{W})\ub{\x} + \min(0, \mat{W})\lb{\x} + \vec{b} \geq \layer{1}(\x),
\end{align*}
provides us with bounds on~\(\layer{1}(\x)\) for~\(\x \in X = [\lb{\x}, \ub{\x}]\), where~\(\min\) and~\(\max\) are applied element-wise.
We can now propagate~\([\lbz, \ubz]\) forward to compute bounds on the next layer.
For example, if~\(\layer{2}(\z) = \max(0, \z)\), we can use~\(\lbz' = \max(0, \lbz) \leq \max(0, \z) \leq \max(0, \ubz) = \ubz'\) to obtain bounds on~\((\layer{2} \circ \layer{1})(\x)\).
Repeating this propagation for~\(\layer{3}\) with~\([\lbz', \ubz']\) and, subsequently, the remaining~\(\layer{k}\), eventually yields bounds on~\(\NN(\x)\) for~\(\x \in X\).

\paragraph{Linear Bound Propagation.}
While \IBP{} propagates intervals from inputs to outputs, linear bound propagation~(\LBP) approaches~\citep{zhang2018crown,wang2018formal,singh2019deeppoly,xu2021alphacrown} propagate two linear functions from outputs to inputs through~\(\NN\).
While the linear functions are eventually converted to an interval~\([\lb{\NN}, \ub{\NN}]\), \LBP{} can significantly reduce the \emph{approximation error}~\(\ub{\NN} - \max_{\x \in X} \NN(\x)\) and~\(\min_{\x \in X} \NN(\x) - \lb{\NN}\) that arises both in \IBP{} and \LBP{}.
In particular, \LBP{} introduces no approximation error for compositions of linear functions.
We refer to~\citet{zhang2018crown} for a detailed introduction to the \LBP{} method \CROWN{}.


\section{Provable Bounds on Exact SHAP}\label{sec:method}

In this section, we present our approach for computing exact SHAP values of neural networks. Unlike linear models and decision trees, existing exact methods for neural networks rely on enumerating all coalitions~\(S \in \otherfeatures{}\) to evaluate \cref{eqn:shapley}. Our method instead iteratively refines a partition of~\(\otherfeatures{}\) to compute guaranteed bounds~\(\lb{\shapval} \leq \shapval \leq \ub{\shapval}\) using bound propagation. This allows us to compute tight bounds on~\(\shapval\) to any desired precision.



We now introduce our approach in more detail, starting with computing guaranteed SHAP bounds based on a partition.
Next, we discuss our partitioning scheme and prove that our scheme terminates after a finite number of iterations. 
Finally, we show that our algorithm terminates immediately for linear models and can terminate early for piecewise-linear models, such as ReLU-activated neural networks. For clarity, we call the elements of a partition \emph{branches} in line with common branch-and-bound terminology.

\paragraph{Bounds on~\(\shapval\).}
Let~\(\branches[t] := (\branch[1][t], \ldots, \branch[K][t])\) be a partition of~\(\otherfeatures\), so that~\(\branch[k][t] \subseteq \otherfeatures\),~\(\forall k \in \upto{K}\).
For each branch~\(\branch \in \branches[t]\), we can compute lower and upper bounds~\([\lb{\contrib}_{\branch}, \ub{\contrib}_{\branch}]\), such that~\(\lb{\contrib}_{\branch} \leq \contrib(S) \leq \ub{\contrib}_{\branch}\) for all~\(S \in \branch\), using bound propagation for neural networks on~\(\contrib(S)\) as defined in \cref{eqn:contrib}.
The practical aspects of this are discussed in \cref{sec:impl}.
%
Assuming we can also compute~\(\sumcoaliw{\branch} := \sum_{S \in \branch} \coaliw(|S|)\), we obtain the result below.
\NewDocumentCommand{\TheoremSoundness}{s}{%
  Let~\(\branches[t]\) be a partition of~\(\otherfeatures\).
  It holds that
  \begin{equation*}
      \lb{\shapval}^{(t)} 
      := \sum_{\branch \in \branches[t]\hspace{-.75em}} \sumcoaliw{\branch}\lb{\contrib}_{\branch} 
      \leq \shapval 
      \leq \sum_{\branch \in \branches[t]\hspace{-.75em}} \sumcoaliw{\branch}\ub{\contrib}_{\branch} 
      =: \ub{\shapval}^{\mathrlap{(t)}}.
  \end{equation*}
}
\begin{theorem}[SHAP Bounds]\label{theory:soundness}
    \TheoremSoundness*%
\end{theorem}

\emph{Proof Sketch.} We group coalitions~\(S\) into branches~\(\branch\).
By definition, each branch's total Shapley weight~\(\sumcoaliw{\branch}\) is exactly the sum of the coalition weights~\(\coaliw(|S|)\) inside it. 
Bound propagation ensures that each coalition's contribution is bounded by the contribution bounds~\([\lb{\contrib}_{\branch}, \ub{\contrib}_{\branch}]\) of the respective branch.
Finally, since the branches form a partition, summing~\(\sumcoaliw{\branch}\lb{\contrib}_{\branch}\) and~\(\sumcoaliw{\branch}\ub{\contrib}_{\branch}\) provides a lower, respectively, upper bound on the Shapley value~\(\shapval\).

\begin{figure}
    \centering
    \tikzexternalenable%
    \begin{subfigure}{.5\linewidth}
        \begin{tikzpicture}[trim left={-.5cm}]
            \begin{axis}[
              width=1.15\linewidth,
              font=\footnotesize,
              axis lines=left,
              xtick={0,1}, ytick={0,1}, 
              enlargelimits=0.5,
              xlabel={$\mask[{j_1}]$}, ylabel={$\mask[{j_2}]$},
              x label style={at={(axis description cs:0.975,0.0)},anchor=north},
              y label style={at={(axis description cs:0.02,0.96)},rotate=270,anchor=east},
            ]
              \addplot [black, mark=*, only marks] coordinates {
                  (0,0) (0,1)
              };
              \addplot [black, mark=*, only marks] coordinates {
                  (1,0) (1,1)
              };
              \node [label={[label distance=-.5mm]225:$\emptyset$}] at (0,0) {};
              \node [label={[label distance=-.5mm]315:$\hspace*{-2mm}\{j_1\}$}] at (1,0) {};
              \node [label={[label distance=-.5mm]135:$\{j_2\}\hspace*{-2mm}$}] at (0,1) {};
              \node [label={[label distance=-.5mm]45:$\hspace*{-6mm}\{j_1,j_2\}$}] at (1,1) {};
              
              \draw[CatB, fill=CatB, fill opacity=.6, ultra thick, rounded corners] (-.1,-.1) rectangle (1.1,.1);
              \draw[CatB, fill=CatB, fill opacity=.6, ultra thick, rounded corners] (-.1,0.9) rectangle (1.1,1.1);

              \draw[CatA, fill=CatA, fill opacity=.6, ultra thick, rounded corners] (-.1,-.1) rectangle (0.1,1.1);
              \draw[CatA, fill=CatA, fill opacity=.6, ultra thick, rounded corners] (0.9,-.1) rectangle (1.1,1.1);
            \end{axis}
        \end{tikzpicture}
    \end{subfigure}\hfill
    \begin{subfigure}{.5\linewidth}
        \centering
        \begin{tikzpicture}[trim left={-0cm}]
            \begin{axis}[
              width=1.4\linewidth,
              font=\footnotesize,
              xtick={.2,1.3}, ytick={1.4}, 
              extra y ticks={0.4},
              xticklabels={0,1},yticklabels={1},
              extra y tick label={0},
              xticklabel style={anchor=north},
              yticklabel style={anchor=east,xshift=1mm,yshift=.35mm},
              extra y tick style={anchor=east,tick label style={xshift=1.5mm,yshift=1.5mm}},
              xmin=0,xmax=1.6,ymin=0,ymax=1.8,
              zmin=0,zmax=1.6,
              ztick=\empty,
              axis lines=middle,
              clip=false,
              zlabel={$\contrib(S)$},
            ]
              \draw[
                CatB, fill=CatB, fill opacity=.6, ultra thick, rounded corners=2.35,
                rotate around={351:(.7,.4,0)}
              ] (.16,.32,0) rectangle (1.24,.48,0);
              \draw[
                CatB, fill=CatB, fill opacity=.6, ultra thick, rounded corners=2.35,
                rotate around={351:(.6,1.2)}
              ] (.13,1.37,0) rectangle (1.18,1.6,0);

              \draw[
                CatA, fill=CatA, fill opacity=.6, ultra thick, rounded corners=2.35,
                rotate around={306:(.0,.7,0)}
              ] (.555,-.42,0.54) rectangle (.66,-.42,1.52);
              \draw[
                CatA, fill=CatA, fill opacity=.6, ultra thick, rounded corners=2.35,
                rotate around={306:(.0,.7,0)}
              ] (.58,1.26,0.83) rectangle (.69,1.26,1.82);

              \draw plot[mark=*] coordinates {(0.2,0.4)};
              \draw plot[mark=*] coordinates {(0.2,1.4)};
              \draw plot[mark=*] coordinates {(1.2,0.4)};
              \draw plot[mark=*] coordinates {(1.2,1.4)};

              \addplot3 [
                mesh, black,
                shift={(.2,.4,.4)},
              ] coordinates {
(0.00,0.00,0.20) (0.00,0.25,0.30) (0.00,0.50,0.40) (0.00,0.75,0.50) (0.00,1.00,0.60)

(0.25,0.00,0.05) (0.25,0.25,0.30) (0.25,0.50,0.40) (0.25,0.75,0.40) (0.25,1.00,0.60)

(0.50,0.00,0.00) (0.50,0.25,0.15) (0.50,0.50,0.15) (0.50,0.75,0.15) (0.50,1.00,0.30)

(0.75,0.00,0.05) (0.75,0.25,0.20) (0.75,0.50,0.15) (0.75,0.75,0.15) (0.75,1.00,0.15)

(1.00,0.00,0.00) (1.00,0.25,0.00) (1.00,0.50,0.00) (1.00,0.75,0.00) (1.00,1.00,0.00)
              };

              \addplot3 [mesh, CatA, ultra thick, shift={(.2,.4,.4)}] coordinates {
(0.00,0.00,0.20) (0.00,0.25,0.30) (0.00,0.50,0.40) (0.00,0.75,0.50) (0.00,1.00,0.60)
              };
              \addplot3 [mesh, CatA, ultra thick, shift={(.2,.4,.4)}] coordinates {
(1.00,0.00,0.00) (1.00,0.25,0.00) (1.00,0.50,0.00) (1.00,0.75,0.00) (1.00,1.00,0.00)
              };
              \addplot3 [mesh, CatB, ultra thick, shift={(.2,.4,.4)}] coordinates {
(0.00,0.00,0.20) (0.25,0.00,0.05) (0.50,0.00,0.00) (0.75,0.00,0.05) (1.00,0.00,0.00)
              };
              \addplot3 [mesh, CatB, ultra thick, shift={(.2,.4,.4)}] coordinates {
(0.00,1.00,0.60) (0.25,1.00,0.60) (0.50,1.00,0.30) (0.75,1.00,0.15) (1.00,1.00,0.00)
              };
            \end{axis}
        \end{tikzpicture}
    \end{subfigure}
    \tikzexternaldisable%
    \caption[Splitting in VeriSHAP]{%
      \textbf{Splitting in \OurAlgo{}.}
      For \(\otherfeatures{} = \{j_1, j_2\}\), there are two options, \legendsplitbox{CatA} and~\legendsplitbox{CatB}, for splitting the initial partition (left).
      For the split~\legendsplitbox{CatA}, the branches are~\(\branch[1] = \{\emptyset, \{j_2\}\}\) and~\(\branch[2] = \{\{j_1\}, \{j_1, j_2\}\}\), defined by~\(\includeset[1] = \excludeset[2] = \emptyset\) and~\(\excludeset[1] = \includeset[2] = \{j_1\}\).
      This split allows computing the exact SHAP value directly since the contribution function~\(\contrib(S)\)~\legendmesh{black} is linear for this split (right).
    }%
    \label{fig:splitting-illustration}
\end{figure}
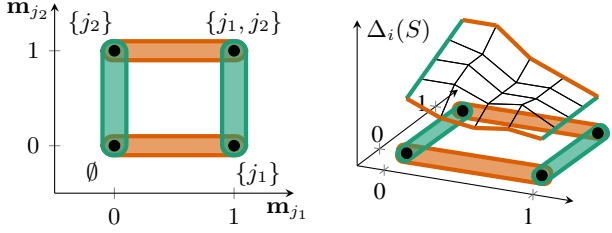

\paragraph{Partitioning~\(\otherfeatures\).}
We now describe how we create and refine the partition~\(\branches\).
We define the branches~\(\branch \in \branches\) in terms of sets of \emph{included} and \emph{excluded features}~\(\includeset, \excludeset \subseteq \upto{n} \setminus \{i\}\) so that~\(\branch := \{S \in \otherfeatures \mid \includeset{} \subseteq S, \excludeset{} \cap S = \emptyset\}\).
As an initial partition, we use the trivial partition~\(\branches[1] := (\branch[1][1])\) where~\(\branch[1][1]\) is defined by~\(\includeset[1][1] = \excludeset[1][1] = \emptyset\), so that~\(\branch[1][1] = \otherfeatures\).
Our refinement proceeds recursively.
The~\(t\)-th partition~\(\branches[t]\) is derived from the previous partition~\(\branches[t-1] = (\branch[1][t-1], \ldots, \branch[K][t-1])\) by selecting a branch~\(\branch[k] \in \branches[t-1]\) defined by~\(\includeset[k]\) and~\(\excludeset[k]\) and selecting a feature~\(j \in (\upto{n} \setminus \{i\}) \setminus (\includeset[k] \cup \excludeset[k])\).
Intuitively, we split~\(\branch[k]\) by putting all coalitions~\(S \in \branch[k]\) containing the selected feature~\(j\) into the first new branch~\(\branchone[k]\) and putting all~\(S \in \branch[k]\) not containing~\(j\) into the second new branch~\(\branchtwo[k]\).
Formally,~\(\branches[t] := (\branch[1][t-1], \ldots, \branchone[k], \branchtwo[k], \ldots, \branch[K][t-1])\), where~\(\branchone[k]\) is defined by~\(\includeone[k] := \includeset[k] \cup \{j\}\) and~\(\excludeone[k] := \excludeset[k]\), and~\(\branchtwo[k]\) is defined by~\(\includetwo[k] := \includeset[k]\) and~\(\excludetwo[k] := \excludeset[k] \cup \{j\}\).
\Cref{fig:splitting-illustration} illustrates this partitioning strategy.
We discuss strategies for selecting branches and features to split in \cref{sec:heuristics}.

A crucial assumption of our bounding approach is that we can compute~\(\sumcoaliw{\branch}\) efficiently.
As we show next, the partitioning approach described above enables this.
\NewDocumentCommand{\SumCoaliWProp}{}{
    Consider~\(\branch \in \branches[t]\) defined by~\(\includeset{}\) and~\(\excludeset{}\).
    Letting~\(r := |\includeset{}|\) and~\(s := |\includeset{}| + |\excludeset{}|\), we have~\(\sumcoaliw{\branch} = \sum_{S \in \branch} \coaliw(|S|) = (s + 1)^{-1}{s \choose r}^{\mathrlap{-1}}\).
}
\begin{proposition}[Closed-Form Expression for~\(\sumcoaliw{\branch}\)]\label{theory:sumcoaliw}
    \SumCoaliWProp%
\end{proposition}

\emph{Proof Sketch.}
The proposition is established through a sequence of combinatorial manipulations, leveraging a beta-function identity and the truncated-sum binomial theorem.

Furthermore, our refinement process eventually produces a partition for which~\(\shapval\) can be computed exactly.

\NewDocumentCommand{\TerminationTheorem}{}{%
  There exists~\(t \in \Nats\), such that~\(\lb{\shapval}^{(t)} = \ub{\shapval}^{(t)} = \shapval\).
}
\begin{theorem}[Termination]\label{theory:termination}
    \TerminationTheorem%
\end{theorem}

\emph{Proof Sketch.}
Since the number of features is finite, our refinement process eventually produces a partition where each branch~\(\branch\) contains a single set~\(S \in \otherfeatures\).
At this point, bound propagation computes~\(\lb{\contrib}_{\branch} = \ub{\contrib}_{\branch} = \contrib(S)\) so that the SHAP values are computed exactly.

As a direct consequence, our algorithm can also compute SHAP values up to arbitrary precision, which, in practice, can be faster than computing the exact SHAP values.
\NewDocumentCommand{\ArbitraryPrecisionCorollary}{}{%
    Given any precision level~$\delta \in \RealsNonNeg$, there exists~\(t \in \Nats\), such that~$\ub{\shapval}^{(t)} - \lb{\shapval}^{(t)} \leq \delta$. 
    This implies~$\lvert\hat{\shapval}-\shapval\rvert\leq \delta$ for every~\(\hat{\shapval} \in [\lb{\shapval}^{(t)}, \ub{\shapval}^{(t)}]\).
}
\begin{corollary}[Arbitrary Precision]\label{theory:arbitrary-precision}
    \ArbitraryPrecisionCorollary%
\end{corollary}

While \cref{theory:termination} proves the termination of our approach, the runtime required until termination can be exponential in the number of features, due to the theoretical complexity of computing SHAP values exactly~\citep{van2022tractability, arenas2023complexity}.
However, as we show next, our approach terminates in the first iteration for linear models.
This result implies that our approach can terminate early even for non-linear models, particularly when the model can be decomposed into linear regions, as is the case for ReLU neural networks, which are inherently piecewise-linear.

\NewDocumentCommand{\LinearModelProposition}{}{%
    If~\([\lb{\contrib}, \ub{\contrib}]\) are computed using an \LBP{} method and~\(\NN(\x) = \transp{\vec{w}}\x + \vec{b}\), it holds that~\(\lb{\shapval}^{(1)} = \ub{\shapval}^{(1)} = \shapval\).
}
\begin{proposition}[Linear Models]\label{theory:linear-shap}
    \LinearModelProposition%
\end{proposition}

\emph{Proof Sketch.}
We first prove that the contribution~\(\contrib(S)\) is constant in~\(S\) for linear models.
Since \LBP{} methods have no approximation error for compositions of linear functions~\citep{zhang2018crown}---such as the contribution of a linear model---\LBP{} computes tight bounds equal to the constant contribution for every branch. This provides the exact SHAP value in the first iteration for a linear model.

An equivalent argument proves that our approach terminates immediately if, after sufficient splitting,~\(\contrib\) becomes linear in each branch, as illustrated in \cref{fig:splitting-illustration}.
This can be the case, e.g., for piecewise-linear functions, such as ReLU-activated neural networks, when each branch lies within a single linear segment.
Specifically, the runtime of \OurAlgo{} is upper-bounded by~$\mathcal{O}(B\cdot T_{BP})$, where~$B$ denotes the number of explored branches and~$T_{BP}$ is the cost of bound propagation per branch.
The central question is how large~\(B\) becomes in practice: if branching requires only a few iterations to restrict the network to regions where the marginal contribution is linear, or sufficiently close to linear for the bound propagation to produce tight bounds,~\(B\) remains small and \OurAlgo{} is efficient in practice.

\Cref{sec:experiments} demonstrates that~\(B\) indeed remains comparatively small in practice, allowing our approach to outperform the existing state-of-the-art for computing exact SHAP values.
Before moving on to \cref{sec:experiments}, we first treat more practical aspects of implementing our approach.



\section{Efficient SHAP Bounding}\label{sec:impl}

While \cref{sec:method} introduces our approach at a conceptual level, this section fills in details for practically bounding SHAP values of neural networks and introduces our concrete \OurAlgo{} algorithm.
Specifically, we discuss
\begin{enumerate*}[\itshape(i)]
    \item leveraging bound propagation for computing bounds on~\(\contrib\);
    \item pruning branches to manage memory demands;
    \item batch-processing branches to leverage parallel hardware; 
    \item computing sums of coalition weights~\(\sumcoaliw{\branch}\) particularly efficiently;
    and 
    \item reusing branches for simultaneously bounding the SHAP values of all features.
\end{enumerate*}
\Cref{sec:heuristics} discusses splitting strategies for partitioning and \cref{appendix:algorithm-additional} further elaborates on selected aspects of \OurAlgo{}.

\Cref{algo:main} summarises \OurAlgo{}.
In \cref{algo:main}, \AlgoBP{} refers to a bound propagation technique as introduced in \cref{sec:bg-nnver}, such as \CROWN{}.
\AlgoAssemble{} and \AlgoRefine{} implement \cref{theory:multi-shap-bounds,theory:sum-coaliws-recursive}, respectively.
\AlgoSelect{} and \AlgoSplit{} are introduced in \cref{sec:heuristics}.
All remaining sub-procedures and variables are described in the following paragraphs.

\begin{algorithm}[tb]
  \caption{\OurAlgo{}}\label{algo:main}
  \begin{algorithmic}
    \REQUIRE Value function~\(\maskval: \Bools^n \to \Reals\), batch size~$b$
    \STATE $[\lb{\mask}, \ub{\mask}] \gets [\vec{0}_{n}, \vec{1}_{n}];$\quad $\sumcoaliw{} \gets 1$
    \STATE $[\lb{\maskval}, \ub{\maskval}] \gets \AlgoBP(\maskval, [\lb{\mask}, \ub{\mask}])$
    \STATE $[\lb{\boldsymbol{\varphi}}^{(1)}, \ub{\boldsymbol{\varphi}}^{(1)}] \gets \AlgoAssemble(\sumcoaliw{}, \lb{\maskval}, \ub{\maskval})$ \COMMENT{\cref{theory:multi-shap-bounds}}
    \STATE $\branches[1] \gets \{(\lb{\mask}, \ub{\mask}, \sumcoaliw{}, \lb{\maskval}, \ub{\maskval})\}$
    \FOR{$t \in \{2, \ldots, 2^n\}$}
      \STATE $(\lb{\maskmat},\ub{\maskmat},\sumcoaliws,\lb{\maskval},\ub{\maskval}), \textit{others} \gets \AlgoSelect(\branches[t-1], b)$
      \STATE $[\lb{\maskmat}',\ub{\maskmat}'] \gets \AlgoSplit(\lb{\maskmat},\ub{\maskmat})$
      \STATE $[\lb{\maskval}', \ub{\maskval}'] \gets \AlgoBP(\maskval, [\lb{\maskmat}', \ub{\maskmat}'])$
      \STATE $\sumcoaliws\!' \gets \AlgoRefine(\sumcoaliws)$ \COMMENT{\cref{theory:sum-coaliws-recursive}}
      \STATE $[\lb{\boldsymbol{\varphi}}, \ub{\boldsymbol{\varphi}}] \gets \AlgoAssemble(\sumcoaliws{}, \lb{\maskval}, \ub{\maskval})$
      \STATE $[\lb{\boldsymbol{\varphi}}', \ub{\boldsymbol{\varphi}}'] \gets \AlgoAssemble(\sumcoaliws{}', \lb{\maskval}', \ub{\maskval}')$
      \STATE $\lb{\boldsymbol{\varphi}}^{(t)} \gets \lb{\boldsymbol{\varphi}}^{(t-1)} - \lb{\boldsymbol{\varphi}} + \lb{\boldsymbol{\varphi}}'$
      \STATE $\ub{\boldsymbol{\varphi}}^{(t)} \gets \ub{\boldsymbol{\varphi}}^{(t-1)} - \ub{\boldsymbol{\varphi}} + \ub{\boldsymbol{\varphi}}'$
      \STATE $\branches[t] \gets \textit{others} \cup \AlgorithmName{Prune}((\lb{\maskmat}',\ub{\maskmat}',\sumcoaliws\!',\lb{\maskval}',\ub{\maskval}'))$
    \ENDFOR
  \end{algorithmic}
\end{algorithm}

\paragraph{Computing~\(\lb{\contrib}_{\branch{}}\),~\(\ub{\contrib}_{\branch{}}\).}
We use bound propagation methods for neural networks as introduced in \cref{sec:bg-nnver} to compute~\(\lb{\contrib}_{\branch{}} \leq \contrib(S) \leq \ub{\contrib}_{\branch{}}, \forall S \in \branch\) where~\(\branch \subseteq \otherfeatures\) is a branch in the partition~\(\branches[t]\) defined by the sets of included and excluded features~\(\includeset, \excludeset \subseteq \upto{n} \setminus \{i\}\).
The challenge here is that bound propagation requires a function with real-valued vector arguments, while~\(\contrib\) has integer sets as arguments.
To resolve this, we represent sets of features~\(S\) using \emph{masks}~\(\mask \in \Bools^{n}\), such that~\(\mask[j] = 1 \equiv j \in S, \forall j \in \upto{n}\), and define~\(\maskcontrib: \Bools^{n} \to \Reals\) as~\(\maskcontrib(\mask) = \contrib(S)\).
The masks allow us to identify branches~\(\branch\) with the bounds~\(\lb{\mask} \leq \mask \leq \ub{\mask}\), where~\(\lb{\mask},\ub{\mask} \in \Bools^{n}\),~\(\lb{\mask}_j = 1 \equiv j \in \includeset\), and~\(\ub{\mask}_j = 0 \equiv j \in \excludeset\).
By relaxing the Boolean masks to the continuous domain~\({[0, 1]}^{n} \subset \Reals^{n}\), we can perform bound propagation on~\(\maskcontrib\) using~\([\lb{\mask},\ub{\mask}]\) as input bounds.

\begin{example}
    To better illustrate our use of masks, we show how to compute~\(\maskcontrib\) in practice.
    Let~\(\mask \in \Bools^{n}\) represent the set of features~\(S \in \otherfeatures\).
    Now,~\(\maskcontrib(\mask) = \maskval(\mask^{+i}) - \maskval(\mask)\), where~\({\mask[i]^{+i}} = 1\),~\({\mask[j]^{+i}} = \mask[j]\) for~\(j \in \upto{n} \setminus \{i\}\), and~\(\maskval: \Bools^n \to \Reals\) is~\(\maskval(\mask) = \mathbb{E}_{\z \sim \mathcal{D}}[f(\mask * \x + (1 - \mask) * \z)]\) with~\(*\) denoting element-wise multiplication and~\(\mathcal{D}\) is the data distribution.
    Here~\(\mask^{+i}\) corresponds to~\(S \cup \{i\}\) in \cref{eqn:contrib}, and~\(\mask * \x + (1 - \mask) * \z = (\x[S];\z_{\Bar{S}})\).
    \Cref{appendix:algorithm-additional} provides further details.
\end{example}

\paragraph{Pruning Branches (\AlgorithmName{Prune}).}
\newcommand{\prunedbranches}{\mathsf{P}\:\!\!\mathsf{B}}
For high-dimensional feature spaces, \OurAlgo{} may create a large number of branches leading to substantial memory demands. 
However, since branches~\(\branch\) with~\(\lb{\contrib}_{\branch} = \ub{\contrib}_{\branch}\) do not have to be refined further, these branches can be \emph{pruned}.
If~\(\prunedbranches\) is the family of all pruned branches, it suffices to store~\(\sum_{\branch \in \prunedbranches} \sumcoaliw{\branch}\lb{\contrib}_{\branch}\) instead of the individual branches in memory.
In \cref{algo:main}, this value is stored implicitly in~\([\lb{\boldsymbol{\varphi}}^{(t)}, \ub{\boldsymbol{\varphi}}^{(t)}]\).

\paragraph{Batch Processing.}
Bound propagation and branch and bound are particularly well-suited for modern massively-parallel hardware since they allow for processing many branches in parallel~\citep{xu2021alphacrown}. 
To optimally utilise such hardware, we implement \OurAlgo{} to split the branches in our partition in batches.
In \cref{algo:main},~\([\lb{\maskmat}, \ub{\maskmat}], \sumcoaliws\) and~\([\lb{\maskval}, \ub{\maskval}]\) represent batches of feature mask bounds~\([\lb{\mask}, \ub{\mask}]\),~\(\sumcoaliw{\branch}\) values, and value bounds~\([\lb{\val}, \ub{\val}]\), respectively.

\paragraph{Computing~\(\sumcoaliw{\branch}\).}
\Cref{theory:sumcoaliw} provides an efficient way to compute~\(\sumcoaliw{\branch}\).
It also gives rise to a recursive formula for~\(\sumcoaliw{\branch}\) that further accelerates the computation and lends itself well to batch-processing.

\NewDocumentCommand{\CorollaryCoaliWsRecursive}{}{
  Consider~\(\branch \in \branches[t]\) defined by~\(\includeset{}\) and~\(\excludeset{}\) with~\(r = |\includeset{}|\) and~\(s = |\includeset{}| + |\excludeset{}|\).
  For~\({\branch}'\) and~\({\branch}''\) being the branches derived from~\(\branch\) by including, respectively, excluding some feature, we have
  \begin{equation*}
    \sumcoaliw{\branchone} = \frac{r + 1}{s + 2} \sumcoaliw{\branch}
    \qquad
    \sumcoaliw{\branchtwo} = \frac{s + 1 - r}{s + 2} \sumcoaliw{\branch}.
  \end{equation*}
}
\begin{corollary}[Recursive Computation of~\(\sumcoaliw{\branch}\)]\label{theory:sum-coaliws-recursive}
    \CorollaryCoaliWsRecursive%
\end{corollary}

\emph{Proof Sketch.} \cref{theory:sum-coaliws-recursive} directly follows from \cref{theory:sumcoaliw} by applying two binomial identities.

\paragraph{Simultaneously Bounding~\(\shapval[1], \ldots, \shapval[n]\).}
\Cref{sec:method} describes computing bounds on~\(\shapval[i]\) for a fixed~\(i \in \upto{n}\).
This procedure can be executed in parallel for several features, e.g.,~\(i_1, i_2 \in \upto{n}\).
However, the branches computed for~\(i_1\) and~\(i_2\) overlap significantly: all branches that exclude both~\(i_1\) and~\(i_2\) are identical in the parallel runs.
We can avoid this duplication of effort by partitioning the entire set~\(\powerset{\upto{n}}\) instead of~\(\otherfeatures\), and computing bounds on~\(\shapval[i]\) for every~\(i \in \upto{n}\) by selecting the branches that contain, respectively, exclude the feature~\(i\).
In this case, we apply bound propagation to bound the \emph{value function}~\(\val(S)\) instead of the contribution~\(\contrib(S)\) for each branch.
Let~\(t \in \Nats\), let~\(\branches[t]\) be a partition of~\(\powerset{\upto{n}}\), and, for~\(\branch \in \branches[t]\), let~\({\lb{\val}}_{\branch} \leq \val(S) \leq {\ub{\val}}_{\branch}\),~\(\forall S \in \branch\) be the value bounds computed by bound propagation.
Using that~\(\contrib(S) = \val(S \cup \{i\}) - \val(S)\) and carefully accounting for~\(\sumcoaliw{\branch}\) following \cref{theory:sum-coaliws-recursive}, we obtain the following result by following similar steps as for proving \cref{theory:soundness}.
%
\NewDocumentCommand{\PropositionMultiSHAPBounds}{}{%
    Let~\(\branches[t]\) be a partition of~\(\powerset{\upto{n}}\).
    We have
    \begin{align*}
        \lb{\boldsymbol{\varphi}}_i^{(t)}\!
        &:=\! \sum_{\mathclap{\branch \in \branches[t]_{+i}}} \sumcoaliw{\branch}^{-} \lb{\val}_{\branch} 
        + \sum_{\mathclap{\branch \in \branches[t]_{\pm i}}} \sumcoaliw{\branch} (\lb{\val}_{\branch}\:\!\!-\:\!\!\ub{\val}_{\branch}) 
        - \sum_{\mathclap{\branch \in \branches[t]_{-i}}} \sumcoaliw{\branch}^{+} \ub{\val}_{\branch}
        \leq \shapval \\
        \ub{\boldsymbol{\varphi}}_i^{(t)}\!
        &:=\! \sum_{\mathclap{\branch \in \branches[t]_{+i}}} \sumcoaliw{\branch}^{-} \ub{\val}_{\branch} 
        + \sum_{\mathclap{\branch \in \branches[t]_{\pm i}}} \sumcoaliw{\branch} (\ub{\val}_{\branch}\:\!\!-\:\!\!\lb{\val}_{\branch}) 
        - \sum_{\mathclap{\branch \in \branches[t]_{-i}}} \sumcoaliw{\branch}^{+} \lb{\val}_{\branch}
        \geq \shapval,
    \end{align*}
    where~\(\branches[t]_{+i} \subseteq \branches[t]\) contains all~\(\branch \in \branches[t]\) with~\(i \in \includeset{}\),~\(\branches[t]_{-i} \subseteq \branches[t]\) contains all~\(\branch \in \branches[t]\) with~\(i \in \excludeset\),~\(\branches[t]_{\pm i} \subseteq \branches[t]\) contains the remaining branches,~\(\sumcoaliw{\branch}^{-} := \sumcoaliw{\branch}(s+1)/r \), and~\(\sumcoaliw{\branch}^{+} := \sumcoaliw{\branch}(s+1)/(s - r)\) with~\(r\) and~\(s\) as in \cref{theory:sumcoaliw}. 
}
\begin{proposition}\label{theory:multi-shap-bounds}
    \PropositionMultiSHAPBounds%
\end{proposition}



\section{Splitting Strategies}\label{sec:heuristics}
Refining partitions as described in \cref{sec:method} requires selecting \begin{enumerate*}[\itshape(i)] \item branches to split and \item features to split on.\end{enumerate*}
This section briefly introduces several strategies for performing these selections, which are discussed in detail in \cref{appendix:additional-splitting-strategies} and compared empirically in \cref{appendix:ablations}.

\paragraph{Selecting Branches to Split (\AlgoSelect{}).}
\newcommand{\branchdiam}{d_{\branch}}
The partitioning strategy in \cref{sec:method} requires selecting a branch~\(\branch[k]\) to split.
Since we process branches in batches, we instead select a batch of~\(b \in \Nats\) branches to split simultaneously.
We present two strategies for selecting branches: \SelectMaxDiam{} and \SelectMinDiam{} select the~\(b\) branches where~\(\branchdiam := \sumcoaliw{\branch}(\ub{\contrib}_{\branch} - \lb{\contrib}_{\branch})\) is the largest, respectively, the smallest.

\paragraph{Selecting Features to Split on (\AlgoSplit{}).}
The branch-and-bound literature provides numerous strategies for selecting features to split, which can also be applied in \OurAlgo{}.
We present a baseline strategy, \SplitInOrder{}, which selects features in a predetermined order, and adapt strong branching~\citep{applegate1995strongbranching}, smart branching~\citep{bunel2020branch,boetius2025solving}, and smears~\citep{kearfott1996smears,wang2018efficient} to our setting.
While strong branching and smart branching simulate splitting every available feature, smears leverages bounds of the contribution's gradient to determine which feature to split on.
Concretely, for the branch~\(\branch\), the smears strategy selects~\(j^\ast := \argmax_{j} \ub{|\contribgrad|}_j\), where~\(\ub{|\contribgrad|} \geq |\contribgrad(S)|\) is computed using \IBP{}.
\Cref{appendix:additional-splitting-strategies} provides further details.


\begin{figure*}
    \centering
    \begin{subfigure}[t]{.2\linewidth}
        \centering
        \includegraphics[width=.95\linewidth]{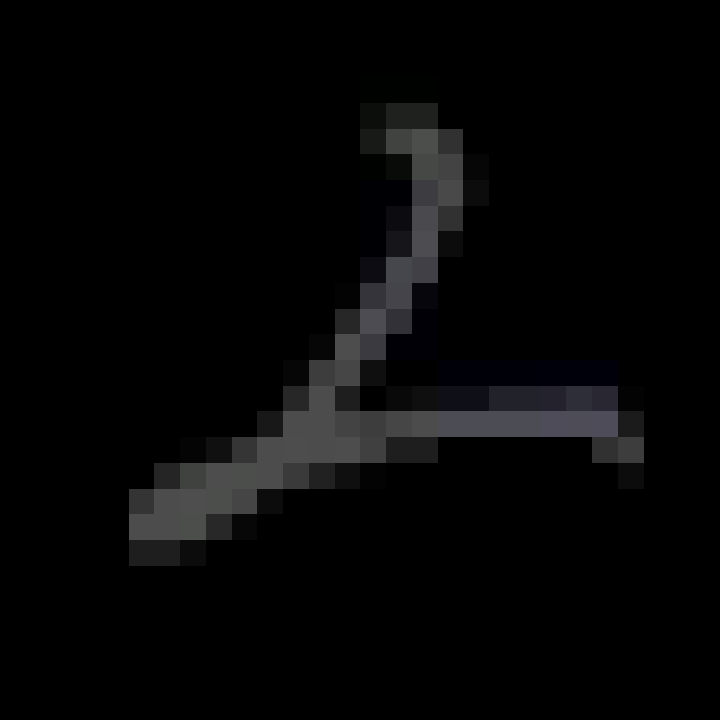}
        \caption{$t=1$}\label{fig:mnist-start}
    \end{subfigure}%
    \begin{subfigure}[t]{.2\linewidth}
        \centering
        \includegraphics[width=.95\linewidth]{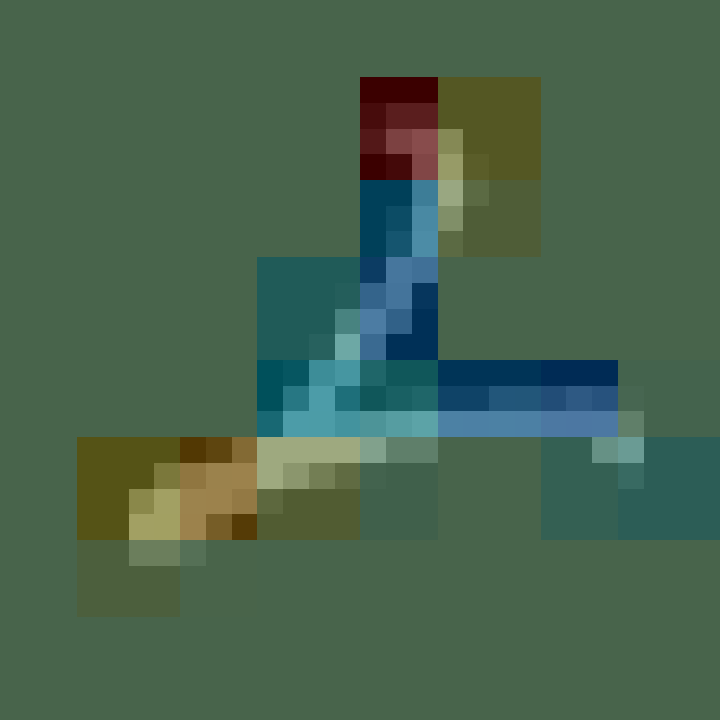}
        \caption{$t=121$}\label{fig:mnist-something-visible}  
    \end{subfigure}%
    \begin{subfigure}[t]{.2\linewidth}
        \centering
        \includegraphics[width=.95\linewidth]{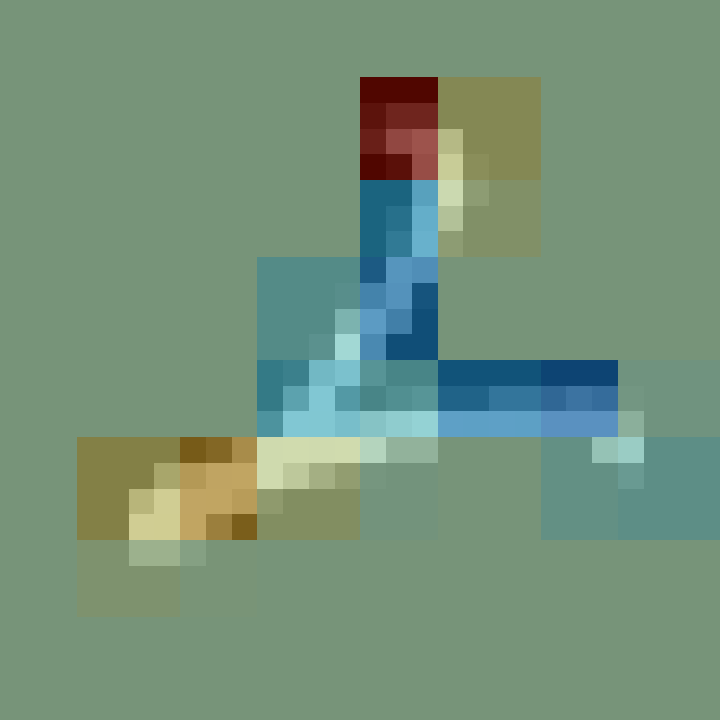}
        \caption{$t=278$ ($10\%$ HR)}\label{fig:mnist-bounds-tight}  
    \end{subfigure}%
    \begin{subfigure}[t]{.2\linewidth}
        \centering
        \includegraphics[width=.95\linewidth]{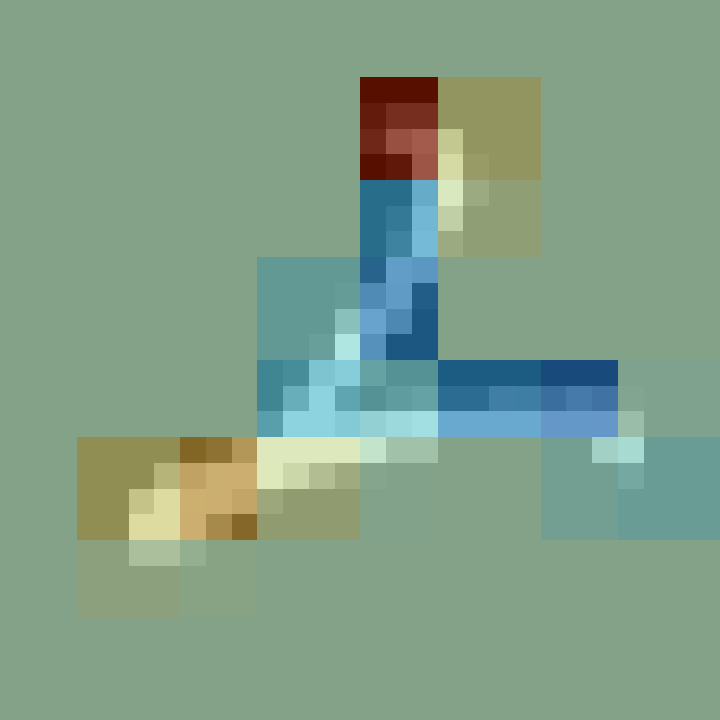}
        \caption{$t=512$ (Exact)}\label{fig:mnist-bounds-exact}  
    \end{subfigure}%
    \begin{subfigure}[t]{.2\linewidth}
        \centering
        \begin{tikzpicture}
            \begin{axis}[
                width=1.18\linewidth, height=1.18\linewidth,
                xlabel={half-range}, ylabel={midpoint},
                xlabel style={yshift=0pt, font=\footnotesize, inner sep=0pt},
                ylabel style={yshift=-8pt, xshift=3pt, font=\footnotesize, inner sep=0pt},
                tick label style={font=\footnotesize},
                enlargelimits=false, axis on top
            ]
                \addplot graphics [
                    xmin=0.0, xmax=1.9677828550338745, ymin=-1.9677828550338745, ymax=1.9677828550338745,
                ] {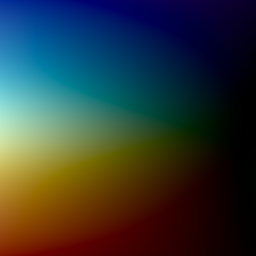};
            \end{axis}
        \end{tikzpicture}
        \caption{Bounds Colourmap}
    \end{subfigure}
    \caption{\textbf{MNIST \OurAlgo{} Results.}
        SHAP value bounds for the class \enquote{4} score of an MNIST CNN on a test set image of a \enquote{2}. The SHAP values are computed for an \(8\times8\) grid of superpixels. The SHAP bounds~\(\lb{\shapval}, \ub{\shapval}\) are visualised by their midpoints~\((\ub{\shapval} + \lb{\shapval}) / 2\) and half-ranges (HR)~\((\ub{\shapval} - \lb{\shapval}) / 2\) determining colour hue and lightness, respectively.
        \textbf{Lighter colour means tighter bounds in this figure.}
        In~\subref{fig:mnist-bounds-tight}, the largest half-range is less than~$10\%$ of the network output for \enquote{4} (\enquote{$10\%$ HR}) and in~\subref{fig:mnist-bounds-exact}, the bounds provide the exact SHAP value.
        An animated version of this figure is available in the supplementary material.
    }\label{fig:mnist-bounds}
\end{figure*}

\begin{figure}
    \centering
    \tikzexternalenable%
    \begin{tikzpicture}
        \begin{axis}[
            width=\linewidth, height=4.5cm,
            xlabel={Iteration~$t$}, ylabel={SHAP Value \(\shapval\)},
            xlabel style={yshift=-2pt},
            ylabel style={yshift=-2pt},
            ymin=-2.75, ymax=2.75,
            xmin=0, xmax=570,
            axis on top,
            axis background/.style={fill=black!65!white},
            every axis plot/.append style={
                mesh, mark=none, line width=3pt, 
                point meta=explicit symbolic,
            },
            mesh/color input=explicit,
        ]
          \addplot[] table [col sep=tab,x=Iteration,y=ub_0, meta=ub_0_color] {data/mnist_8x8_selected_features_bounds.tsv};
          \addplot[] table [col sep=tab,x=Iteration,y=lb_0, meta=lb_0_color] {data/mnist_8x8_selected_features_bounds.tsv};
          
          \addplot[] table [col sep=tab,x=Iteration,y=ub_12, meta=ub_12_color] {data/mnist_8x8_selected_features_bounds.tsv};
          \addplot[] table [col sep=tab,x=Iteration,y=lb_12, meta=lb_12_color] {data/mnist_8x8_selected_features_bounds.tsv};
          
          \addplot[] table [col sep=tab,x=Iteration,y=ub_38, meta=ub_38_color] {data/mnist_8x8_selected_features_bounds.tsv};
          \addplot[] table [col sep=tab,x=Iteration,y=lb_38, meta=lb_38_color] {data/mnist_8x8_selected_features_bounds.tsv};
          
          \node [pin={[font=\footnotesize,draw=black,fill=white,rectangle,pin distance=.2cm,pin edge={white, line width=1pt}]340:$i=38$}]  at (380,1.5) {};
          \node [pin={[font=\footnotesize,draw=black,fill=white,rectangle,pin distance=.2cm,pin edge={white, line width=1pt}]300:$i=0$}]  at (420,0.15) {};
          \node [pin={[font=\footnotesize,draw=black,fill=white,rectangle,pin distance=.2cm,pin edge={white, line width=1pt}]150:$i=12$}]  at (440,-2.1) {};

          \node [pin={[font=\footnotesize,draw=black,fill=white,rectangle,pin distance=.075cm,pin edge={white, line width=1pt}]265:$10\%$ HR}]  at (277,0.0) {};  
          \draw[black, line width=1.5pt, |-|] (277,-0.15) -- (277,0.15);
          
          \addplot graphics [
              xmin=530, xmax=5000, ymin=-1.9677828550338745, ymax=1.9677828550338745,
          ] {figures/mnist_bounds_zero_baseline/replay_colormap};
        \end{axis}
    \end{tikzpicture}
    \tikzexternaldisable%
    \caption{%
        \OurAlgo{} bounds for three features from \cref{fig:mnist-bounds}.
    }\label{fig:mnist-bounds-three-features}
\end{figure}
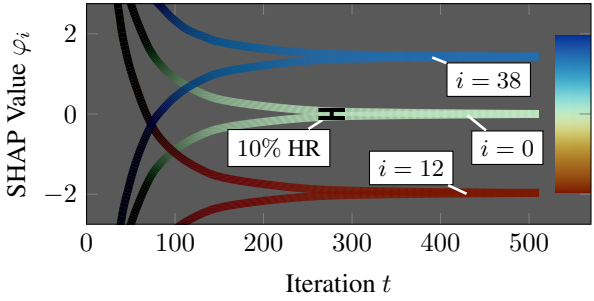

\section{Experiments}\label{sec:experiments}
In this section, we investigate the practical scalability of \OurAlgo{}, comparing it to state-of-the-art approaches for exactly computing and estimating SHAP values.
Our theoretical results in \cref{sec:method} establish that \OurAlgo{} can compute tight bounds on SHAP values, potentially without exploring all feature sets.
In the following, we 
\begin{enumerate}[\itshape(i)]
    \item investigate the practical convergence and tightness of \OurAlgo{}'s bounds,
    \item compare \OurAlgo{} to the state-of-the-art approach \ExactSHAP{}~\citep{shaplib} for computing exact SHAP values,
    \item investigate the relationship between \OurAlgo{} and SHAP estimators,
    \item compare SHAP estimators using the \enquote{ground-truth} exact SHAP values computed by \OurAlgo{} on larger search spaces than were previously feasible, and
    \item demonstrate the versatility of \OurAlgo{} by applying it to different neural network architectures.
\end{enumerate}
\Cref{appendix:datasets-networks-details} provides details on the neural networks and datasets used in this section.
\Cref{appendix:additional-experiments} contains additional experiments, such as a comparison of the splitting strategies introduced in \cref{sec:heuristics}.
In summary, we find that 
\begin{enumerate}[\itshape(i)]
    \item the bounds computed by \OurAlgo{} provide helpful insights long before \OurAlgo{} terminates;
    \item \OurAlgo{} scales to significantly larger search spaces than \ExactSHAP{}, both for computing exact SHAP values and, especially, for computing tight bounds;
    \item \OurAlgo{} complements SHAP estimators when high accuracy and precision are required;
    \item \OurAlgo{} provides novel insights into SHAP estimators by enabling evaluation on neural networks for larger search spaces; and
    \item \OurAlgo{} can compute tight bounds for a variety of network architectures.
\end{enumerate}

\paragraph{Experimental Setup.}
We implement \OurAlgo{} in Python using JAX~\citep{jax2018github}, leveraging a batched priority queue~\citep{chen2021gpu-priority-queue} for storing branches, and rely on the SHAP estimator implementations of~\citet{shaplib} and~\citet{witter2025regression}.
Following our ablation experiments in \cref{appendix:ablations}, we use the \CROWNIBP{}~\citep{zhang2020crown-ibp} \LBP{} algorithm for computing bounds, \SelectMaxDiam{} for branch selection, and \SplitSmears{} for selecting splits.
We run our experiments on an L40S NVIDIA GPU with 48GB of GPU memory deployed on an Ubuntu 24.04 machine.
Our code is available at~\url{https://github.com/sen-uni-kn/verishap/}.

\subsection{Insights into SHAP Values from Bounds}\label{sec:experiments-mnist}
For studying the convergence and tightness of the \OurAlgo{} bounds, we apply it to an \MNISTDataset{}~\citep{lecun1998mnist} convolutional neural network using the zero-baseline SHAP value function~\citep{sundararajan20many}.
\Cref{appendix:experiments-additional-image-datasets} provides further results on additional datasets and value functions.
To manage the input dimensionality, we compute SHAP values for an~\(8 \times 8\) grid of evenly spaced superpixels.
While the overall number of coalitions~\(2^{64} > 10^{19}\) is still unmanageably large, \OurAlgo{} is able to compute the exact SHAP values within~\(34\)s.
Furthermore, the bounds computed by \OurAlgo{} prove insightful many iterations before the exact SHAP values are computed, as visualised in \cref{fig:mnist-bounds,fig:mnist-bounds-three-features}.
While the initial bounds are too loose to be insightful, after~\(120\) iterations (\(25\)s), the bounds already reveal attribution patterns.
In subsequent iterations, the bounds tighten further until the exact value is computed after~\(512\) iterations (\(34\)s).
We found that \OurAlgo{} can compute the exact SHAP values in this setting because \CROWNIBP{} detects that the value function is constant across many coalitions.
As a result of \cref{theory:linear-shap}, this enables \OurAlgo{} to compute exact SHAP values without enumerating all coalitions.

\subsection{Comparison to \ExactSHAP{}}\label{sec:experiments-compare-to-exactshap}
We now compare \OurAlgo{} to the existing state-of-the-art algorithm \ExactSHAP{} for computing exact SHAP values.
\ExactSHAP{} is an optimised algorithm for enumerating all coalitions implemented in the \texttt{shap} library~\citep{shaplib}.
We train neural networks on a variety of widely used tabular datasets from the UCI ML repository~\citep{UCIMLRepo} and compute SHAP values for the first~\(10\) test set samples for each network, evaluating the marginal value function using a background dataset of~\(100\) samples.

\Cref{tab:ours-vs-exactshap} provides the median runtimes of \ExactSHAP{} and \OurAlgo{}.
\Cref{appendix:experiments-exactshap-additional} provides additional statistics.
While \ExactSHAP{} computes the exact SHAP values faster than \OurAlgo{} for up to~$|\powerset{\upto{n}}| = 2^{20} \simeq 10^{6}$ coalitions, \ExactSHAP{} exhausts the GPU memory for larger search spaces.
In contrast, \OurAlgo{} is able to compute exact SHAP values for instances with up to~$2^{25}\simeq 3\cdot10^7$ coalitions and provides tight bounds on SHAP values for up to~$2^{60}\simeq 10^{18}$ coalitions.
As \cref{tab:ours-vs-exactshap} shows, the runtime of \OurAlgo{} is not determined solely by the search space, but varies by dataset.
Although we also encountered three datasets (\HepatitisCDataset{}, \LungCancerDataset{}, \OnlineNewsDataset{}) with fewer than~$60$ features for which \OurAlgo{} was unable to compute tight bounds, \OurAlgo{} provides tight SHAP bounds for the majority of datasets.

\begin{table}
  \centering
  \caption{%
    \textbf{\OurAlgo{} vs. \ExactSHAP{}.}
    We report the median runtime required by \ExactSHAP{} and \OurAlgo{} for computing (\enquote{Exact}) or tightly bounding (\enquote{$10\%$ HR}, \enquote{$1\%$ HR}) the marginal SHAP values on tabular datasets of dimension~$n$, where \enquote{$p\%$ HR} denotes the half-range~$(\ub{\shapval} - \lb{\shapval}) / 2$ of the SHAP bounds~\(\lb{\shapval}, \ub{\shapval}\) being at most~$p\%$ of the network output to attribute, and \enquote{--} denotes exhausting the GPU memory or reaching the timeout of~\(600\)s.
  }\label{tab:ours-vs-exactshap}%
  \begin{tabular}{@{}l@{\hspace{.15cm}}l@{\hspace{.0cm}}R{.6cm}@{\hspace{.75cm}}R{.7cm}R{.7cm}R{.7cm}@{}}
        &                  & \llap{\textbf{\textsc{Exact}}\!\!} & \multicolumn{3}{l}{\hspace*{-.4cm}\textbf{\OurAlgo{} (Ours)}} \\
    $|\powerset{\upto{n}}|$ & \textbf{Dataset} & \llap{\textbf{\textsc{SHAP}}\,\!\!}  & \llap{$10$\% HR} & \llap{$1$\% HR} & \llap{Exact} \\
    \midrule
    $2^{16}{\scriptstyle>\!10^4}$ & \ObesityDataset{}       & $4$s & $ 18$s & $ 18$s & $ 18$s \\
    $2^{20}{\scriptstyle>\!10^6}$ & \GermanDataset{}        & $9$s & $ 16$s & $ 19$s & $ 20$s \\
    $2^{22}{\scriptstyle>\!10^6}$ & \MushroomDataset{}      &   -- & $ 17$s & $ 20$s & $ 25$s \\
    $2^{23}{\scriptstyle>\!10^6}$ & \DefaultDataset{}       &   -- & $127$s & $132$s & $133$s \\
    $2^{25}{\scriptstyle>\!10^7}$ & \AutomobileDataset{}    &   -- & $ 81$s & $213$s & $316$s \\
    $2^{27}{\scriptstyle>\!10^8}$ & \SteelDataset{}         &   -- & $ 37$s & $322$s &    --  \\
    $2^{30}{\scriptstyle>\!10^9}$ & \BreastCancerDataset{}  &   -- & $ 49$s & $322$s &    --  \\
    $2^{38}{\scriptstyle>\!10^{11}}$ & \AnnealingDataset{}     &   -- & $269$s &    --  &    --  \\
    $2^{60}{\scriptstyle>\!10^{18}}$ & \SonarDataset{}         &   -- & $ 13$s &    --  &    --
  \end{tabular}%
  \vspace*{-.1cm}
\end{table}

\subsection{Comparison to \KernelSHAP{} and \TreeMSR{}}\label{sec:experiments-compare-to-kernelshap}

\Cref{sec:experiments-compare-to-exactshap} shows that \OurAlgo{} outperforms the state-of-the-art approach for computing exact SHAP values. 
Since it is more common to statistically estimate SHAP values than to compute them exactly, we now compare \OurAlgo{} to the popular \KernelSHAP{}~\citep{lundberg2017unified} and the recent \TreeMSR{}~\citep{witter2025regression} estimators on the three highest-dimensional datasets from \cref{tab:ours-vs-exactshap} for which we obtained exact SHAP values.

\Cref{fig:ours-vs-estimators} presents our results.
The figure shows that both estimators cannot achieve the same accuracy as \OurAlgo{} within the memory constraints of our hardware.
While an estimator run requires significantly less time than \OurAlgo{}, it does not provide an indication of the estimate variance. 
When repeating the estimation~\(100\) times, as in \cref{fig:ours-vs-estimators}, \KernelSHAP{} requires 346s, 123s, and 351s on \MushroomDataset{}, \DefaultDataset{}, and \AutomobileDataset{}, respectively, at the largest feasible sample sizes.
Similarly, \TreeMSR{} requires 895s, 132s, and 341s on \MushroomDataset{}, \DefaultDataset{}, and \AutomobileDataset{}, respectively.
For both estimators, this exceeds the overall runtime of \OurAlgo{} on \MushroomDataset{} and \AutomobileDataset{}.
Overall, \OurAlgo{} complements statistical estimators when high accuracy and precision are required.
\Cref{appendix:experiments-additional-estimators} contains equivalent results for additional statistical SHAP estimators.

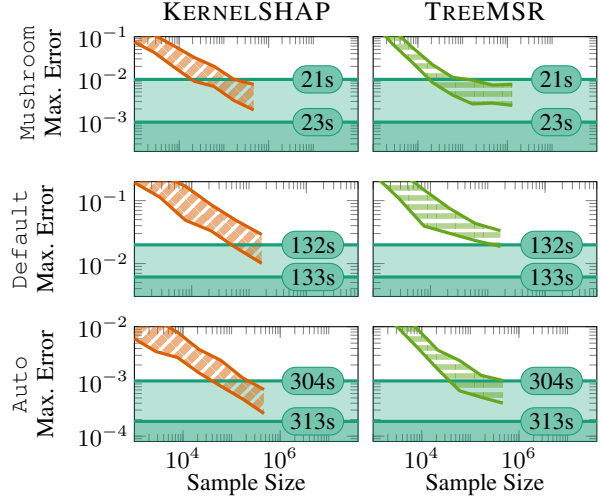
\begin{figure}
  \centering
  \tikzexternalenable%
  \hspace*{-.1cm}
  \begin{tikzpicture}
    \begin{groupplot}[
        group style={
            group size=2 by 3,
            xlabels at=edge bottom,
            ylabels at=edge left,
            xticklabels at=edge bottom,
            yticklabels at=edge left,
            vertical sep=.4cm,
            horizontal sep=.2cm,
        },
        height=3.1cm, width=.55\linewidth,
        xlabel={Sample Size},
        xticklabel style={font=\footnotesize},
        yticklabel style={font=\footnotesize},
        label style={inner sep=0pt, font=\footnotesize, align=center},
        title style={align=center, inner sep=0pt},
        ticklabel style={inner sep=2pt},
    ]
      \nextgroupplot[
        xmode=log, ymode=log, 
        title={\KernelSHAP{}}, 
        ylabel={\clap{\MushroomDataset{}}\\ Max. Error\\[-3pt]},
        ylabel style={yshift=-.1cm, xshift=.15cm},
        xmin=1000,xmax=100000000, ymin=0.0002, ymax=0.1,
        xtickten={4,6},
        minor xtick={1000,2000,3000,4000,5000,6000,7000,8000,9000,10000,20000,30000,40000,50000,60000,70000,80000,90000,100000,200000,300000,400000,500000,600000,700000,800000,900000,1000000,2000000,3000000,4000000,5000000,6000000,7000000,8000000,9000000,10000000},
      ]
      \addplot[name path=Zero, forget plot, domain=100:100000000, samples=2] {0.000000001};
      \addplot[name path=OursHigh, FirstSHAPLine, forget plot, domain=100:100000000, samples=2] {0.009887456893920898};  
      \addplot[name path=OursLow, FirstSHAPLine, forget plot, domain=100:100000000, samples=2] {0.0009784698486328125};  
      \addplot[FirstSHAPArea, opacity=0.3, forget plot] fill between [of=Zero and OursHigh];
      \addplot[FirstSHAPArea, opacity=0.3] fill between [of=Zero and OursLow];
      
      \addplot[name path=KernelSHAPMin, SecondSHAPLine, forget plot] table [col sep=comma,x=num_samples,y=KernelSHAP_min_error] {data/bab_vs_estimators/mushroom-mlp-8x1_linf_0.csv};
      \addplot[name path=KernelSHAPMax, SecondSHAPLine, forget plot] table [col sep=comma,x=num_samples,y=KernelSHAP_max_error] {data/bab_vs_estimators/mushroom-mlp-8x1_linf_0.csv};
      \addplot[SecondSHAPArea] fill between [of=KernelSHAPMin and KernelSHAPMax];
      
      
      \node[FirstTimeMarker] at (axis cs:12500000,0.009887456893920898) {21s};
      \node[FirstTimeMarker] at (axis cs:12500000,0.0009784698486328125) {23s};
      
      \nextgroupplot[
        xmode=log, ymode=log, 
        title={\TreeMSR{}}, 
        ylabel style={yshift=-.1cm, xshift=-0.0cm},
        xmin=1000,xmax=100000000, ymin=0.0002, ymax=0.1,
        xtickten={4,6},
        minor xtick={1000,2000,3000,4000,5000,6000,7000,8000,9000,10000,20000,30000,40000,50000,60000,70000,80000,90000,100000,200000,300000,400000,500000,600000,700000,800000,900000,1000000,2000000,3000000,4000000,5000000,6000000,7000000,8000000,9000000,10000000},
      ]
      \addplot[name path=Zero, forget plot, domain=100:1000000000, samples=2] {0.000000001};
      \addplot[name path=OursHigh, FirstSHAPLine, forget plot, domain=100:1000000000, samples=2] {0.009887456893920898};  
      \addplot[name path=OursLow, FirstSHAPLine, forget plot, domain=100:1000000000, samples=2] {0.0009784698486328125};  
      \addplot[FirstSHAPArea, opacity=0.3, forget plot] fill between [of=Zero and OursHigh];
      \addplot[FirstSHAPArea, opacity=0.3] fill between [of=Zero and OursLow];
      
      \addplot[name path=TreeMSRMin, ThirdSHAPLine, forget plot] table [col sep=comma,x=num_samples,y=TreeMSR_min_error] {data/bab_vs_estimators/mushroom-mlp-8x1_linf_0.csv};
      \addplot[name path=TreeMSRMax, ThirdSHAPLine, forget plot] table [col sep=comma,x=num_samples,y=TreeMSR_max_error] {data/bab_vs_estimators/mushroom-mlp-8x1_linf_0.csv};
      \addplot[ThirdSHAPArea] fill between [of=TreeMSRMin and TreeMSRMax];
      
      \node[FirstTimeMarker] at (axis cs:12500000,0.009887456893920898) {21s};
      \node[FirstTimeMarker] at (axis cs:12500000,0.0009784698486328125) {23s};

      
      \nextgroupplot[
        xmode=log, ymode=log, 
        ylabel={\clap{\DefaultDataset{}}\\ Max. Error\\[-3pt]},
        xmin=1000,xmax=8000000, ymin=0.003, ymax=0.2,
        xtickten={4,6},
        minor xtick={1000,2000,3000,4000,5000,6000,7000,8000,9000,10000,20000,30000,40000,50000,60000,70000,80000,90000,100000,200000,300000,400000,500000,600000,700000,800000,900000,1000000,2000000,3000000,4000000,5000000,6000000,7000000,8000000,9000000,10000000},
      ]
      \addplot[name path=Zero, forget plot, domain=100:10000000, samples=2] {0.000000001};
      \addplot[name path=OursHigh, FirstSHAPLine, forget plot, domain=100:10000000, samples=2] {0.019828736782073975};  
      \addplot[name path=OursLow, FirstSHAPLine, forget plot, domain=100:10000000, samples=2] {0.006113827228546143};  
      \addplot[FirstSHAPArea, opacity=0.3, forget plot] fill between [of=Zero and OursHigh];
      \addplot[FirstSHAPArea, opacity=0.3] fill between [of=Zero and OursLow];

      \addplot[name path=KernelSHAPMin, SecondSHAPLine, forget plot] table [col sep=comma,x=num_samples,y=KernelSHAP_min_error] {data/bab_vs_estimators/default-mlp-64x3_linf_0.csv};
      \addplot[name path=KernelSHAPMax, SecondSHAPLine, forget plot] table [col sep=comma,x=num_samples,y=KernelSHAP_max_error] {data/bab_vs_estimators/default-mlp-64x3_linf_0.csv};
      \addplot[SecondSHAPArea] fill between [of=KernelSHAPMin and KernelSHAPMax];


      \node[FirstTimeMarker] at (axis cs:1300000,0.019828736782073975) {132s};
      \node[FirstTimeMarker] at (axis cs:1300000,0.006113827228546143) {133s};

      \nextgroupplot[
        xmode=log, ymode=log, 
        xmin=1000,xmax=8000000, ymin=0.003, ymax=0.2,
        xtickten={4,6},
        minor xtick={1000,2000,3000,4000,5000,6000,7000,8000,9000,10000,20000,30000,40000,50000,60000,70000,80000,90000,100000,200000,300000,400000,500000,600000,700000,800000,900000,1000000,2000000,3000000,4000000,5000000,6000000,7000000,8000000,9000000,10000000},
      ]
      \addplot[name path=Zero, forget plot, domain=100:10000000, samples=2] {0.000000001};
      \addplot[name path=OursHigh, FirstSHAPLine, forget plot, domain=100:10000000, samples=2] {0.019828736782073975};  
      \addplot[name path=OursLow, FirstSHAPLine, forget plot, domain=100:10000000, samples=2] {0.006113827228546143};  
      \addplot[FirstSHAPArea, opacity=0.3, forget plot] fill between [of=Zero and OursHigh];
      \addplot[FirstSHAPArea, opacity=0.3] fill between [of=Zero and OursLow];

      \addplot[name path=TreeMSRMin, ThirdSHAPLine, forget plot] table [col sep=comma,x=num_samples,y=TreeMSR_min_error] {data/bab_vs_estimators/default-mlp-64x3_linf_0.csv};
      \addplot[name path=TreeMSRMax, ThirdSHAPLine, forget plot] table [col sep=comma,x=num_samples,y=TreeMSR_max_error] {data/bab_vs_estimators/default-mlp-64x3_linf_0.csv};
      \addplot[ThirdSHAPArea] fill between [of=TreeMSRMin and TreeMSRMax];

      \node[FirstTimeMarker] at (axis cs:1300000,0.019828736782073975) {132s};
      \node[FirstTimeMarker] at (axis cs:1300000,0.006113827228546143) {133s};

      \nextgroupplot[
        xmode=log, ymode=log, 
        ylabel={\clap{\AutomobileDataset{}}\\ Max. Error\\[-3pt]},
        xmin=1000,xmax=40000000, ymin=0.00008, ymax=0.01,
        xtickten={4,6},
        minor xtick={1000,2000,3000,4000,5000,6000,7000,8000,9000,10000,20000,30000,40000,50000,60000,70000,80000,90000,100000,200000,300000,400000,500000,600000,700000,800000,900000,1000000,2000000,3000000,4000000,5000000,6000000,7000000,8000000,9000000,10000000},
      ]
      \addplot[name path=Zero, forget plot, domain=100:100000000, samples=2] {0.000000001};
      \addplot[name path=OursHigh, FirstSHAPLine, forget plot, domain=100:100000000, samples=2] {0.001021549105644226};  
      \addplot[name path=OursLow, FirstSHAPLine, forget plot, domain=100:100000000, samples=2] {0.00018627941608428955};  
      \addplot[FirstSHAPArea, opacity=0.3, forget plot] fill between [of=Zero and OursHigh];
      \addplot[FirstSHAPArea, opacity=0.3] fill between [of=Zero and OursLow];

      \addplot[name path=KernelSHAPMin, SecondSHAPLine, forget plot] table [col sep=comma,x=num_samples,y=KernelSHAP_min_error] {data/bab_vs_estimators/automobile-mlp-32x2_linf_0.csv};
      \addplot[name path=KernelSHAPMax, SecondSHAPLine, forget plot] table [col sep=comma,x=num_samples,y=KernelSHAP_max_error] {data/bab_vs_estimators/automobile-mlp-32x2_linf_0.csv};
      \addplot[SecondSHAPArea] fill between [of=KernelSHAPMin and KernelSHAPMax];


      \node[FirstTimeMarker] at (axis cs:4700000,0.001021549105644226) {304s};
      \node[FirstTimeMarker] at (axis cs:4700000,0.00018627941608428955) {313s};

      \nextgroupplot[
        xmode=log, ymode=log, 
        xmin=1000,xmax=40000000, ymin=0.00008, ymax=0.01,
        xtickten={4,6},
        minor xtick={1000,2000,3000,4000,5000,6000,7000,8000,9000,10000,20000,30000,40000,50000,60000,70000,80000,90000,100000,200000,300000,400000,500000,600000,700000,800000,900000,1000000,2000000,3000000,4000000,5000000,6000000,7000000,8000000,9000000,10000000},
      ]
      \addplot[name path=Zero, forget plot, domain=100:100000000, samples=2] {0.000000001};
      \addplot[name path=OursHigh, FirstSHAPLine, forget plot, domain=100:100000000, samples=2] {0.001021549105644226};  
      \addplot[name path=OursLow, FirstSHAPLine, forget plot, domain=100:100000000, samples=2] {0.00018627941608428955};  
      \addplot[FirstSHAPArea, opacity=0.3, forget plot] fill between [of=Zero and OursHigh];
      \addplot[FirstSHAPArea, opacity=0.3] fill between [of=Zero and OursLow];

      \addplot[name path=TreeMSRMin, ThirdSHAPLine, forget plot] table [col sep=comma,x=num_samples,y=TreeMSR_min_error] {data/bab_vs_estimators/automobile-mlp-32x2_linf_0.csv};
      \addplot[name path=TreeMSRMax, ThirdSHAPLine, forget plot] table [col sep=comma,x=num_samples,y=TreeMSR_max_error] {data/bab_vs_estimators/automobile-mlp-32x2_linf_0.csv};
      \addplot[ThirdSHAPArea] fill between [of=TreeMSRMin and TreeMSRMax];

      \node[FirstTimeMarker] at (axis cs:4700000,0.001021549105644226) {304s};
      \node[FirstTimeMarker] at (axis cs:4700000,0.00018627941608428955) {313s};
      
    \end{groupplot}
  \end{tikzpicture}
  \tikzexternaldisable%
  \caption[VeriSHAP vs. KernelSHAP \& TreeMSR]{%
      \textbf{\OurAlgo{}~\legendcolorbox{FirstSHAPArea} vs. \KernelSHAP{}~\legendcolorbox{SecondSHAPArea} \& \TreeMSR{}~\legendcolorbox{ThirdSHAPArea}}
      We run \KernelSHAP{} and \TreeMSR{} with increasing sample sizes until each exhausts the available GPU memory, repeating each run~\(100\) times.
      For each sample size, we plot the range across repetitions of the \emph{largest error} between the estimated SHAP value and the true SHAP value for any feature.
      For \OurAlgo{}, we mark the half-ranges of the SHAP bounds at different~\tikz[baseline={([yshift={-2.7mm}]current bounding box.north)}]{\node[FirstTimeMarker] at (0,0) {runtimes}}.
  }\label{fig:ours-vs-estimators}
\end{figure}

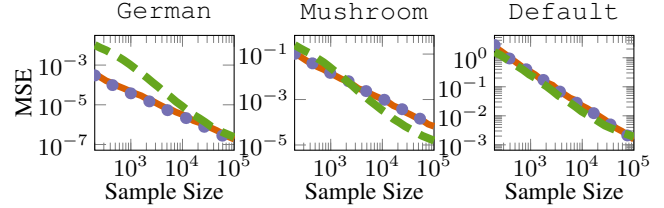
\begin{figure}[t]
  \centering
  \tikzexternalenable%
  \hspace*{-.1cm}
  \begin{tikzpicture}
    \begin{groupplot}[
        group style={
            group size=3 by 1,
            xlabels at=edge bottom,
            ylabels at=edge left,
            xticklabels at=edge bottom,
            vertical sep=.25cm,
            horizontal sep=.8cm,
        },
        height=3.1cm, width=.415\linewidth,
        xlabel={Sample Size},
        xticklabel style={font=\footnotesize},
        yticklabel style={font=\footnotesize},
        label style={inner sep=0pt, font=\footnotesize, align=center},
        title style={align=center, inner sep=0pt},
        ticklabel style={inner sep=2pt},
    ]
      \nextgroupplot[
        xmode=log, ymode=log, 
        title={\GermanDataset{}}, 
        ylabel={MSE\\[-7pt]},
        ylabel style={yshift=-.1cm},
        xmin=200,xmax=100000,
      ]
      \addplot[SecondSHAP] table [col sep=comma,x=num_samples,y=KernelSHAP_mean_error] {data/estimators_comparison/german-mlp-8x1_l2_mean.csv};
      \addplot[FourthSHAP] table [col sep=comma,x=num_samples,y=LeverageSHAP_mean_error] {data/estimators_comparison/german-mlp-8x1_l2_mean.csv};
      \addplot[ThirdSHAP] table [col sep=comma,x=num_samples,y=TreeMSR_mean_error] {data/estimators_comparison/german-mlp-8x1_l2_mean.csv};
      
      \nextgroupplot[
        xmode=log, ymode=log, 
        title={\MushroomDataset{}}, 
        ylabel style={yshift=-.1cm, xshift=.15cm},
        xmin=200,xmax=100000,
      ]
      \addplot[SecondSHAP] table [col sep=comma,x=num_samples,y=KernelSHAP_mean_error] {data/estimators_comparison/mushroom-mlp-8x1_l2_mean.csv};
      \addplot[FourthSHAP] table [col sep=comma,x=num_samples,y=LeverageSHAP_mean_error] {data/estimators_comparison/mushroom-mlp-8x1_l2_mean.csv};
      \addplot[ThirdSHAP] table [col sep=comma,x=num_samples,y=TreeMSR_mean_error] {data/estimators_comparison/mushroom-mlp-8x1_l2_mean.csv};
      
      \nextgroupplot[
        xmode=log, ymode=log, 
        title={\DefaultDataset{}}, 
        ylabel style={yshift=-.1cm, xshift=.15cm},
        xmin=200,xmax=100000,
      ]
      \addplot[SecondSHAP] table [col sep=comma,x=num_samples,y=KernelSHAP_mean_error] {data/estimators_comparison/default-mlp-64x3_l2_mean.csv};
      \addplot[FourthSHAP] table [col sep=comma,x=num_samples,y=LeverageSHAP_mean_error] {data/estimators_comparison/default-mlp-64x3_l2_mean.csv};
      \addplot[ThirdSHAP] table [col sep=comma,x=num_samples,y=TreeMSR_mean_error] {data/estimators_comparison/default-mlp-64x3_l2_mean.csv};
      
    \end{groupplot}
  \end{tikzpicture}
  \tikzexternaldisable%
  \caption[KernelSHAP vs. LeverageSHAP vs. TreeMSR]{%
      \textbf{\KernelSHAP{}~\legendcolorline{SecondSHAP} vs. \LeverageSHAP{}~\legendcolorline{FourthSHAP} vs. \TreeMSR{}~\legendcolorline{ThirdSHAP}}.
      We report the mean squared error (MSE) of the SHAP estimators across~\(100\) runs on three representative datasets. 
  }\label{fig:estimators-comparison}
\end{figure}

\subsection{Evaluating SHAP Estimators Using \OurAlgo{}}\label{sec:estimators-comparison}
We evaluate statistical SHAP estimators using the exact SHAP values computed by \OurAlgo{} on datasets where obtaining exact SHAP values was previously infeasible.
Following \citet{witter2025regression}, we compute the value function using the dataset mean as a singleton background dataset and report mean squared errors (MSE) in this experiment.
As \cref{fig:estimators-comparison} shows, \LeverageSHAP{}~\citep{musco2025leverageshap} outperforms \KernelSHAP{}, confirming a trend emerging on smaller search spaces~\citep{witter2025regression, musco2025leverageshap}.
\TreeMSR{} shows volatile performance across datasets when applied to neural networks, a nuanced insight that is not directly apparent from the results of~\citet{witter2025regression} for explaining tree-based models with large search spaces.
This exemplifies the potential of the \enquote{ground-truth} exact SHAP values computed by \OurAlgo{} on larger search spaces to provide insights that are not apparent in tree-based ML models and neural networks on lower-dimensional datasets.

\subsection{Activation Functions and Network Architectures}\label{sec:network-architectures}
To demonstrate the versatility of \OurAlgo{}, we apply \OurAlgo{} to \MushroomDataset{} networks with different activation functions and architectures.
Concretely, we apply \OurAlgo{} to fully-connected networks with ReLU,~\(\tanh\), and~\(\mathrm{Swish}\)~\cite{Hendrycks2016GELU} as activation function, as well as a RestNet~\cite{he2016resnet} with fully-connected layers.
\Cref{tab:mushroom-network-architectures} provides the runtime of \OurAlgo{} for the different networks.
While the runtime of \OurAlgo{} increases for the non-piecewise-linear~\(\tanh\) and~\(\mathrm{Swish}\)-activated neural networks, as well as the larger ResNet architecture, \OurAlgo{} is able to compute exact SHAP values for all networks.
\Cref{appendix:experiments-network-size} provides a complementary experiment comparing the scalability of \OurAlgo{} in the network size.

\begin{table}
    \centering
    \caption{
      \textbf{\OurAlgo{} for Different Activation Functions and Network Architectures.}
      In this table, \enquote{\(\mathrm{ReLU}\)},~\enquote{\(\tanh\)}, and~\enquote{\(\mathrm{Swish}\)} represent fully-connected networks with the respective activation functions, and~\enquote{ResNet} represents a ReLU-activated fully-connected network with residual connections, all trained on \MushroomDataset{}. The column headings are as for \cref{tab:ours-vs-exactshap}.
      The timeout for this experiment is~\(900s\).
    }\label{tab:mushroom-network-architectures}
    \begin{tabular}{lrrrr}
    & \multicolumn{3}{c}{\textbf{\OurAlgo{} (Ours) Runtime}} \\
         \textbf{Network} & \(10\%\) HR & \(1\%\) HR & \(0.1\%\) HR &	Exact \\ \midrule
\(\mathrm{ReLU}\)  & \(17\)s & \(20\)s  & \(24\)s  & \(25\)s \\
\(\tanh\)  & \(19\)s & \(37\)s  & \(61\)s  & \(670\)s \\
\(\mathrm{Swish}\) &	\(18\)s & \(35\)s  & \(56\)s  & \(70\)s \\
ResNet    & \(72\)s & \(135\)s & \(188\)s & \(199\)s
    \end{tabular}
\end{table}

\section{Related Work}

\paragraph{Exact SHAP.} To obtain Shapley values while preserving their full game-theoretic guarantees, one must compute them exactly. 
While exact computation is feasible, e.g., for tree-based models~\citep{bifet2022linear} and additive models~\citep{enoueninstashap, bordt2023shapley}, computing exact Shapley values is \#P-hard for general neural networks~\citep{van2022tractability, arenas2023complexity, marzouk2025computational}.
Unlike existing approaches that rely on restricted neural network architectures to enable exact SHAP computation~\citep{marzoukshap, chen2023harsanyinet, heidari2025tractable, muschalikexact}, our algorithm is applicable to general neural networks.


\paragraph{SHAP Value Estimation Methods.} 
Due to the computational complexity of computing exact SHAP values, a multitude of approaches exist for estimating SHAP values, including \KernelSHAP{}~\citep{lundberg2017unified,covert2021improving}, \TreeMSR{} and \LinearMSR{}~\citep{witter2025regression}, \LeverageSHAP{}~\citep{musco2025leverageshap}, \AlgorithmName{FastSHAP}~\citep{jethani2021fastshap}, \AlgorithmName{DeepSHAP}~\citep{Chen2021}, as well as Monte Carlo-based approaches~\citep{strumbelj2014explaining, mitchell2022sampling}, uncertainty-based estimates~\citep{ancona2019explaining, watson2023explaining}, and further model-specific methods for kernel models~\citep{chau2022rkhs, chau2023explaining}.
Our work addresses the substantially more challenging task of computing \emph{exact} SHAP values, and has the potential to serve as a \enquote{ground-truth} framework for evaluating SHAP estimators.

\paragraph{Neural Network Verification \& Formal Explainability.} 
Our approach is based on the recent rapid progress in the field of neural network verification~\citep{brix2024fifth, wang2021beta, zhou2025clipandverify, zhouscalable, singh2019deeppoly, muller2021scaling, chiusdp, ferraricomplete, boetius2025solving, wu2024marabou}. 
Recently, related techniques have also been used in the context of formal explainability~\citep{marques2022delivering, ignatiev2019aaai, darwiche2020reasons, bassan2023towards, audemard2022trading}, a subfield that seeks explanations with formal guarantees. 
However, in this context~\citep[see, e.g.,][]{bassanexplaining, wu2023verix, izza2024distance, emanuele2021guaranteed, labbaf2025, hadad2026formal, soria2026formal}, neural network verifiers are used for explanation types that relate directly to adversarial robustness, which simplifies the incorporation of neural network verifiers. 
Our work is the first to apply ideas from neural network verification to computing exact SHAP values.

\section{Limitations and Future Work}

Since computing exact SHAP values is~\#P-hard, all algorithms for computing them, including \OurAlgo{}, face substantial computational challenges.
As \OurAlgo{} relies on neural network verification, which has made rapid progress in recent years~\citep{brix2024fifth,wang2021beta,zhou2025clipandverify,zhouscalable,chiusdp}, and has demonstrated remarkable success on other NP-hard~\citep{katz2017reluplex} and \#P-hard problems~\citep{boetius2025solving}, its scalability will improve as neural network verification continues to advance. Importantly, \OurAlgo{} already scales to substantially larger search spaces than existing exact SHAP approaches. While \OurAlgo{} can still be slow for some networks, users can diagnose such cases early from the initial SHAP bounds, as demonstrated in \cref{appendix:predict-runtime}. 
We also acknowledge existing critiques of SHAP~\citep{huang2024failings,kumar2020problems,slack2020fooling,fryer2021shapley,biradar2024axiomatic}; our goal is not to defend its axiomatic foundations, but to provide a way to compute provably exact values for a widely used explanation method. 
%
%

Future work may extend our algorithm to higher-order Shapley interactions~\citep{fumagalli2023shap, fumagalli2024kernelshap, sundararajan20many, kolpaczki2024svarm}, exact SHAP under alternative value functions~\citep{mohammadi2025computing, letoffe2025towards} and input distributions~\citep{ghalebikesabi2021locality}, as well as other feature attribution indices~\citep{banieckiexplaining, barcelo2025computation}.


\section{Conclusion}\label{sec:conclusion}
We present \OurAlgo{}, the first algorithm that leverages recent advances in neural network verification to compute \emph{exact} SHAP values of neural networks. Our approach scales to search spaces orders of magnitude larger than those supported by existing methods, marking an important first step towards exact SHAP computation for neural networks. Moreover, our method flexibly provides provably exact bounds on SHAP values with arbitrary precision, enabling more scalable early termination while still yielding meaningful explanatory insights. Finally, our verification-based approach produces reliable \enquote{ground-truth} explanations for evaluating SHAP estimators for neural networks on larger search spaces. 
Hence, our work is an important step forward for producing and evaluating trustworthy explanations for neural networks through provable, verification-based guarantees.

\section*{Acknowledgements}
The work of Leue and Boetius was partially funded through the DFG research grant LE 1342/4 \enquote{SCADNet - Structural Causal Analysis of Deep Neural Networks}. The work of Katz and Bassan was partially funded by the European Union (ERC, VeriDeL, 101112713). Views and
opinions expressed are however those of the author(s) only and do not necessarily reflect those of the
European Union or the European Research Council Executive Agency. Neither the European Union
nor the granting authority can be held responsible for them. This research was additionally supported
by a grant from the Israeli Science Foundation (grant number 558/24).

%
%

\section*{Impact Statement}
Since our work advances both the practical and theoretical aspects of AI explainability, it shares broader implications common to the field, such as vulnerability to adversarial manipulation, privacy risks, and potential bias. However, by computing \emph{exact} SHAP values and certified bounds with provable guarantees, we aim to make these explanations more trustworthy and to provide a rigorous foundation for evaluating SHAP estimators in more realistic, higher-dimensional settings.

{%
  \def\UrlBreaks{\do\a\do\b\do\c\do\d\do\e\do\f\do\g\do\h\do\i\do\j\do\k\do\l\do\m\do\n\do\o\do\p\do\q\do\r\do\s\do\t\do\u\do\v\do\w\do\x\do\y\do\z\do\A\do\B\do\C\do\D\do\E\do\F\do\G\do\H\do\I\do\J\do\K\do\L\do\M\do\N\do\O\do\P\do\Q\do\R\do\S\do\T\do\U\do\V\do\W\do\X\do\Y\do\Z}

  \bibliography{main}
  \bibliographystyle{icml2026}
}%

\newpage
\appendix
\onecolumn


%
%
%
%
%

\NewDocumentCommand{\faketocline}{O{section} m m}{
  \contentsline{#1}{\numberline{\ref{#2}} \hyperref[#2]{#3}}{\pageref{#2}}{}
}
\NewDocumentCommand{\theorytocline}{O{section} m m}{
  \contentsline{#1}{\hyperref[#2]{#3}}{\pageref{#2}}{}
}

\vspace*{0cm}
\noindent{
    \hfill\bfseries{\huge Appendix}\hfill \llap{Page No.}
    \rule{\textwidth}{0.4pt}%
    \par%
}
\vspace*{-.25cm}
\faketocline{appendix:proofs}{Proofs}
\faketocline[subsection]{appendix:proofs-sumcoaliw}{Sums of Coalition Weights}
\theorytocline[subsubsection]{appendix:proofs-sumcoaliw-one}{\Cref{theory:sumcoaliw}}
\theorytocline[subsubsection]{appendix:proofs-sumcoaliw-two}{\Cref{theory:sum-coaliws-recursive}}
\faketocline[subsection]{appendix:proofs-shap-bounds}{SHAP Bounds}
\theorytocline[subsubsection]{appendix:proofs-soundness}{\Cref{theory:soundness}}
\theorytocline[subsubsection]{appendix:proofs-multi-shap-bounds}{\Cref{theory:multi-shap-bounds}}
\faketocline[subsection]{appendix:proofs-termination}{Termination}
\theorytocline[subsubsection]{appendix:proofs-termination-one}{\Cref{theory:termination}}
\theorytocline[subsubsection]{appendix:proofs-termination-two}{\Cref{theory:arbitrary-precision}}
\theorytocline[subsubsection]{appendix:proofs-linear-model}{\Cref{theory:linear-shap}}
\faketocline{appendix:algorithm-additional}{Additional Details on \OurAlgo{}}
\faketocline{appendix:additional-splitting-strategies}{Splitting Strategies}
\faketocline{appendix:datasets-networks-details}{Datasets and Networks}
\faketocline{appendix:additional-experiments}{Additional Experiments}
\faketocline[subsection]{appendix:ablations}{Ablation Studies}
\faketocline[subsection]{appendix:experiments-network-size}{Scalability in Network Size}
\faketocline[subsection]{appendix:experiments-additional-image-datasets}{Additional Experiments on Image Datasets}
\faketocline[subsection]{appendix:experiments-exactshap-additional}{Additional Details on the Comparison to \ExactSHAP{}}
\faketocline[subsection]{appendix:experiments-additional-estimators}{Comparison to Additional Statistical Estimators}
\faketocline[subsection]{appendix:additional-compare-estimators}{Evaluating SHAP Estimators using \OurAlgo{}}

\par\noindent\rule{\textwidth}{0.4pt}

\vspace*{1cm}


\section{Proofs}\label{appendix:proofs}
We restate all theoretical results for easier reference in this section.

\subsection{Sums of Coalition Weights}\label{appendix:proofs-sumcoaliw}

\begin{repeatedtheory}[\Cref{theory:sumcoaliw}]\label{appendix:proofs-sumcoaliw-one}
    Consider~\(\branch \in \branches[t]\) defined by~\(\includeset{}\) and~\(\excludeset{}\).
    Letting~\(r := |\includeset{}|\) and~\(s := |\includeset{}| + |\excludeset{}|\), we have 
    \begin{equation}
        \sumcoaliw{\branch} = \sum_{S \in \branch} \coaliw(|S|) = \frac{1}{(s + 1){s \choose r}}.\label{eqn:sumcoaliw}
    \end{equation}
\end{repeatedtheory}

\emph{Proof.} 
By plugging the value of~$\coaliw$ into \cref{eqn:sumcoaliw}, \cref{theory:sumcoaliw} becomes equivalent to showing that

\begin{equation}
    \frac{1}{n}\sum^{n-1-s}_{k=0}\binom{n-1-s}{k}\binom{n-1}{k+r}^{-1}=\frac{1}{(s+1)\binom{s}{r}}.\label{eqn:sumcoaliws-expanded}
\end{equation}

By the definition of the binomial coefficient,
\begin{equation*}
    \binom{n-1}{k+r}^{-1} = \frac{(k+r)!(n-1-k-r)!}{(n-1)!}.
\end{equation*}

We also recall the following identity for the Beta function 
\begin{equation}
    \beta(a,b) = \int^1_0t^{a-1}(1-t)^{b-1}dt=\frac{(a-1)!(b-1)!}{(a+b-1)!},\label{beta_function}
\end{equation}

as well as the classical truncated binomial theorem, which for~$x \in \Reals$ takes the form
\begin{equation}
    \sum_{k=0}^{K}\binom{K}{k}x^k 
    = (1+x)^{K}.\label{eqn:binomial-thm-truncated}
\end{equation}

Setting~$a := k+r+1$ and~$b := n-k-r$ in \cref{beta_function} yields
\begin{align*}
 &    &  \beta(k+r+1,n-k-r) &= \frac{(k+r)!(n-k-r-1)!}{n!}  \\
 &\iff& n\cdot \beta(k+r+1,n-k-r) &= \frac{(k+r)!(n-k-r-1)!}{(n-1)!}=\binom{n-1}{k+r}^{-1}.
\end{align*}

Hence, we obtain
\begin{align*}
  &    &  \binom{n-1}{k+r}^{-1} &= n\int^1_0 t^{k+r}(1-t)^{n-k-r-1}dt  \\
  &\iff&  \frac{1}{n}\cdot\binom{n-1}{k+r}^{-1} &= \int^1_0 t^{k+r}(1-t)^{n-k-r-1}dt  \\
  &\iff&  \frac{1}{n}\cdot\binom{n-1}{k+r}^{-1}\cdot \binom{n-1-s}{k} &= \binom{n-1-s}{k}\cdot \int^1_0 t^{k+r}(1-t)^{n-k-r-1}dt.
\end{align*}

Substituting this result into the left-hand side of \cref{eqn:sumcoaliws-expanded} yields\begin{equation}
    \begin{aligned}
       \frac{1}{n}\sum^{n-1-s}_{k=0}\binom{n-1-s}{k}\binom{n-1}{k+r}^{-1}
       &= \sum^{n-1-s}_{k=0}\binom{n-1-s}{k}\cdot \int^1_0 t^{k+r}(1-t)^{n-k-r-1}dt \\
       &= \sum^{n-1-s}_{k=0}\binom{n-1-s}{k}\cdot \int^1_0 t^{r}(1-t)^{n-r-1}\Big(\frac{t}{1-t}\Big)^kdt.
    \end{aligned}\label{eqn:sumcoaliw-as-beta}
\end{equation}

Setting~$x:=\frac{t}{1-t}$ and~\(K := n-1-s\) in \cref{eqn:binomial-thm-truncated}, we get
\begin{equation*}
    \sum_{k=0}^{n-1-s}\binom{n-1-s}{k}\Big(\frac{t}{1-t}\Big)^k=\Big(\frac{1}{1-t}\Big)^{n-1-s}.
\end{equation*}
Inserting this expression into \cref{eqn:sumcoaliw-as-beta}, we obtain
\begin{align*}
    \frac{1}{n}\sum^{n-1-s}_{k=0}\binom{n-1-s}{k}\binom{n-1}{k+r}^{-1}
    &= \sum^{n-1-s}_{k=0}\binom{n-1-s}{k}\cdot \int^1_0 t^{r}(1-t)^{n-r-1}\Big(\frac{t}{1-t}\Big)^kdt \\
    &= \int^1_0 t^{r}(1-t)^{n-r-1}\Big(\frac{1}{1-t}\Big)^{n-1-s}dt \\
    &= \int^1_0 t^{r}(1-t)^{s-r}dt.
\end{align*}

Using \cref{beta_function}, we obtain
\begin{equation*}
    \int^1_0 t^{r}(1-t)^{s-r}dt=\beta(r+1,s-r+1)=\frac{r!(s-r)!}{(s+1)!}= \frac{r!(s-r)!}{s!(s+1)}=\frac{1}{\binom{s}{r}(s+1)}.
\end{equation*}
Therefore,
\begin{equation*}
\sumcoaliw{\branch} = \frac{1}{n}\sum^{n-1-s}_{k=0}\binom{n-1-s}{k}\binom{n-1}{k+r}^{-1}=\int^1_0 t^{r}(1-t)^{s-r}dt=\frac{1}{\binom{s}{r}(s+1)},
\end{equation*}
which completes the proof.

$\qedsymbol{}$


\begin{repeatedtheory}[\Cref{theory:sum-coaliws-recursive}]\label{appendix:proofs-sumcoaliw-two}
    \CorollaryCoaliWsRecursive%
\end{repeatedtheory}

\emph{Proof.}
  Let~\(\branch\),~\(\branchone\),~\(\branchtwo\),~\(r\), and~\(s\) be as in \cref{theory:sum-coaliws-recursive}.
  By \cref{theory:sumcoaliw}, we have
  \begin{align*}
    \sumcoaliw{\branchone}
    &= \frac{1}{(s+2){s+1 \choose r+1}} 
    = \frac{r + 1}{s + 2}\frac{1}{(s+1){s \choose r}} \\
    &= \frac{r + 1}{s + 2}\sumcoaliw{\branch} \\
    \sumcoaliw{\branchtwo}
    &= \frac{1}{(s+2){s+1 \choose r}} 
    = \frac{s + 1 - r}{s + 2}\frac{1}{(s+1){s \choose r}} \\
    &= \frac{s + 1 - r}{s + 2}\sumcoaliw{\branch}.
  \end{align*}
  
$\qedsymbol{}$


\subsection{SHAP Bounds}\label{appendix:proofs-shap-bounds}
\begin{repeatedtheory}[\Cref{theory:soundness}]\label{appendix:proofs-soundness}
    \TheoremSoundness%
\end{repeatedtheory}

\emph{Proof.} 
By definition, we have~\(\sumcoaliw{\branch} = \sum_{S \in \branch} \coaliw(|S|)\). 
%
%
It directly follows that
\begin{align*}
    \lb{\shapval}^{(t)} 
    &= \sum_{\branch \in \branches[t]} \sum_{S \in \branch} \coaliw(|S|)\lb{\contrib}_{\branch} \\
    \ub{\shapval}^{(t)} 
    &= \sum_{\branch \in \branches[t]} \sum_{S \in \branch} \coaliw(|S|)\ub{\contrib}_{\branch}.
\end{align*}
Since~\([\lb{\contrib}_{\branch}, \ub{\contrib}_{\branch}]\) are computed using bound propagation, it holds that~$\lb{\contrib}_{\branch} \leq \contrib(S) \leq \ub{\contrib}_{\branch}$,~\(\forall S \in \branch\). 
Hence, since~\(\coaliw(k) > 0\),~\(\forall k \in \{0, \ldots, n-1\}\), we obtain our final result by applying \cref{eqn:shapley}
\begin{equation*}
\lb{\shapval}^{(t)} 
    = \sum_{S \in \otherfeatures} \coaliw(|S|)\lb{\contrib}_{\branch}\leq
\shapval=\sum_{S \in \otherfeatures} \coaliw(|S|) \contrib(S)
    \leq \sum_{S \in \otherfeatures} \coaliw(|S|)\ub{\contrib}_{\branch} = \ub{\shapval}^{(t)}.
\end{equation*}
$\qedsymbol{}$

\begin{repeatedtheory}[\Cref{theory:multi-shap-bounds}]\label{appendix:proofs-multi-shap-bounds}
    \PropositionMultiSHAPBounds%
\end{repeatedtheory}

\emph{Proof.}
We first recall \cref{eqn:shapley} and rewrite it as
\begin{equation}
    \shapval 
    = \sum_{S \in \otherfeatures} \coaliw(|S|) \contrib(S) 
    = \sum_{S \in \otherfeatures} \coaliw(|S|) \val(S \cup \{i\}) - \sum_{S \in \otherfeatures} \coaliw(|S|) \val(S)
    = \underbrace{\sum_{S \in \otherfeatures^{+}} \coaliw(|S| - 1) \val(S)}_{(\ast)} 
      - \underbrace{\sum_{S \in \otherfeatures} \coaliw(|S|) \val(S)}_{(\ddagger)},\label{eqn:multishap-proof-one}
\end{equation}
where~\(\otherfeatures^{+} = \{S \cup \{i\} \mid S \in \otherfeatures\}\) is the set of all feature sets, or \emph{coalitions}, containing~\(i\).
We now derive bounds on~\((\ast)\) and~\((\ddagger)\).
\begin{description}
    \item[\textnormal{\itshape Bounds on~\((\ast)\).}] 
    Since~\(\branches[t]\) is a partition of~\(\powerset{\upto{n}} \supset \otherfeatures^{+}\), we can partition~\(\otherfeatures^{+}\) along~\(\branches[t]\).
    Since~\(\otherfeatures^{+}\) is the set of all coalitions containing~\(i\), we have~\(\otherfeatures^{+} \cap \branch \neq \emptyset\) for all~\(\branch \in \branches[t]_{+i} \cup \branches[t]_{\pm i}\) and~\(\otherfeatures^{+} \cap \branch = \emptyset\) for all~\(\branch \in \branches[t]_{-i}\).
    In particular, by definition, the branches in~\(\branches[t]_{+i}\) \emph{only} contain coalitions containing~\(i\) and \emph{no} branch in~\(\branches[t]_{-i}\) contains coalitions containing~\(i\), while the branches in~\(\branches[t]_{\pm i}\) contain both coalitions with and without~\(i\).
    This allows us to partition~\(\otherfeatures^{+}\) as follows
    \begin{align*}
        \sum_{S \in \otherfeatures^{+}} \coaliw(|S| - 1) \val(S) 
        &= \sum_{\branch \in \branches[t]_{+i}} \sum_{S \in \branch} \coaliw(|S| - 1) \val(S)
           + \sum_{\branch \in \branches[t]_{\pm i}} \sum_{S \in \branch \cap \otherfeatures^{+}} \coaliw(|S| - 1) \val(S) \\
        &\geq \sum_{\branch \in \branches[t]_{+i}} \lb{\val}_{\branch} \sum_{S \in \branch} \coaliw(|S| - 1)
           + \sum_{\branch \in \branches[t]_{\pm i}} \lb{\val}_{\branch} \sum_{S \in \branch \cap \otherfeatures^{+}} \coaliw(|S| - 1),
    \end{align*}
    where the inequality holds since~\(\lb{\val}_{\branch} \leq \val(S) \leq \ub{\val}_{\branch}\),~\(\forall S \in \branch\), since~\([\lb{\val}_{\branch}, \ub{\val}_{\branch}]\) is computed using bound propagation.
    We now derive a closed-form expression for the terms~\(\sum_{S} \coaliw(|S| - 1)\) in terms of~\(\sumcoaliw{\branch}\) for~\(\branch \in \branches[t]_{+i}\) and~\(\branch \in \branches[t]_{\pm i}\).
    The main difficulty here is that, when~\(\branches[t]\) partitions~\(\powerset{\upto{n}}\) instead of~\(\otherfeatures\) and~\(\sumcoaliw{\branch}\) is computed in terms of the number of included and excluded features following \cref{theory:sumcoaliw},~\(\sumcoaliw{\branch}\) is effectively computed as if there were~\(n+1\) features instead of~\(n\).
    We follow through this reasoning for~\(\branches[t]_{+i}\) and~\(\branches[t]_{\pm i}\) separately.
    \begin{itemize}
      \item 
        We first consider~\(\branch \in \branches[t]_{+i}\) defined by~\(\includeset \ni i\) and~\(\excludeset \not\ni i\).
        Consider an equivalent branch~\({\branch}^{-} \subseteq \otherfeatures\) from a partitioning of~\(\otherfeatures\) defined by~\({\excludeset}^{-} = \excludeset\) and~\({\includeset}^{-} = \includeset \setminus \{i\}\), since~\(i\) cannot be included in a branch when partitioning~\(\otherfeatures\).
        It holds that~\(|\branch| = |{\branch}^{-}|\) and, since the coalitions in~\({\branch}^{-}\) contain exactly one element less than the corresponding coalitions in~\(\branch\), we have~\(
            \sum_{S \in {\branch}^{-}} \coaliw(|S|) = \sum_{S \in \branch} \coaliw(|S| - 1)
        \).
        Let~\(\sumcoaliw{\branch}^{-} = \sum_{S \in {\branch}^{-}} \coaliw(|S|)\).
        By \cref{theory:sumcoaliw}, we can compare~\(\sumcoaliw{\branch}\) and~\(\sumcoaliw{\branch}^{-}\) in terms of the number of included and excluded features.
        Let~\(r = |\includeset|\) and~\(s = |\includeset| + |\excludeset|\). 
        Since~\(|{\includeset}^{-}| = r - 1\) and~\(|{\includeset}^{-}| + |{\excludeset}^{-}| = s - 1\), by \cref{theory:sum-coaliws-recursive}, we have
        \begin{equation*}
            \sumcoaliw{\branch}
            = \frac{|{\includeset}^{-}| + 1}{|{\includeset}^{-}| + |{\excludeset}^{-}| + 2}\sumcoaliw{\branch}^{-} 
            \qquad\iff\qquad
            \sumcoaliw{\branch}
            = \frac{r}{s + 1}\sumcoaliw{\branch}^{-} 
            \qquad\iff\qquad
            \frac{s + 1}{r}\sumcoaliw{\branch} 
            = \sumcoaliw{\branch}^{-}.
        \end{equation*}
      \item
        Now, we consider~\(\branch \in \branches[t]_{\pm i}\) defined by~\(\includeset \not\ni i\) and~\(\excludeset \not\ni i\).
        There is a branch~\(\hat{\branch} \subseteq \otherfeatures\) from a partitioning of~\(\otherfeatures\) defined by~\(\hat{\includeset} = \includeset\) and~\(\hat{\excludeset} = \excludeset\) that satisfies~\(\hat{\branch} = \branch \cap \otherfeatures\).
        Since~\(|\hat{\branch}| = |\branch \cap \otherfeatures| = |\branch \cap \otherfeatures^{+}|\), we have~\(
            \sumcoaliw{\hat{\branch}} = \sum_{S \in \hat{\branch}} \coaliw(|S|) = \sum_{S \in \branch \cap \otherfeatures^{+}} \coaliw(|S| - 1)
        \).
        By \cref{theory:sumcoaliw}, we have~\(\sumcoaliw{\hat{\branch}} = \sumcoaliw{\branch}\).
    \end{itemize}
    Overall, we obtain
    \begin{equation*}
        \sum_{S \in \otherfeatures^{+}} \coaliw(|S| - 1) \val(S) 
        \geq \sum_{\branch \in \branches[t]_{+i}} \lb{\val}_{\branch} \sum_{S \in \branch} \coaliw(|S| - 1)
           + \sum_{\branch \in \branches[t]_{\pm i}} \lb{\val}_{\branch} \sum_{S \in \branch \cap \otherfeatures^{+}} \coaliw(|S| - 1)
        = \sum_{\branch \in \branches[t]_{+i}} \sumcoaliw{\branch}^{-} \lb{\val}_{\branch}
           + \sum_{\branch \in \branches[t]_{\pm i}} \sumcoaliw{\branch}\lb{\val}_{\branch}.
    \end{equation*}
    From an equivalent chain of arguments, we also obtain the upper bound
    \begin{equation*}
        \sum_{S \in \otherfeatures^{+}} \coaliw(|S| - 1) \val(S) 
        \leq \sum_{\branch \in \branches[t]_{+i}} \sumcoaliw{\branch}^{-} \ub{\val}_{\branch}
           + \sum_{\branch \in \branches[t]_{\pm i}} \sumcoaliw{\branch}\ub{\val}_{\branch}.
    \end{equation*}
    \item[\textnormal{\textit{Bounds on~\((\ddagger)\).}}]
    We follow similar steps as for~\((\ast)\).
    As for~\(\otherfeatures^{+}\), we can similarly partition~\(\otherfeatures\) along~\(\branches[t]\).
    Since~\(\otherfeatures = \powerset{\upto{n} \setminus \{i\}}\), we have~\(\otherfeatures \cap \branch \neq \emptyset\) for all~\(\branch \in \branches[t]_{-i} \cup \branches[t]_{\pm i}\) and~\(\otherfeatures \cap \branch = \emptyset\) for all~\(\branch \in \branches[t]_{+i}\).
    Since~\(\val(S) \geq \lb{\val}_{\branch}\),~\(\forall S \in \branch\), we obtain
    \begin{align*}
        \sum_{S \in \otherfeatures} \coaliw(|S|) \val(S) 
        &\geq \sum_{\branch \in \branches[t]_{-i}} \lb{\val}_{\branch} \sum_{S \in \branch} \coaliw(|S|)
           + \sum_{\branch \in \branches[t]_{\pm i}} \lb{\val}_{\branch} \sum_{S \in \branch \cap \otherfeatures} \coaliw(|S|).
    \end{align*}
    As for~\((\ast)\), we express the terms~\(\sum_{S} \coaliw(|S|)\) in terms of~\(\sumcoaliw{\branch}\), while differentiating~\(\branch \in \branches[t]_{-i}\) and~\(\branch \in \branches[t]_{\pm i}\).
    \begin{itemize}
        \item 
          Let~\(\branch \in \branches[t]_{-i}\) be defined by~\(\includeset \not\ni i\) and~\(\excludeset \ni i\).
          Consider a branch~\({\branch}^{+}\) containing the same coalitions as~\(\branch\) but from a partitioning of~\(\otherfeatures\).
          Being from a partitioning of~\(\otherfeatures\),~\({\branch}^{+}\) is defined by~\({\includeset}^{+} = \includeset\) and~\({\excludeset}^{+} = \excludeset \setminus \{i\}\), since~\(i\) cannot be excluded in a partitioning of~\(\otherfeatures\), as it is always excluded.
          Since~\({\branch}^{+}\) contains the same elements as~\(\branch\), we trivially have~\(\sum_{S \in {\branch}^{+}} \coaliw(|S|) = \sum_{S \in {\branch}} \coaliw(|S|)\).
          Let~\(\sumcoaliw{\branch}^{+} = \sum_{S \in {\branch}^{+}} \coaliw(|S|)\),~\(r = |\includeset|\), and~\(s = |\includeset| + |\excludeset|\).
          We have~\(|{\includeset}^{+}| = r\) and~\(|{\includeset}^{+}| + |{\excludeset}^{+}| = s - 1\).
          By applying \cref{theory:sumcoaliw} and \cref{theory:sum-coaliws-recursive}, we obtain
          \begin{equation*}
              \sumcoaliw{\branch}
              = \frac{|{\includeset}^{+}| + |{\excludeset}^{+}| + 1 - |{\includeset}^{+}|}{|{\includeset}^{+}| + |{\excludeset}^{+}| + 2}\sumcoaliw{\branch}^{+} 
              \qquad\iff\qquad
              \sumcoaliw{\branch}
              = \frac{s - r}{s + 1}\sumcoaliw{\branch}^{+} 
              \qquad\iff\qquad
              \frac{s + 1}{s - r}\sumcoaliw{\branch} 
              = \sumcoaliw{\branch}^{+}.
          \end{equation*}
        \item
          Let~\(\branch \in \branches[t]_{\pm i}\) be defined by~\(\includeset \not\ni i\) and~\(\excludeset \not\ni i\).
          As for~\((\ast)\), there is a branch~\(\hat{\branch} \subseteq \otherfeatures\) from a partitioning of~\(\otherfeatures\) defined by~\(\hat{\includeset} = \includeset\) and~\(\hat{\excludeset} = \excludeset\) that satisfies~\(\hat{\branch} = \branch \cap \otherfeatures\).
          Therefore,~\(
                \sumcoaliw{\hat{\branch}} = \sum_{S \in \hat{\branch}} \coaliw(|S|) = \sum_{S \in \branch \cap \otherfeatures} \coaliw(|S|)
          \).
          By \cref{theory:sumcoaliw}, we have~\(\sumcoaliw{\hat{\branch}} = \sumcoaliw{\branch}\).
    \end{itemize}
    Overall, we obtain the bounds
    \begin{equation*}
        \sum_{\branch \in \branches[t]_{-i}} \sumcoaliw{\branch}^{+} \lb{\val}_{\branch}
           + \sum_{\branch \in \branches[t]_{\pm i}} \sumcoaliw{\branch}\lb{\val}_{\branch}
        \leq 
        \sum_{S \in \otherfeatures} \coaliw(|S|) \val(S) 
        \leq 
        \sum_{\branch \in \branches[t]_{-i}} \sumcoaliw{\branch}^{+} \ub{\val}_{\branch}
           + \sum_{\branch \in \branches[t]_{\pm i}} \sumcoaliw{\branch}\ub{\val}_{\branch}.
    \end{equation*}
\end{description}
Inserting these bounds into \cref{eqn:multishap-proof-one} gives us \cref{theory:multi-shap-bounds}, i.e.,
\begin{align*}
  \shapval 
  & \geq 
  \sum_{\branch \in \branches[t]_{+i}} \sumcoaliw{\branch}^{-} \lb{\val}_{\branch}
     + \sum_{\branch \in \branches[t]_{\pm i}} \sumcoaliw{\branch}\lb{\val}_{\branch}
     - \sum_{\branch \in \branches[t]_{-i}} \sumcoaliw{\branch}^{+} \ub{\val}_{\branch}
     - \sum_{\branch \in \branches[t]_{\pm i}} \sumcoaliw{\branch}\ub{\val}_{\branch} \\
  &=
  \sum_{\branch \in \branches[t]_{+i}} \sumcoaliw{\branch}^{-} \lb{\val}_{\branch}
     + \sum_{\branch \in \branches[t]_{\pm i}} \sumcoaliw{\branch}(\lb{\val}_{\branch} - \ub{\val}_{\branch})
     - \sum_{\branch \in \branches[t]_{-i}} \sumcoaliw{\branch}^{+} \ub{\val}_{\branch} = \lb{\boldsymbol{\varphi}}_i^{(t)} \\
  \shapval
  & \leq 
  \sum_{\branch \in \branches[t]_{+i}} \sumcoaliw{\branch}^{-} \ub{\val}_{\branch}
     + \sum_{\branch \in \branches[t]_{\pm i}} \sumcoaliw{\branch}\ub{\val}_{\branch}
     - \sum_{\branch \in \branches[t]_{-i}} \sumcoaliw{\branch}^{+} \lb{\val}_{\branch}
     - \sum_{\branch \in \branches[t]_{\pm i}} \sumcoaliw{\branch}\lb{\val}_{\branch} \\
  &=
  \sum_{\branch \in \branches[t]_{+i}} \sumcoaliw{\branch}^{-} \ub{\val}_{\branch}
     + \sum_{\branch \in \branches[t]_{\pm i}} \sumcoaliw{\branch}(\ub{\val}_{\branch} - \lb{\val}_{\branch})
     - \sum_{\branch \in \branches[t]_{-i}} \sumcoaliw{\branch}^{+} \lb{\val}_{\branch} = \ub{\boldsymbol{\varphi}}_i^{(t)}.
\end{align*}

$\qedsymbol{}$

\subsection{Termination}\label{appendix:proofs-termination}
\begin{repeatedtheory}[\Cref{theory:termination}]\label{appendix:proofs-termination-one}
    \TerminationTheorem{}
\end{repeatedtheory}

\emph{Proof.} 
To show that \OurAlgo{} must eventually terminate, we note that in iteration~$t=2^{n-1}$, each branch contains only a single set of features~\(S \in \otherfeatures\).
Since we compute~\([\lb{\contrib}_{\branch}, \ub{\contrib}_{\branch}]\) using bound propagation, by the assumptions we pose in \cref{sec:bg-nnver} on bound propagation methods, for each branch~\(\branch \in \branches[t]\) we have~\(\lb{\contrib}_{\branch} = \ub{\contrib}_{\branch} = {\contrib}(S)\) where~\(\branch = \{S\}\).
Now, as a direct result of \cref{theory:soundness}, we have~\(\lb{\shapval}^{(t)} = \ub{\shapval}^{(t)} = \shapval\).
Hence, \OurAlgo{} terminates with~\(\lb{\shapval}^{(2^{n-1})} = \ub{\shapval}^{(2^{n-1})} = \shapval\).

$\qedsymbol{}$

\begin{repeatedtheory}[\Cref{theory:arbitrary-precision}]\label{appendix:proofs-termination-two}
    \ArbitraryPrecisionCorollary%
\end{repeatedtheory}

\emph{Proof.} From \cref{theory:termination} we know that the algorithm reaches the exact SHAP value at some iteration, hence it is also guaranteed that for any given precision tolerance~$\delta\in\RealsNonNeg$, there exists some iteration over which the computed bounds satisfy~\(\ub{\shapval}^{(t)} -\lb{\shapval}^{(t)} \leq \delta\).
Hence, at this iteration, we can terminate the algorithm and return the computed bounds up to this precision.

$\qedsymbol{}$

\begin{repeatedtheory}[\Cref{theory:linear-shap}]\label{appendix:proofs-linear-model}
    \LinearModelProposition{}%
\end{repeatedtheory}

{%
\emph{Proof.} 
\RenewDocumentCommand{\assign}{m m m o}{{[{#1}_{#2};{#3}_{\Bar{#2}}]}_{\IfNoValueF{#4}{#4}}}
Let~\(\NN(\x) = \transp{\vec{w}}\x + \vec{b}\).
We follow the arguments provided by~\citet[Corollary 1]{lundberg2017unified} to show that the contribution~\(\contrib(S)\) is constant, or, in other words, the value of~\(\contrib(S)\) is independent of~\(S\).
First, by the linearity of~\(\NN\), we obtain~\(\val(S) = \transp{\vec{w}} \mathbb{E}_{\z \in \mathcal{D}} \assign{\x}{S}{\z}  + \vec{b}\), where~\(\assign{\x}{S}{\z}\) denotes a vector where the features in~$S$ are taken from the corresponding features in~\(\x\) and those in~\(\Bar{S} = \upto{n} \setminus S\) are taken from~\(\z\).
In consequence,
\begin{align*}
    \contrib(S) 
    &= \val(S \cup \{i\}) - \val(S) \\
    &= \transp{\vec{w}} \mathbb{E}_{\z \in \mathcal{D}} [\x[S \cup \{i\}];\z_{\overbar{S \cup \{i\}}}] + \vec{b} - \transp{\vec{w}} \mathbb{E}_{\z \in \mathcal{D}} \assign{\x}{S}{\z} - \vec{b} \\
    &= \transp{\vec{w}} (\mathbb{E}_{\z \in \mathcal{D}} [\x[S \cup \{i\}];\z_{\overbar{S \cup \{i\}}}] - \mathbb{E}_{\z \in \mathcal{D}} \assign{\x}{S}{\z}) \\
    &= \sum_{j=1}^n \vec{w}_{j}(\mathbb{E}_{\z \in \mathcal{D}} [\x[S \cup \{i\}];\z_{\overbar{S \cup \{i\}}}]_j - \mathbb{E}_{\z \in \mathcal{D}} \assign{\x}{S}{\z}[j]) \\
    &= \vec{w}_{i}(\x[i] - \mathbb{E}_{\z \in \mathcal{D}} \z_i).
\end{align*}
By propagating linear functions as bounds, \LBP{} methods do not suffer from approximation error for compositions of linear functions.
Since~\(\contrib\) is a composition of linear functions, \LBP{} produces~\(\lb{\contrib}_{\branch} = \ub{\contrib}_{\branch} = \vec{w}_{i}(\x[i] - \mathbb{E}_{\z \in \mathcal{D}} \z_i)\) for any branch~\(\branch\), including the initial branch~\(\branch[1][1] = \otherfeatures\).
Therefore,
\begin{align*}
    \lb{\shapval}^{(1)} 
    = \sumcoaliw{\branch[1][1]} \lb{\contrib}_{\branch[1][1]} 
    = \vec{w}_{i}(\x[i] - \mathbb{E}_{\z \in \mathcal{D}} \z_i)
    = \sumcoaliw{\branch[1][1]} \ub{\contrib}_{\branch[1][1]} 
    = \ub{\shapval}^{(1)},
\end{align*}
and since~\(\shapval = \vec{w}_{i}(\x[i] - \mathbb{E}_{\z \in \mathcal{D}} \z_i)\)~\citep[Corollary 1]{lundberg2017unified}, this proves that~\(\lb{\shapval}^{(1)} = \ub{\shapval}^{(1)} = \shapval\).

$\qedsymbol{}$
}%

\section{Splitting Strategies}\label{appendix:additional-splitting-strategies}
This section provides additional details on the splitting strategies introduced in \cref{sec:heuristics}.
\Cref{appendix:ablations} provides an experimental comparison of both the \AlgoSelect{} and \AlgoSplit{} strategies presented below.

\paragraph{Selecting Branches to Split (\AlgoSelect{}).}
The partitioning strategy in \cref{sec:method} requires selecting a branch~\(\branch[k]\) to split.
Since we process branches in batches, we instead select a batch~\(\branch[k_1], \ldots, \branch[k_b]\) of branches to split simultaneously, where~\(b \in \Nats\) is the batch size.
We present two strategies for selecting branches to split: \SelectMaxDiam{} selects the~\(b\) branches~\(\branch\) in~\(\branches[t]\), where~\(\branchdiam := \sumcoaliw{\branch}(\ub{\contrib}_{\branch} - \lb{\contrib}_{\branch})\) is the largest, while \SelectMinDiam{} selects the~\(b\) branches with the smallest~\(\branchdiam\).
The motivation behind \SelectMaxDiam{} is to greedily select those branches that contribute the most to the gap between~\(\lb{\shapval}\) and~\(\ub{\shapval}\).
Conversely, \SelectMinDiam{} seeks to select those branches first that are likely to be pruned soon, in order to reduce the memory demands of \OurAlgo{}.

\paragraph{Selecting Features to Split on (\AlgoSplit{}).}
The branch-and-bound literature provides numerous strategies for selecting features to split, which can also be applied in \OurAlgo{}.
We adapt strong branching~\citep{applegate1995strongbranching}, smart branching~\citep{bunel2020branch,boetius2025solving}, and smears~\citep{kearfott1996smears,wang2018efficient} to our setting and compare the strategies experimentally in \cref{appendix:ablations}.
In the following, let~\(\branch \in \branches[t]\) be defined by~\(\includeset,\excludeset \subseteq \upto{n} \setminus \{i\}\).
We select a feature~\(j^\ast\) to split from the set of features~\(\splittable := \upto{n} \setminus (\{i\} \cup \includeset \cup \excludeset)\) that are available for splitting.

\begin{itemize}
    \item \textbf{\SplitInOrder{}.} A \naive{} approach for selecting features to split is selecting the features in a predefined order, e.g.,~\(j^\ast := \min \splittable\). 
    \SplitInOrder{} provides a baseline for evaluating other strategies.
    \item \textbf{\SplitStrongBranching{}.} This more powerful splitting strategy greedily selects those features that lead to the smallest diameter of the bounds on the contribution~\(\lb{\contrib}_{\branch}, \ub{\contrib}_{\branch}\) across the branches resulting from the split.
    Formally, if~\([\lb{\contrib}'_{j}, \ub{\contrib}'_{j}]\) and~\([\lb{\contrib}''_{j}, \ub{\contrib}''_{j}]\) are the contribution bounds of the two branches resulting from splitting feature~\(j\), \SplitStrongBranching{} selects~\(j^\ast := \argmin_{j} \max(\ub{\contrib}'_{j} - \lb{\contrib}'_{j}, \ub{\contrib}''_{j} - \lb{\contrib}''_{j})\) for splitting.
    \SplitStrongBranching{} is a comparatively expensive strategy, since it has to compute bounds on the contribution for each possible split.
    \item \textbf{\SplitSmartBranchingIBP{}.} This splitting strategy aims to approximate \SplitStrongBranching{}, while reducing its computational cost. 
    It is an instance of a \emph{smart branching} strategy.
    In particular, \SplitSmartBranchingIBP{} applies the smart branching strategy of~\citet{boetius2025solving} that evaluates all possible splits in the same manner as \SplitStrongBranching{}, but uses a cheaper bound propagation method (IBP) for bounding the contributions.
    \item \textbf{\SplitSmears{}.} The magnitude of the contribution gradient~\(\contribgrad(S) \in \Reals^{n-1}\) gives an indication along which features the contribution varies most. 
    \SplitSmears{} leverages this information for determining which feature to split.
    For the branch~\(\branch\), \SplitSmears{} selects~\(j^\ast := \argmax_{j} \ub{|\contribgrad|}_j\), where~\(\ub{|\contribgrad|} \geq |\contribgrad(S)|\) is computed using \IBP{}.
\end{itemize}

\section{Additional Details on \OurAlgo{}}\label{appendix:algorithm-additional}
This section complements \cref{sec:impl,sec:heuristics,appendix:additional-splitting-strategies} by providing further details on \begin{enumerate*}[\itshape(i)]
    \item relaxing the Boolean masks~\(\mask\) to bounds over a continuous domain~\([\lb{\mask}, \ub{\mask}]\), and
    \item performing bound propagation on the marginal value function.
\end{enumerate*}
For ease of exposition, we describe these aspects for computing bounds on a single SHAP value~\(\shapval\), while our discussion extends naturally to simultaneously bounding all SHAP values~\(\shapval[1], \ldots, \shapval[n]\), as described in \cref{sec:impl}.

\paragraph{Relaxing Boolean Masks to Bounds.}
As described in \cref{sec:impl}, we represent sets of features, or coalitions,~\(S \in \otherfeatures\) using Boolean masks~\(\mask \in \Bools^{n}\) for computing bounds on the contribution~\(\contrib: \otherfeatures \to \Reals\). 
We represent branches~\(\branch \subseteq \otherfeatures\) using bounds on the Boolean masks, i.e.~\([\lb{\mask}, \ub{\mask}] \subseteq \Reals^{n}\), where~\(\lb{\mask}, \ub{\mask} \in \Bools^{n}\).
Performing bound propagation on~\(\contrib\) given~\([\lb{\mask}, \ub{\mask}]\) corresponds to relaxing the Boolean masks to the continuous domain~\({[0, 1]}^{n}\).
For this purpose, we define~\(\maskcontrib: {[0, 1]}^{n} \to \Reals\) so that~\(\maskcontrib(\mask) = \contrib(S)\) if~\(\mask\) corresponds to~\(S\).

Most operations performed by~\(\contrib\) can equally be applied by~\(\maskcontrib\).
However,~\(\contrib\) also relies on computing~\(S \cup \{i\}\) in \cref{eqn:contrib} and the vector~\((\x[S];\z_{\Bar{S}})\) that is part of the marginal value function.
We now define equivalent operations for Boolean masks that are used by~\(\maskcontrib\) in place of these operations.
Let~\(\mask\) be the mask vector corresponding to~\(S\).
First, for~\(S \cup \{i\}\), we insert a~\(1\) into~\(\mask\) at index~\(i\) to obtain~\(\mask^{+i}\), which corresponds to~\(S \cup \{i\}\).
This is a linear operation that is similarly applicable to real-valued vectors.
Second, we use that we already identified the Booleans with~\(\Bools\), so that~\((\x[S];\z_{\Bar{S}}) = \mask * \x + (1 - \mask) * \z\), where~\(*\) denotes element-wise multiplication. 
Overall, we have~\(\maskcontrib(\mask) = \maskval(\mask^{+i}) - \maskval(\mask)\), where~\(\maskval(\mask) = \mathbb{E}_{z \sim \mathcal{D}}[\NN(\mask * \x + (1 - \mask) * \z)]\).
Therefore, by relaxing the Boolean masks to~\({[0, 1]}^{n-1}\), we also consider linear combinations of~\(\x\) and~\(\z\).
However, since all operations besides~\(\NN\) are linear, this does not affect the tightness of linear bound propagation techniques as shown in \cref{theory:linear-shap}.

\paragraph{Propagating Bounds through the Marginal Value Function.}
The marginal value function as defined in \cref{sec:bg-shap} is the expected value of the neural network output, typically over a dataset of background samples~\(\mathcal{D}\).
We now discuss performing bound propagation on this expected value.
Since, practically,~\(\mathcal{D}\) is a finite dataset, we have~\(\val(S) = \mathbb{E}_{\z \sim \mathcal{D}}[\NN(\x[S];\z_{\Bar{S}})] = 1 / |\mathcal{D}| \sum_{\z \in \mathcal{D}}\NN(\x[S];\z_{\Bar{S}})\), where~\(|\mathcal{D}|\) is the size of the background dataset.
Since computing the average is a linear function of the network output, bound propagation proceeds as described in \cref{sec:bg-nnver} for linear functions.
Concretely, if~\([\lb{\NN}_\z, \ub{\NN}_\z]\) are bounds on~\(\NN(\x[S];\z_{\Bar{S}})\) for each~\(\z \in \mathcal{D}\), we have
\begin{equation*}
    \frac{1}{|\mathcal{D}|} \sum_{\z \in \mathcal{D}} \lb{\NN}_\z \leq \val(S) \leq \frac{1}{|\mathcal{D}|} \sum_{\z \in \mathcal{D}} \ub{\NN}_\z.
\end{equation*}

\begin{table}
    \NewDocumentCommand{\FC}{m o}{$\text{FC}_{#1}^{\IfNoValueF{#2}{#2}}$}
    \NewDocumentCommand{\Out}{m}{$\text{Out}_{#1}$}
    \NewDocumentCommand{\Conv}{m m}{$\text{Conv}_{#1}({#2})$}
    \NewDocumentCommand{\AvgPool}{m}{$\text{AvgPool}({#1})$}
    \NewDocumentCommand{\BN}{}{$\text{BN}$}
    \NewDocumentCommand{\SkipConnection}{}{$\text{SkipConnection}$}
    \NewDocumentCommand{\ReLU}{}{$\text{ReLU}$}
    \centering
    \caption{
        \textbf{Datasets and Neural Networks used in our Experiments.}
        \enquote{Architecture} lists the neural network used on each dataset, specified as the sequence of layers.
        We use~\FC{k}[\sigma] to denote a $\sigma$-activated fully-connected layer of~\(k\) neurons. If~\(\sigma\) is omitted, the layer is ReLU-activated.
        Similarly, \Out{m} denotes the linear output layer of dimension~\(m\),~\Conv{c}{k_1\times k_2} denotes a convolution layer with~\(c\) channels, a window size of~\(k_1\times k_2\), a stride of one, and \enquote{same} padding,~\AvgPool{k_1\times k_2} denotes an average pooling layer with a window size of~\(k_1\times k_2\), strides of~\(k_1\) and~\(k_2\), and \enquote{valid} padding,~\SkipConnection{} denotes a residual connection adding the input of the previous layer to the output of the layer,~\BN{} denotes batch normalisation, and~\ReLU{} denotes a ReLU activation layer for convolutional neural networks.
        For classification tasks, \enquote{Performance} specifies accuracy.
        For regression tasks, \enquote{Performance} provides the root mean square error.
    }\label{tab:datasets-and-networks}
    \begin{tabular}{@{}clL{3.4cm}cC{.95cm}C{.95cm}@{}}
             & & & & \multicolumn{2}{c}{\,\,\textbf{Performance}} \\
         $n$ & \multicolumn{1}{c}{\textbf{Dataset}} & \multicolumn{1}{c}{\textbf{Architecture}} & \textbf{Task}  & \textbf{Train} & \textbf{~Test~} \\ \midrule
        $12$ & \AdultDataset{}~\citep{AdultDataset}                    & \FC{32}, \FC{32}, \Out{2}          & Classification & $85\%$  & $85\%$ \\
        $16$ & \ObesityDataset{}~\citep{ObesityDataset}                & \FC{32}, \FC{32}, \Out{7}          & Classification & $100\%$ & $95\%$ \\
        $20$ & \GermanDataset{}~\citep{GermanDataset}                  & \FC{8}, \Out{1}                    & Classification & $81\%$  & $77\%$ \\
        $23$ & \DefaultDataset{}~\citep{DefaultDataset}                & \FC{64}, \FC{64}, \FC{64}, \Out{1} & Classification & $90\%$  & $71\%$ \\
        $25$ & \AutomobileDataset{}~\citep{AutomobileDataset}          & \FC{32}, \FC{32}, \Out{1}          & Regression     & $0.14$  & $0.59$ \\
        $27$ & \SteelDataset{}~\citep{SteelDataset}                    & \FC{8}, \FC{8}, \Out{7}            & Classification & $78\%$  & $72\%$ \\
        $28$ & \HepatitisCDataset{}~\citep{HepatitisCDataset}          & \FC{8}, \FC{8}, \Out{4}            & Classification & $53\%$  & $25\%$ \\
        $30$ & \BreastCancerDataset{}~\citep{BreastCancerDataset}      & \FC{32}, \FC{32}, \Out{1}          & Classification & $100\%$ & $97\%$ \\
        $38$ & \AnnealingDataset{}~\citep{AnnealingDataset}            & \FC{32}, \FC{32}, \Out{1}          & Classification & $100\%$ & $99\%$ \\
        $56$ & \LungCancerDataset{}~\citep{LungCancerDataset}          & \FC{4}, \FC{4}, \Out{3}            & Classification & $84\%$  & $57\%$ \\
        $58$ & \OnlineNewsDataset{}~\citep{OnlineNewsDataset}          & \FC{64}, \FC{64}, \Out{1}          & Classification & $67\%$  & $66\%$ \\
        $60$ & \SonarDataset{}~\citep{SonarDataset}                    & \FC{32}, \FC{32}, \Out{1}          & Classification & $100\%$ & $100\%$ \\ 
        \midrule
        $22$ & \MushroomDataset{}~\citep{MushroomDataset} & \FC{1}, \Out{1}                    & Classification & $96\%$ & $95\%$ \\
        $22$ & \MushroomDataset{}~\citep{MushroomDataset} & \FC{2}, \Out{1}                    & Classification & $96\%$ & $96\%$ \\
        $22$ & \MushroomDataset{}~\citep{MushroomDataset} & \FC{4}, \Out{1}                    & Classification & $99\%$ & $99\%$ \\
        $22$ & \MushroomDataset{}~\citep{MushroomDataset} & \FC{8}, \Out{1}                    & Classification & $100\%$ & $100\%$ \\
        $22$ & \MushroomDataset{}~\citep{MushroomDataset} & \FC{16}, \Out{1}                    & Classification & $100\%$ & $100\%$ \\
        $22$ & \MushroomDataset{}~\citep{MushroomDataset} & \FC{32}, \Out{1}                    & Classification & $100\%$ & $100\%$ \\
        $22$ & \MushroomDataset{}~\citep{MushroomDataset} & \FC{64}, \Out{1}                    & Classification & $100\%$ & $100\%$ \\
        $22$ & \MushroomDataset{}~\citep{MushroomDataset} & \FC{128}, \Out{1}                    & Classification & $100\%$ & $100\%$ \\
        $22$ & \MushroomDataset{}~\citep{MushroomDataset} & \FC{256}, \Out{1}                    & Classification & $100\%$ & $100\%$ \\
        $22$ & \MushroomDataset{}~\citep{MushroomDataset} & \FC{512}, \Out{1}                    & Classification & $100\%$ & $100\%$ \\
        $22$ & \MushroomDataset{}~\citep{MushroomDataset} & \FC{1024}, \Out{1}                    & Classification & $100\%$ & $100\%$ \\
        $22$ & \MushroomDataset{}~\citep{MushroomDataset} & \FC{2048}, \Out{1}                    & Classification & $100\%$ & $100\%$ \\
        $22$ & \MushroomDataset{}~\citep{MushroomDataset} & \FC{4096}, \Out{1}                    & Classification & $100\%$ & $100\%$ \\
        $22$ & \MushroomDataset{}~\citep{MushroomDataset} & \FC{8192}, \Out{1}                    & Classification & $100\%$ & $100\%$ \\
        $22$ & \MushroomDataset{}~\citep{MushroomDataset} & \FC{16384}, \Out{1}                    & Classification & $100\%$ & $100\%$ \\
        $22$ & \MushroomDataset{}~\citep{MushroomDataset} & \FC{32768}, \Out{1}                    & Classification & $100\%$ & $100\%$ \\
        \midrule
        $22$ & \MushroomDataset{}~\citep{MushroomDataset} & \FC{8}[\tanh], \Out{1}               & Classification & $98\%$ & $98\%$ \\
        $22$ & \MushroomDataset{}~\citep{MushroomDataset} & \FC{8}[\mathrm{swish}], \Out{1}      & Classification & $100\%$ & $100\%$ \\
        $22$ & \MushroomDataset{}~\citep{MushroomDataset} & ResNet\({}^{\ast}\)                  & Classification & $100\%$ & $100\%$ \\
        \midrule
        $784$ & \MNISTDataset{}~\citep{lecun1998mnist}                 & ConvNet\({}^{\dagger}\)  & Classification & $98.9\%$ & $98.6\%$ \\
        $784$ & \FashionMNISTDataset{}~\citep{xiao2017fashionmnist}    & ConvNet\({}^{\dagger}\)  & Classification & $97.0\%$ & $89.0\%$ \\
        $3072$ & \CIFARTenDataset{}~\citep{krizhevsky2009cifar10}       & ConvNet\({}^{\ddagger}\) & Classification & $87.3\%$ & $68.7\%$ \\
        $3072$ & \GTSRBDataset{}~\citep{stallkamp2011gtsrb}             & ConvNet\({}^{\star}\)    & Classification &$100\%$ & $90.5\%$ \\
    \end{tabular}
    
    \vspace*{.2cm}
    \({}^{\ast}\)\FC{22}, \SkipConnection{}, \FC{22}, \SkipConnection{}, \FC{22}, \SkipConnection{}, \Out{1} \\
    \({}^{\dagger}\)\Conv{4}{5\times5}, \AvgPool{2\times2}, \BN, \ReLU, \Conv{8}{5\times5}, \AvgPool{2\times2}, \BN, \ReLU, \FC{392}, \FC{64}, \Out{10} \\
    \({}^{\ddagger}\)\Conv{16}{5\times5}, \AvgPool{2\times2}, \BN, \ReLU, \Conv{32}{5\times5}, \AvgPool{2\times2}, \BN, \ReLU, \FC{2048}, \FC{256}, \Out{10} \\
    \({}^{\star}\)\Conv{32}{5\times5}, \AvgPool{2\times2}, \BN, \ReLU, \Conv{64}{5\times5}, \AvgPool{2\times2}, \BN, \ReLU, \FC{4096}, \FC{256}, \Out{10}
\end{table}

\section{Datasets and Networks}\label{appendix:datasets-networks-details}
\Cref{tab:datasets-and-networks} summarises the datasets and neural networks we use in our experiments.

\paragraph{Neural Networks.}
For all tabular datasets, we train fully connected neural networks, manually tuning the architecture to achieve maximal test set performance.
On the \MushroomDataset{} dataset, we train neural networks with different architectures to demonstrate the versatility of \OurAlgo{} in \cref{sec:network-architectures}. Additionally, we train neural networks of increasing size on \MushroomDataset{} to test the scalability of \OurAlgo{} with respect to the network size in \cref{appendix:experiments-network-size}.
The main \MushroomDataset{} neural network used in our remaining experiments contains a single hidden layer of~\(8\) ReLU neurons.
We train convolutional neural networks on all vision datasets, using a standard convolutional architecture, and manually tune the network size to achieve good test set performance.

\paragraph{Datasets.}
All tabular datasets are taken from the UCI ML Repository~\citep{UCIMLRepo}, which provides concise descriptions of each dataset.
\MNISTDataset{} is the standard, widely known hand-written digit recognition dataset introduced by~\citet{lecun1998mnist}.
\FashionMNISTDataset{}~\citep{xiao2017fashionmnist} is a more challenging dataset in the same format as \MNISTDataset{}, concerned with classifying images of fashion items.
\CIFARTenDataset{}~\citep{krizhevsky2009cifar10} is concerned with recognising vehicles and animals in~\(32 \times 32\) pixel colour images.
\GTSRBDataset{}~\citep{stallkamp2011gtsrb} contains colour images of German traffic signs scaled to the same size as \CIFARTenDataset{}.

\section{Additional Experiments}\label{appendix:additional-experiments}
This appendix contains additional experiments on different bound propagation algorithms, branch selection strategies, and split selection strategies for \OurAlgo{} in \cref{appendix:ablations}, quantifying the effect of different sources of bound looseness in \OurAlgo{} in \cref{appendix:experiments-relaxations-impact}, analysing the scalability of \OurAlgo{} in the network size in \cref{appendix:experiments-network-size}, further results on image datasets in \cref{appendix:mnist-additional-statistics,appendix:experiments-additional-image-datasets}, extended statistics for our comparison to \ExactSHAP{} in \cref{appendix:experiments-exactshap-additional}, results for additional statistical SHAP estimators in \cref{appendix:experiments-additional-estimators,appendix:additional-compare-estimators}, and an experiment on predicting the runtime of \OurAlgo{} in \cref{appendix:predict-runtime}.
Below, we first provide additional details on the value functions, summarise the metrics used in our experiments, and provide additional details on our hardware setup and our hyper-parameter values for \OurAlgo{}.

\paragraph{Value Functions.}
Our experiments leverage three variants of the marginal value function introduced in \cref{sec:bg-shap}.
In \cref{sec:experiments-compare-to-exactshap,sec:experiments-compare-to-kernelshap,appendix:additional-compare-estimators,appendix:ablations}, we study the marginal value function~\(\val_{\text{marginal}}(S) := \mathbb{E}_{\z \sim \mathcal{D}}[\NN(\x[S];\z_{\Bar{S}})]\), where~\(\mathcal{D}\) is a background dataset of~\(100\) samples.
We additionally use the \emph{baseline} value function~\(\val_{\text{baseline}}(S) := \NN(\x[S];\z_{\Bar{S}})\) for the \emph{baseline input}~\(\z \in \Reals^n\) that corresponds to the marginal value function under a singleton background dataset.
Concretely, we use the \emph{zero-baseline} value function, where~\(\z = \mathbf{0}_n\), in \cref{sec:experiments-mnist,appendix:experiments-additional-image-datasets} and the \emph{mean-baseline} value function, where~\(\z\) is the input mean over the training dataset, in \cref{sec:estimators-comparison,appendix:experiments-additional-image-datasets,appendix:additional-compare-estimators}.

\paragraph{Evaluation Metrics.}
We frequently compare SHAP bounds to the attributed neural network output to assess the tightness of the bounds.
Given a multi-class classifier or multi-regressor~\(\NN: \Reals^n \to \Reals^m\), one output~\(k \in \upto{m}\) is selected to be attributed.
Given SHAP bounds~\([\lb{\shapval}, \ub{\shapval}]\), we evaluate, for example, whether the bounds half-range~\((\ub{\shapval} - \lb{\shapval}) / 2\) is at most~\(p\%\) of the network output to attribute~\({\NN(\x)}_k\) (\enquote{$p\%$ HR}).
Furthermore, we rely on the maximum (\(\ell_\infty\)) error between estimated SHAP values~\(\hat{\boldsymbol{\varphi}}\) and the exact SHAP values~\(\boldsymbol{\varphi}\), which is~\(\max_{i} |\hat{\boldsymbol{\varphi}}_i - \boldsymbol{\varphi}_i|\).
For better comparability with~\citet{witter2025regression}, we also report the mean squared error~\(\frac{1}{n}\|\hat{\boldsymbol{\varphi}} - \boldsymbol{\varphi}\|_2^2\).

\paragraph{Hardware \& Hyperparameters.}
Our main hardware is an L40S NVIDIA GPU with 48GB of GPU memory, deployed on an Ubuntu 24.04 machine with 188GB of RAM and an AMD EPYC 9375F 32-Core processor.
Unless otherwise noted, we run \OurAlgo{} with a batch size of~\(4096\).

\subsection{Ablation Studies}\label{appendix:ablations}
We compare different bound propagation algorithms, branch selection strategies, and split selection strategies. 
For this comparison, we apply \OurAlgo{} to marginal SHAP with~\(100\) background data points on a selection of the datasets from \cref{sec:experiments-compare-to-exactshap}. Concretely, we use the \GermanDataset{}, \MushroomDataset{}, \DefaultDataset{}, \AutomobileDataset{}, \SteelDataset{}, and \SonarDataset{} datasets.
In this experiment, we compute SHAP values only for the first test sample and set a timeout of~\(400\)s.

\paragraph{Bound Propagation Algorithms.}
We compare the \IBP~\citep{moore2009introduction}, \CROWN{}~\citep{zhang2018crown}, \CROWNIBP{}~\citep{zhang2020crown-ibp}, and \AlphaCROWN{}~\citep{xu2021alphacrown} bound propagation algorithms for computing bounds on the value function in \OurAlgo{}. 
In this comparison, we use the \SelectMaxDiam{} branch selection and \SplitSmears{} split selection strategies.
Due to the increased memory requirements of \CROWN{} for recursively propagating bounds, \CROWN{} exceeds the memory of our GPU for all networks we use in this comparison. For this reason, \CROWN{} is excluded from the remainder of this comparison.

\Cref{tab:bound-propagation-ablation} contains the results of this comparison.
\IBP{} is fastest for computing exact SHAP values, but reaches the timeout or requires more runtime to compute tight bounds on the higher-dimensional datasets. 
While \AlphaCROWN{} computes the tightest contribution bounds, it is more runtime-intensive than \CROWNIBP{} and \IBP{}, causing it to require more runtime both for computing exact SHAP values and tight bounds.
Furthermore, the additional runtime cost of \AlphaCROWN{} means that it cannot perform as many iterations as \CROWNIBP{} and \IBP{} within the timeout.
For this reason, the final bounds that \OurAlgo{} computes on the \SteelDataset{} and \SonarDataset{} datasets are less tight when using \AlphaCROWN{} than for \CROWNIBP{}, while still significantly tighter than for \IBP{}.
Concretely, on \SteelDataset{}, the half-range of the SHAP value bounds is~\(1.52\%\) of the network output to attribute when using \CROWNIBP{},~\(1.95\%\) for \AlphaCROWN{}, and~\(17.70\%\) for \IBP{}.
On \SonarDataset{}, the half-range is~\(1.13\%\) for \CROWNIBP{},~\(1.37\%\) for \AlphaCROWN{}, and~\(28.77\%\) for \IBP{}.
Overall, the bound propagation algorithm~\CROWNIBP{} that we apply in our experiments performs best for computing tight bounds for higher-dimensional datasets.

\paragraph{Branch Selection.}
Another central component of \OurAlgo{} is the strategy used to select the branches to split. 
We compare the \SelectMaxDiam{} and \SelectMinDiam{} strategies introduced in \cref{sec:heuristics}, while using \CROWNIBP{} as the bound propagation algorithm and \SplitSmears{} for selecting splits.

The results of our comparison are summarised in \cref{tab:branch-selection-ablation}.
While \SelectMinDiam{} reduces the runtime for computing exact SHAP values, it increases the runtime required to obtain tight bounds compared to \SelectMaxDiam{}.
Since the inherent computational cost of \SelectMinDiam{} and \SelectMaxDiam{} is the same, the runtime differences we observe for computing exact SHAP values require further investigation, which we leave for future work.
In our experiments, we rely on \SelectMaxDiam{}, as it enables \OurAlgo{} to compute tight bounds faster on most datasets.

\paragraph{Split Selection.}
Finally, we compare the \SplitInOrder{}, \SplitSmears{}, \SplitSmartBranchingIBP{}, and \SplitStrongBranching{} split selection strategies introduced in \cref{appendix:additional-splitting-strategies}.
For this comparison, we use \CROWNIBP{} as the bound propagation algorithm and \SelectMaxDiam{} for selecting branches.

\Cref{tab:split-selection-ablation} summarises the results of this comparison.
While \SplitInOrder{} allows \OurAlgo{} to compute exact SHAP values the fastest in most cases, due to its minimal runtime requirements, \SplitInOrder{} does not allow for computing tight bounds on the \SteelDataset{} and \SonarDataset{} datasets.
Both \SplitSmartBranchingIBP{} and \SplitStrongBranching{} only produce results on the lower-dimensional \GermanDataset{} and \MushroomDataset{} datasets, due to their high runtime requirements.
\SplitSmears{} is the only strategy that allows for computing tight bounds on the \SteelDataset{} and \SonarDataset{} datasets.
Furthermore, it offers a comparable runtime to \SplitInOrder{} for computing exact SHAP values.
For this reason, we apply \SplitSmears{} in our experiments.

\begin{table}[p]
    \caption{%
        \textbf{Bound Propagation Algorithms.}
        This table contains the runtime required by \OurAlgo{} for computing SHAP values (\enquote{Exact}) or the bounds half-range reaching~\(10\%\) of the network output to attribute (\enquote{10\% HR}). 
        A \enquote{--} denotes \OurAlgo{} reaching the timeout, which is~\(400\)s for this experiment.
    }\label{tab:bound-propagation-ablation}
    \centering
    \begin{tabular}{lrrrrrr}
                 & \multicolumn{3}{c}{\textbf{Exact Runtime}} & \multicolumn{3}{c}{\textbf{\(\mathbf{10\%}\) HR Runtime}} \\ \cmidrule(lr){2-4}\cmidrule(lr){5-7}
\textbf{Dataset} & \textbf{\AlphaCROWN} & \textbf{\CROWNIBP} & \textbf{\IBP} & \textbf{\AlphaCROWN} & \textbf{\CROWNIBP} & \textbf{\IBP} \\
\midrule
\GermanDataset     &  $22$s & $ 20$s &  $\mathbf{14}$\textbf{s} & $ 18$s &                  $ 16$s & $\mathbf{12}$\textbf{s} \\
\MushroomDataset   &  $28$s & $ 25$s &  $\mathbf{19}$\textbf{s} & $ 19$s &                  $ 16$s & $\mathbf{15}$\textbf{s} \\
\DefaultDataset    & $138$s & $132$s &  $\mathbf{92}$\textbf{s} & $138$s &                  $132$s & $\mathbf{92}$\textbf{s} \\
\AutomobileDataset & $317$s & $315$s & $\mathbf{257}$\textbf{s} & $ 68$s & $\mathbf{52}$\textbf{s} &                   $74$s \\
\SteelDataset      &     -- &     -- &                       -- & $114$s & $\mathbf{86}$\textbf{s} &                      -- \\
\SonarDataset      &     -- &     -- &                       -- & $ 16$s & $\mathbf{13}$\textbf{s} &                      -- 
    \end{tabular}
\end{table}

\begin{table}[p]
    \caption{%
        \textbf{Branch Selection Strategies.}
        This table contains the runtime required by \OurAlgo{} for computing SHAP values (\enquote{Exact}) or the bounds half-range reaching~\(10\%\) of the network output to attribute (\enquote{10\% HR}). 
        A \enquote{--} denotes \OurAlgo{} reaching the timeout, which is~\(400\)s for this experiment.
    }\label{tab:branch-selection-ablation}
    \centering
    \begin{tabular}{lrrrrrr}
                 & \multicolumn{2}{c}{\textbf{Exact Runtime (s)}} & \multicolumn{2}{c}{\textbf{\(\mathbf{10\%}\) HR Runtime (s)}} \\ \cmidrule(lr){2-3}\cmidrule(lr){4-5}
\textbf{Dataset} & \textbf{\SelectMaxDiam} & \textbf{\SelectMinDiam} & \textbf{\SelectMaxDiam} & \textbf{\SelectMinDiam} \\
\midrule
\GermanDataset     & $ 20$s &  $\mathbf{18}$\textbf{s} &  $\mathbf{16}$\textbf{s} &                  $ 18$s \\
\MushroomDataset   & $ 25$s &  $\mathbf{21}$\textbf{s} &  $\mathbf{16}$\textbf{s} &                  $ 21$s \\
\DefaultDataset    & $133$s & $\mathbf{110}$\textbf{s} &                   $133$s & $\mathbf{110}$\textbf{s} \\
\AutomobileDataset & $315$s & $\mathbf{190}$\textbf{s} & $ \mathbf{52}$\textbf{s} &                   $190$s \\
\SteelDataset      &     -- &                       -- & $ \mathbf{86}$\textbf{s} &                       -- \\
\SonarDataset      &     -- &                       -- &                   $ 13$s &                   $ 13$s
    \end{tabular}
\end{table}

\begin{table}[p]
    \caption{%
        \textbf{Split Selection Strategies.}
        This table contains the runtime required by \OurAlgo{} for computing SHAP values (\enquote{Exact}) or the bounds half-range reaching~\(10\%\) of the network output to attribute (\enquote{10\% HR}). 
        A \enquote{--} denotes \OurAlgo{} reaching the timeout, which is~\(400\)s for this experiment.
        In this table, \enquote{\AlgorithmName{IStB}} and~\enquote{\AlgorithmName{StB}} abbreviate~\SplitSmartBranchingIBP{} and \SplitStrongBranching{}, respectively.
    }\label{tab:split-selection-ablation}
    \centering
    \begin{tabular}{lrrrrrrrr}
                 & \multicolumn{4}{c}{\textbf{Exact Runtime (s)}} & \multicolumn{4}{c}{\textbf{\(\mathbf{10\%}\) HR Runtime (s)}} \\ \cmidrule(lr){2-5}\cmidrule(lr){6-9}
\textbf{Dataset} & \textbf{\SplitInOrder{}} & \textbf{\SplitSmears{}} & \textbf{\AlgorithmName{IStB}} & \textbf{\AlgorithmName{StB}} & \textbf{\SplitInOrder{}} & \textbf{\SplitSmears{}} & \textbf{\AlgorithmName{IStB}} & \textbf{\AlgorithmName{StB}} \\
\midrule
\GermanDataset     & $ \mathbf{18}$\textbf{s} &                   $ 20$s & $21$s & $28$s & $ \mathbf{15}$\textbf{s} &                   $ 16$s & $16$s & $19$s \\
\MushroomDataset   &                   $ 34$s & $ \mathbf{26}$\textbf{s} & $48$s & $77$s &                   $ 25$s & $ \mathbf{16}$\textbf{s} & $34$s & $50$s \\
\DefaultDataset    & $\mathbf{111}$\textbf{s} &                   $133$s &    -- &    -- & $\mathbf{111}$\textbf{s} &                   $133$s &    -- &    -- \\
\AutomobileDataset & $\mathbf{289}$\textbf{s} &                   $314$s &    -- &    -- &                   $132$s & $ \mathbf{52}$\textbf{s} &    -- &    -- \\
\SteelDataset      &                       -- &                       -- &    -- &    -- &                       -- & $ \mathbf{87}$\textbf{s} &    -- &    -- \\
\SonarDataset      &                       -- &                       -- &    -- &    -- &                       -- & $ \mathbf{13}$\textbf{s} &    -- &    -- 
    \end{tabular}
\end{table}

\subsection{Impact of Continuous Relaxation}\label{appendix:experiments-relaxations-impact}
As described in \cref{sec:impl}, \OurAlgo{} relaxes discrete masks to a continuous space in order to apply bound propagation. Therefore, the bounds computed by \OurAlgo{} can be loose for two reasons: the bound propagation and the continuous relaxation. In this section, we tell the two sources of looseness apart by computing the true range of the baseline value function~\(\val\) over the discrete coalition space and the relaxed continuous space on low-dimensional datasets. Comparing these ranges to the bounds computed by bound propagation reveals which technique is the larger source of looseness. We apply neural network verification based on mixed integer linear programming (MILP) to compute the true range over the relaxed continuous space. For the discrete space, we enumerate all coalitions to compute the true range. \Cref{fig:looseness-sources} reveals that the contribution of the continuous relaxation is marginal across the \AdultDataset{}, \ObesityDataset{}, and \GermanDataset{} datasets.

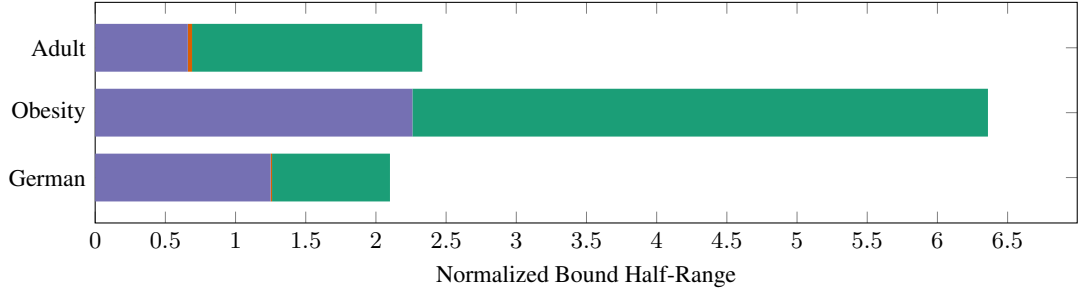
\begin{figure}
  \centering
  \tikzexternaldisable%
  \begin{tikzpicture}
    \begin{axis}[
        xbar stacked,
        bar width=18pt,
        width=.85\linewidth, height=4.5cm,
        xmin=0,
        enlarge y limits=0.35,
        symbolic y coords={German, Obesity, Adult},
        ytick=data,
        xlabel={Normalized Bound Half-Range},
        tick label style={font=\footnotesize},
        label style={font=\footnotesize},
        every axis plot/.append style={fill, draw=none},
    ]
        \addplot[xbar, fill=CatC, draw=none] table[row sep=\\, x=val, y=dataset] {
          dataset val   \\
          Adult   0.66  \\
          Obesity 2.26  \\
          German  1.25  \\
        };
        
        \addplot[xbar, fill=CatB, draw=none] table[row sep=\\, x=val, y=dataset] {
          dataset val   \\
          Adult   0.03  \\
          Obesity 0.00  \\
          German  0.01  \\
        };
        
        \addplot[xbar, fill=CatA, draw=none] table[row sep=\\, x=val, y=dataset] {
          dataset val   \\
          Adult   1.64  \\
          Obesity 4.10  \\
          German  0.84  \\
        };
    \end{axis}
  \end{tikzpicture}
  \caption[]{%
    \textbf{Sources of Looseness in the \OurAlgo{} Value Bounds.}
    This plot decomposes the initial value function~\(\val\) bounds computed by \OurAlgo{} into the true range of~\(v\) over the discrete coalition space~\legendcolorbox{CatC}, the looseness introduced by the continuous relaxation~\legendcolorbox{CatB}, and the looseness introduced by \CROWNIBP{}~\legendcolorbox{CatA}. For each dataset, we normalise the bounds by the network output to attribute. Since the looseness introduced by the continuous relaxation is marginal compared to the \CROWNIBP{} bounds for all three datasets, the \legendcolorbox{CatB} bar is barely visible in this plot.
  }\label{fig:looseness-sources}
\end{figure}

\subsection{Scalability in Network Size}\label{appendix:experiments-network-size}
In this section, we investigate the scalability of \OurAlgo{} with respect to the size of the neural network for which we compute SHAP values.
For this experiment, we train neural networks of increasing size on the \MushroomDataset{}. 
The networks are described in \cref{appendix:datasets-networks-details}.
To reduce computational demand, we use the mean-baseline value function and compute SHAP values only for the first test sample in this experiment.
To still obtain reliable runtime measurements, we repeat each run five times.

\Cref{fig:network-size-scalability} summarises the results of our comparison.
As apparent from this figure, the runtime of \OurAlgo{} scales \emph{linearly} with the network size, measured by the number of network parameters, on the \MushroomDataset{}.
This perhaps surprising result is consistent with earlier observations in the related field of probabilistic verification of neural networks~\citep{boetius2025solving}.
A key observation is that the network complexity determines the runtime of \OurAlgo{} only indirectly, via the complexity of the decision boundary.
For example, as \cref{theory:linear-shap} proves, \OurAlgo{} terminates immediately also for high-dimensional linear functions with many parameters, since no splitting is required for computing tight bounds on the contribution function.
As most networks used here achieve identical performance on \MushroomDataset{}, they likely share similar decision boundaries. 
Another indicator of the decision boundary's complexity is the number of iterations, i.e., splits, that \OurAlgo{} performs to compute the exact SHAP values. 
As \cref{fig:network-size-small-rt,fig:network-size-small-iters} show, the only substantial jump in runtime from the smallest network containing only a single neuron to the next larger network containing two neurons is connected to a jump in the number of iterations of \OurAlgo{}, which indicates changes in the complexity of the decision boundary.

\begin{figure}
    \pgfplotsset{
      x label style={at={(axis description cs:0.35,-0.2)},anchor=north},
    }
    \centering
    \begin{subfigure}{.245\textwidth}
        \centering
        \begin{tikzpicture}
            \begin{axis}[height=4cm, width=\linewidth, ylabel=runtime (s), xlabel=parameters]
                \addplot[CatA, mark=*, mark size=2pt, ultra thick] table[x=num_parameters, y=runtime, col sep=comma] {data/network_size_runtimes.csv};
            \end{axis}
        \end{tikzpicture}
        \caption{All Networks (Runtime)}\label{fig:network-size-all-rt}
    \end{subfigure}
    \begin{subfigure}{.245\textwidth}
        \centering
        \begin{tikzpicture}
            \begin{axis}[height=4cm, width=\linewidth, ylabel=iterations, xlabel=parameters]
                \addplot[CatA, mark=*, mark size=2pt, ultra thick] table[x=num_parameters, y=iterations, col sep=comma] {data/network_size_runtimes.csv};
            \end{axis}
        \end{tikzpicture}
        \caption{All Networks (Iterations)}\label{fig:network-size-all-iters}
    \end{subfigure} \hfill
    \begin{subfigure}{.245\textwidth}
        \centering
        \begin{tikzpicture}
            \begin{axis}[height=4cm, width=\linewidth, ylabel=runtime (s), xlabel=parameters, xmax=10000]
                \addplot[CatA, mark=*, mark size=2pt, ultra thick] table[x=num_parameters, y=runtime, col sep=comma] {data/network_size_runtimes.csv};
            \end{axis}
        \end{tikzpicture}
        \caption{Small Networks (Runtime)}\label{fig:network-size-small-rt}
    \end{subfigure}
    \begin{subfigure}{.245\textwidth}
        \centering
        \begin{tikzpicture}
            \begin{axis}[height=4cm, width=\linewidth, ylabel=iterations, xlabel=parameters, xmax=10000]
                \addplot[CatA, mark=*, mark size=2pt, ultra thick] table[x=num_parameters, y=iterations, col sep=comma] {data/network_size_runtimes.csv};
            \end{axis}
        \end{tikzpicture}
        \caption{Small Networks (Iterations)}\label{fig:network-size-small-iters}
    \end{subfigure}
    \caption{
      \textbf{Scalability in Network Size.} 
      We plot the runtime and the number of iterations \OurAlgo{} requires to compute the exact SHAP values for neural networks trained on the \MushroomDataset{} dataset of increasing size.
      We measure network size by the number of parameters a network possesses.
      The right plots~\subref{fig:network-size-small-rt} and~\subref{fig:network-size-small-iters} show cut-outs of the left plots~\subref{fig:network-size-all-rt} and~\subref{fig:network-size-all-iters} focussing on the results for small networks.
    }\label{fig:network-size-scalability}
\end{figure}
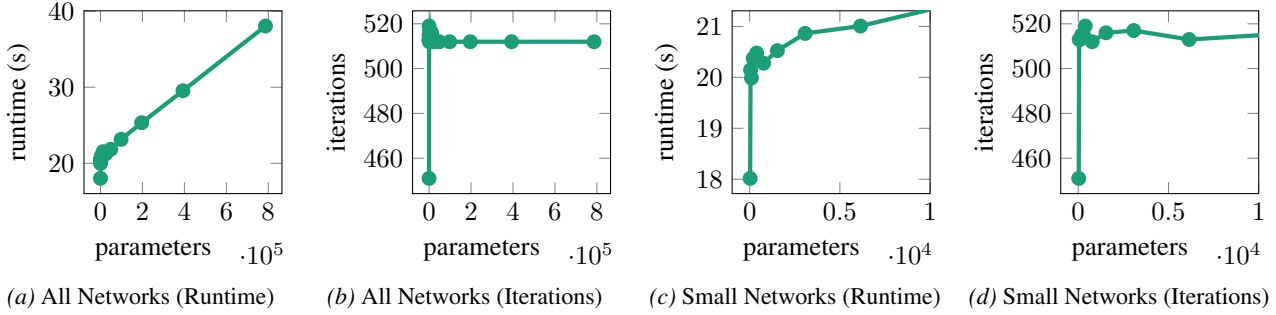

\subsection{Additional Statistics for \Cref{sec:experiments-mnist}}\label{appendix:mnist-additional-statistics}
We provide additional diagnostic statistics for our MNIST experiment presented in \cref{sec:experiments-mnist} and \cref{fig:mnist-bounds}.
\Cref{fig:bab-run-summary} complements \cref{fig:mnist-bounds-three-features} by plotting the total number of branches, not currently pruned (active) branches, and pruned branches, as well as the reasons for pruning branches and the margin of the computed SHAP bounds.
When inspected together with \cref{fig:mnist-bounds}, \cref{fig:bab-run-summary} reveals that the SHAP bounds computed by \OurAlgo{} become tight before any branches are pruned. 
Furthermore, all branches are pruned because they have tight value function bounds, rather than because they are fully split.
The minimum and average bounds margin differ only minimally from the maximal bounds margin presented in \cref{fig:bab-bounds-margin}.

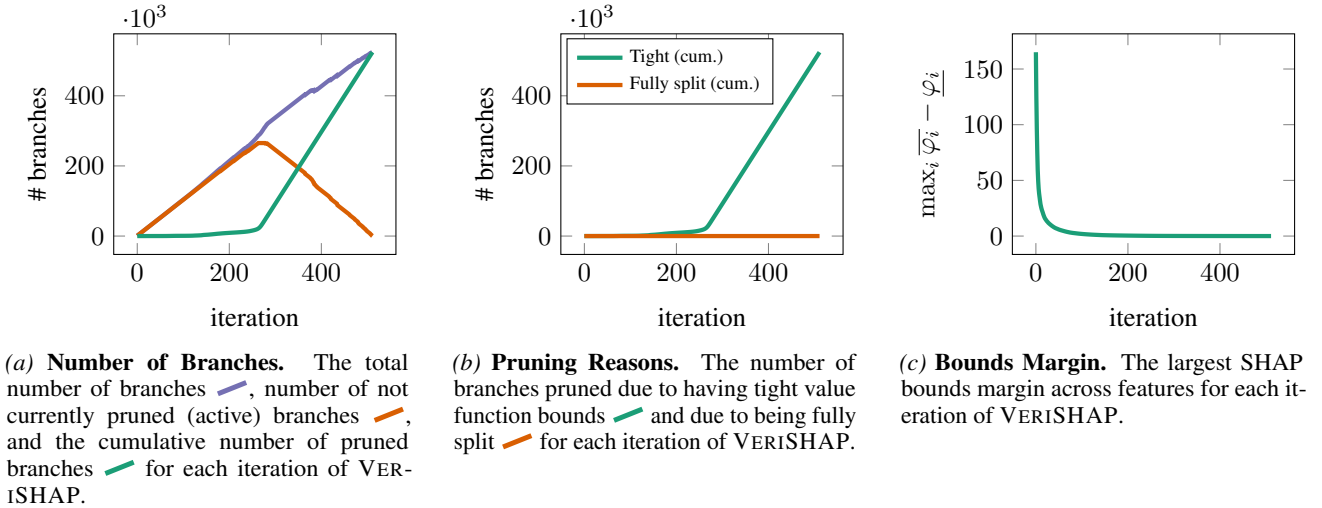
\begin{figure}
    \pgfplotsset{
      x label style={at={(axis description cs:0.5,-0.2)},anchor=north},
    }
    \centering
    \tikzexternaldisable%
    \begin{subfigure}[t]{.31\textwidth}
        \centering
        \begin{tikzpicture}
            \begin{axis}[
                height=4.5cm, width=\linewidth,
                xlabel=iteration, ylabel={\# branches},
                scaled y ticks=base 10:-3,
                legend style={font=\scriptsize, at={(0.02,0.98)}, anchor=north west},
                legend cell align=left,
            ]
                \addplot[CatC, mark=none, ultra thick] table[x=iteration, y=total,      col sep=comma] {data/mnist_bounds_zero_baseline/run_summary_branches.csv};
                \addplot[CatB, mark=none, ultra thick] table[x=iteration, y=active,     col sep=comma] {data/mnist_bounds_zero_baseline/run_summary_branches.csv};
                \addplot[CatA, mark=none, ultra thick] table[x=iteration, y=pruned_cum, col sep=comma] {data/mnist_bounds_zero_baseline/run_summary_branches.csv};
            \end{axis}
        \end{tikzpicture}
        \caption{
            \textbf{Number of Branches.}
            The total number of branches~\legendcolorlineThick{CatC}, number of not currently pruned (active) branches~\legendcolorlineThick{CatB}, and the cumulative number of pruned branches~\legendcolorlineThick{CatA} for each iteration of \OurAlgo{}.
        }\label{fig:bab-branch-counts}
    \end{subfigure}\hfill
    \begin{subfigure}[t]{.31\textwidth}
        \centering
        \begin{tikzpicture}
            \begin{axis}[
                height=4.5cm, width=\linewidth,
                xlabel=iteration, ylabel=\# branches,
                scaled y ticks=base 10:-3,
                legend style={font=\scriptsize, at={(0.02,0.98)}, anchor=north west},
                legend cell align=left,
            ]
                \addplot[CatA, mark=none, ultra thick] table[x=iteration, y=tight_cum,       col sep=comma] {data/mnist_bounds_zero_baseline/run_summary_pruning.csv};
                \addplot[CatB, mark=none, ultra thick] table[x=iteration, y=fully_split_cum, col sep=comma] {data/mnist_bounds_zero_baseline/run_summary_pruning.csv};
                \legend{Tight (cum.), Fully split (cum.)}
            \end{axis}
        \end{tikzpicture}
        \caption{
            \textbf{Pruning Reasons.}
            The number of branches pruned due to having tight value function bounds~\legendcolorlineThick{CatA} and due to being fully split~\legendcolorlineThick{CatB} for each iteration of \OurAlgo{}.
        }\label{fig:bab-pruning-reasons}
    \end{subfigure}\hfill
    \begin{subfigure}[t]{.31\textwidth}
        \centering
        \begin{tikzpicture}
            \begin{axis}[
                height=4.5cm, width=\linewidth,
                xlabel=iteration, ylabel={$\max_i \overline{\varphi_i} - \underline{\varphi_i}$},
                legend style={font=\scriptsize, at={(0.98,0.98)}, anchor=north east},
                legend cell align=left,
            ]
                \addplot[CatA, mark=none, ultra thick] table[x=iteration, y=max_gap,  col sep=comma] {data/mnist_bounds_zero_baseline/run_summary_margin.csv};
            \end{axis}
        \end{tikzpicture}
        \caption{
            \textbf{Bounds Margin.}
            The largest SHAP bounds margin across features for each iteration of \OurAlgo{}.
        }\label{fig:bab-bounds-margin}
    \end{subfigure}
    \caption{\textbf{Extended \MNISTDataset{} Experiment Analysis.}
      This figure provides additional statistics for our \MNISTDataset{} experiment described in \cref{sec:experiments-mnist} and visualised in \cref{fig:mnist-bounds}
    }\label{fig:bab-run-summary}
\end{figure}

\subsection{Experiments on Additional Image Datasets}\label{appendix:experiments-additional-image-datasets}
Complementing the results presented in \cref{sec:experiments-mnist}, we apply \OurAlgo{} to compute SHAP values on additional image datasets using both the zero-baseline and the mean-baseline value functions.
Concretely, we compute bounds on \MNISTDataset{}, \FashionMNISTDataset{}, \CIFARTenDataset{}, and \GTSRBDataset{} convolutional neural networks.
\Cref{tab:datasets-and-networks} provides the architectures of these networks.
For \MNISTDataset{} and \FashionMNISTDataset{}, we group the input images into evenly-spaced grids of superpixels, for which we compute SHAP values.
On \CIFARTenDataset{} and \GTSRBDataset{}, we compute superpixels using the SLIC~\citep{achanta12slic} and watershed~\citep{neubert14watershed} image segmentation algorithms.
While we run \OurAlgo{} with a batch size of~\(4096\) until it reaches a timeout of~\(900\)s on \MNISTDataset{} and \FashionMNISTDataset{}, due to the size of the neural networks for \CIFARTenDataset{} and \GTSRBDataset{}, we reduce the batch size to~\(1024\) and increase the timeout to~\(2\) hours.
\Cref{fig:image-mean-baseline-bounds} contains the results of \OurAlgo{} using the mean-baseline value function. 
\Cref{fig:image-zero-baseline-bounds} complements \cref{fig:mnist-bounds} by presenting additional \OurAlgo{} results for the zero-baseline value function.

\begin{figure*}
    \centering
    \begin{tabular}{@{}cC{.19\linewidth}@{}C{.19\linewidth}@{}C{.19\linewidth}@{}C{.19\linewidth}@{}C{.19\linewidth}@{}}
        & \textbf{First Iteration} & & & \textbf{Exact} & \textbf{Bounds Colourmap} \\
        \rotatebox{90}{\hspace{-.5cm}\MNISTDataset{} \(5\times5\)} &
        \includegraphics[width=.95\linewidth]{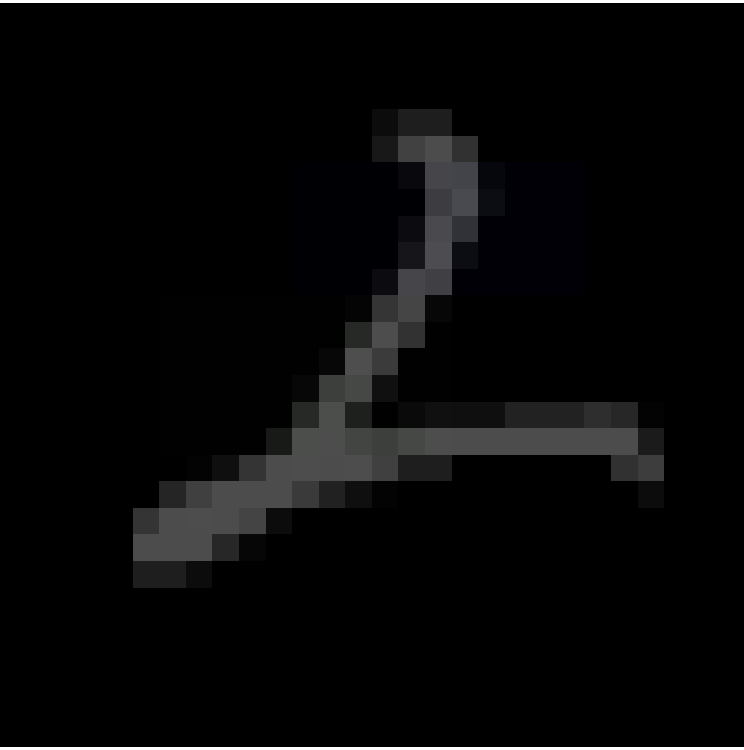}\par $t=1$ &
        \includegraphics[width=.95\linewidth]{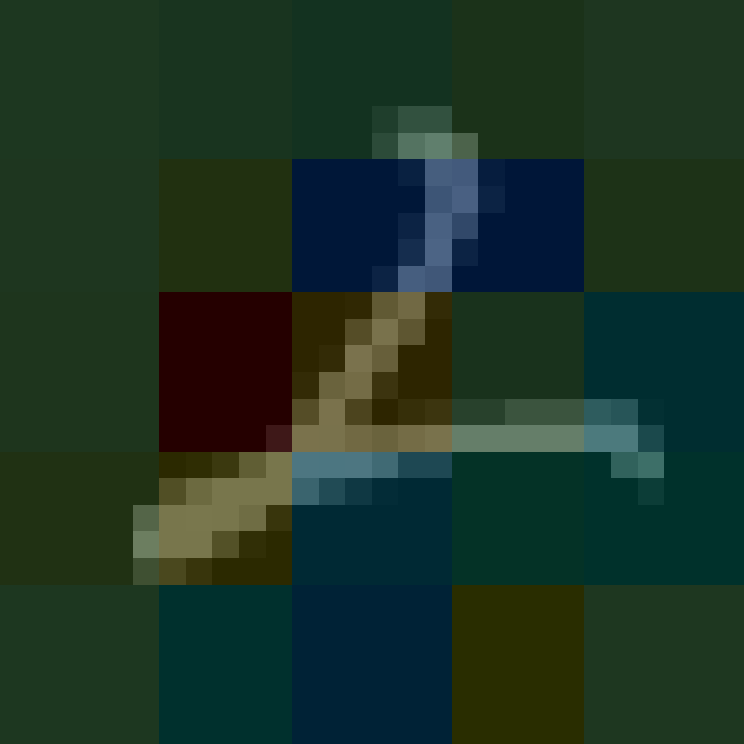}\par $t=100$ &
        \includegraphics[width=.95\linewidth]{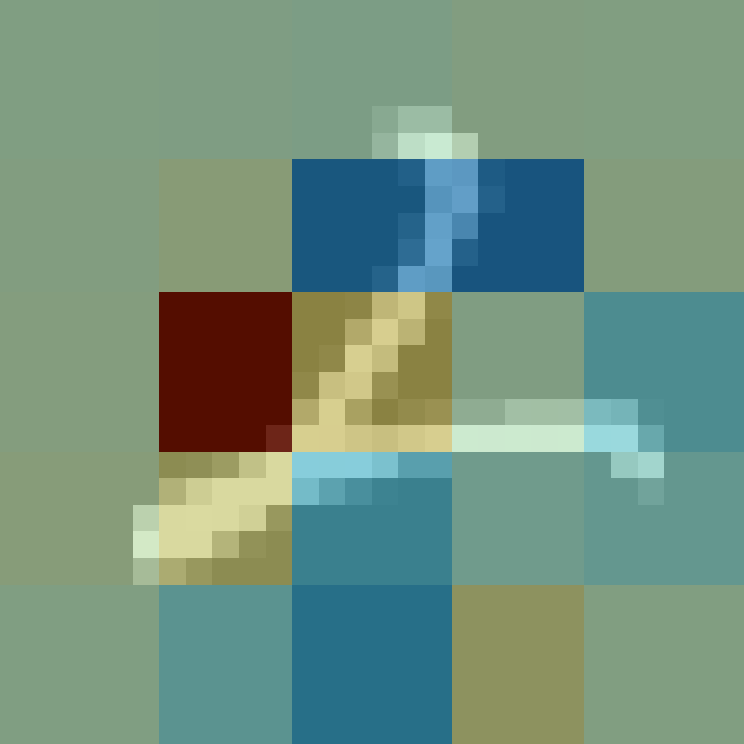}\par $t=629$ ($10\%$ HR) &  
        \includegraphics[width=.95\linewidth]{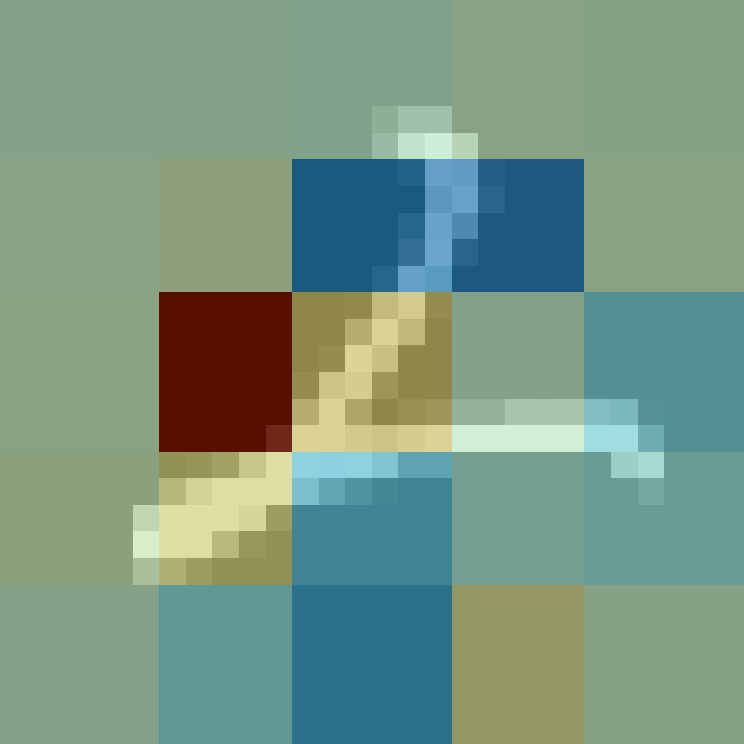}\par $t=8186$ (Exact) & 
        \begin{tikzpicture}
            \begin{axis}[
                width=1.18\linewidth, height=1.18\linewidth,
                xlabel={half-range}, ylabel={midpoint},
                xlabel style={yshift=0pt, font=\footnotesize, inner sep=0pt},
                ylabel style={yshift=-8pt, xshift=3pt, font=\footnotesize, inner sep=0pt},
                tick label style={font=\footnotesize},
                enlargelimits=false, axis on top
            ]
                \addplot graphics [
                    xmin=0.0, xmax=2.7074172496795654, ymin=-2.7074172496795654, ymax=2.7074172496795654,
                ] {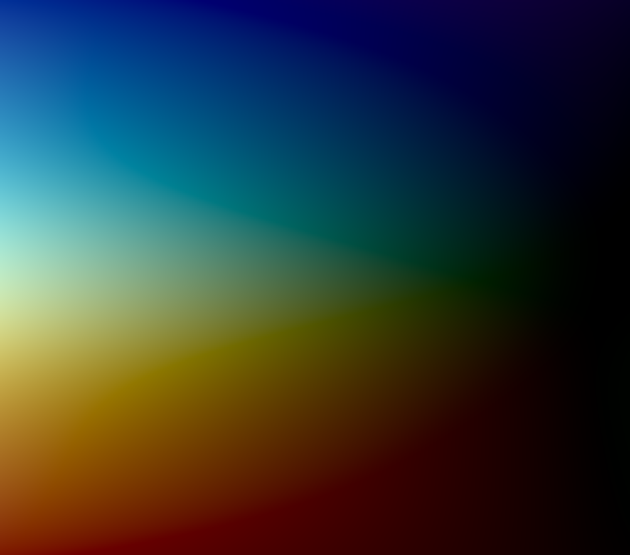};
            \end{axis}
        \end{tikzpicture} \par \\[1.75cm]
        \rotatebox{90}{\hspace{-1.5cm}\FashionMNISTDataset{} \(5\times5\)} &
        \includegraphics[width=.95\linewidth]{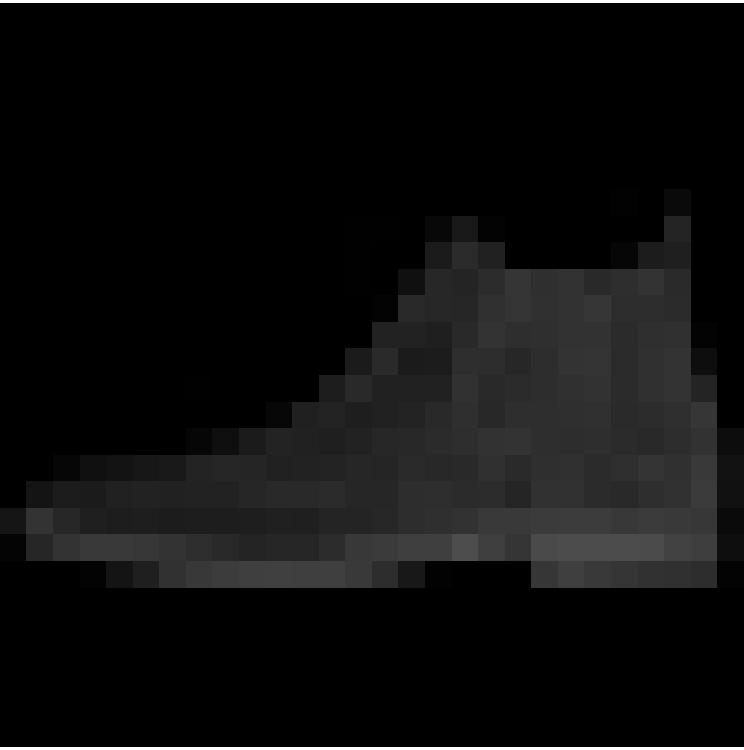}\par $t=1$ &
        \includegraphics[width=.95\linewidth]{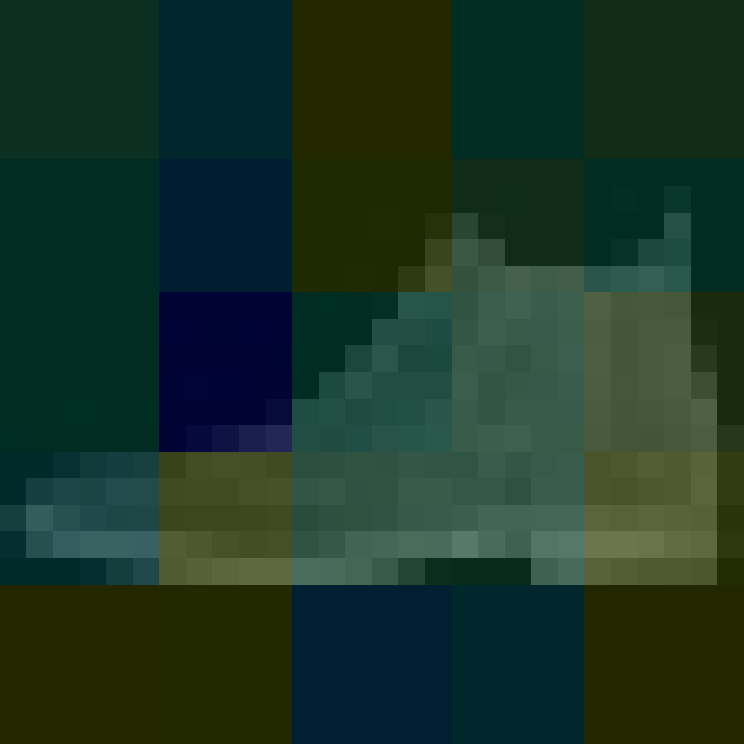}\par $t=5000$ &
        \includegraphics[width=.95\linewidth]{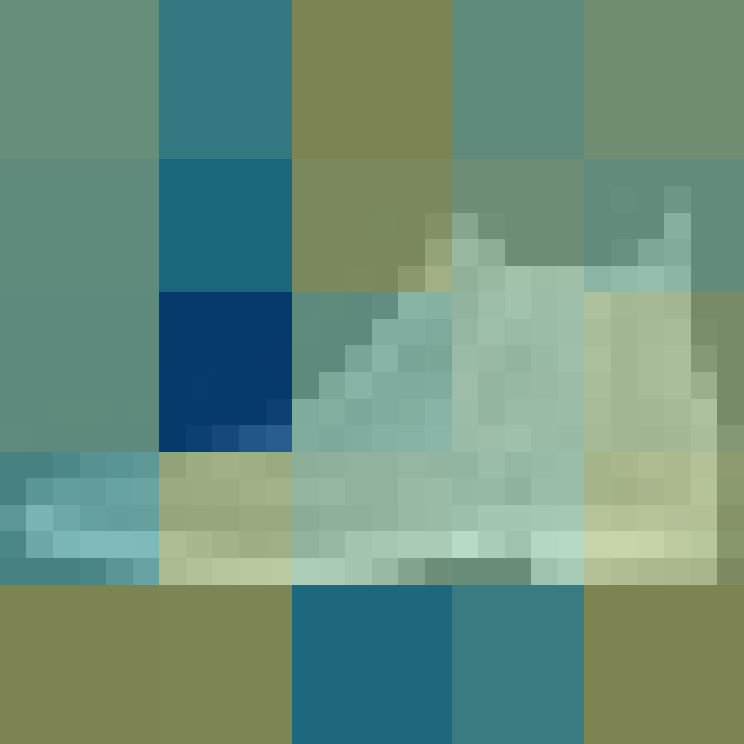}\par $t=7376$ ($10\%$ HR) &
        \includegraphics[width=.95\linewidth]{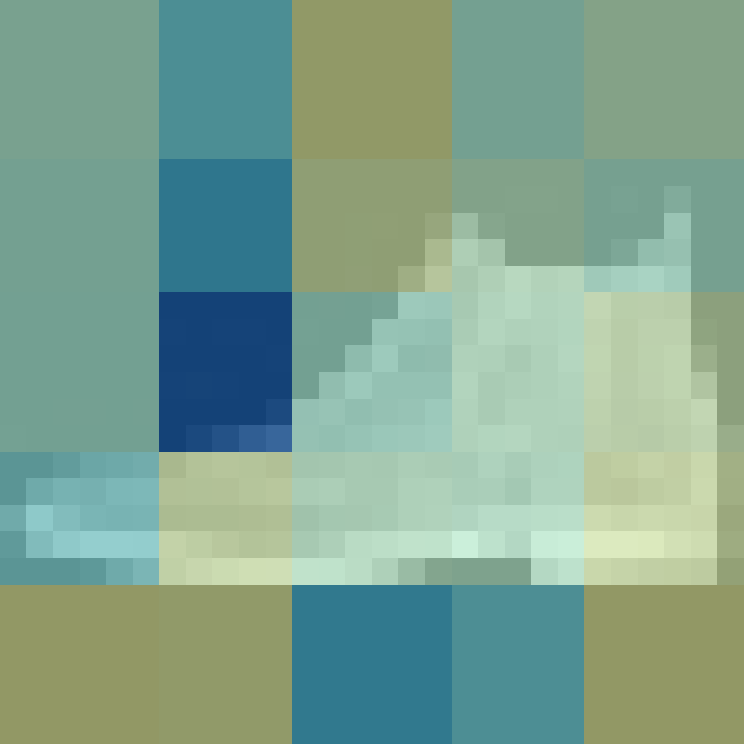}\par $t=8192$ (Exact) &
        \begin{tikzpicture}
            \begin{axis}[
                width=1.18\linewidth, height=1.18\linewidth,
                xlabel={half-range}, ylabel={midpoint},
                xlabel style={yshift=0pt, font=\footnotesize, inner sep=0pt},
                ylabel style={yshift=-8pt, xshift=3pt, font=\footnotesize, inner sep=0pt},
                tick label style={font=\footnotesize},
                enlargelimits=false, axis on top
            ]
                \addplot graphics [
                    xmin=0.0, xmax=11.159149646759033, ymin=-14.091064453125, ymax=14.091064453125,
                ] {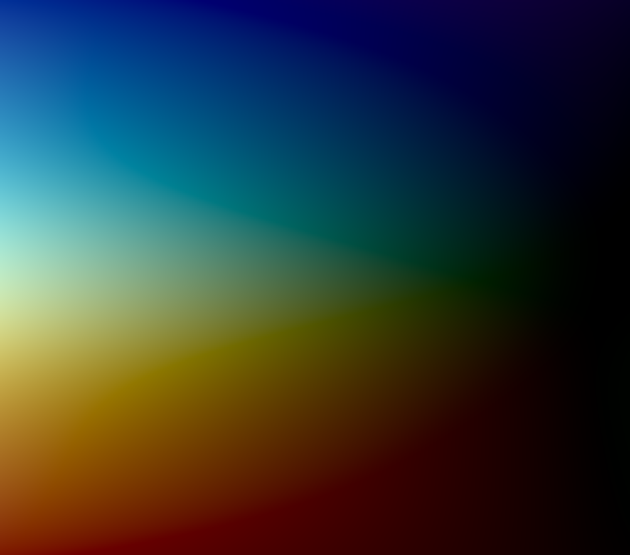};
            \end{axis}
        \end{tikzpicture} \\[1.75cm]
        \rotatebox{90}{\hspace{-1.3cm}\CIFARTenDataset{} \(26\) features} &
        \includegraphics[width=.95\linewidth]{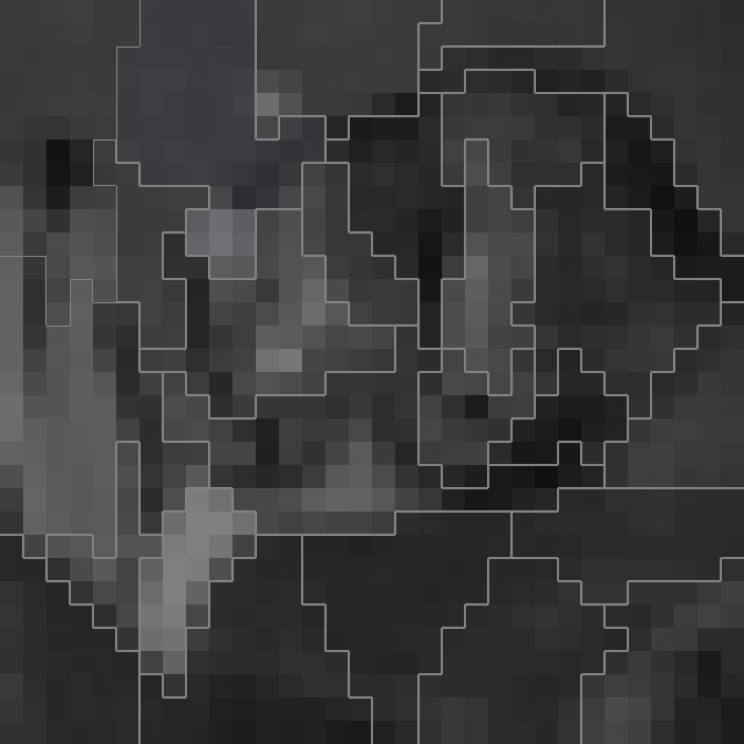}\par $t=1$ &
        \includegraphics[width=.95\linewidth]{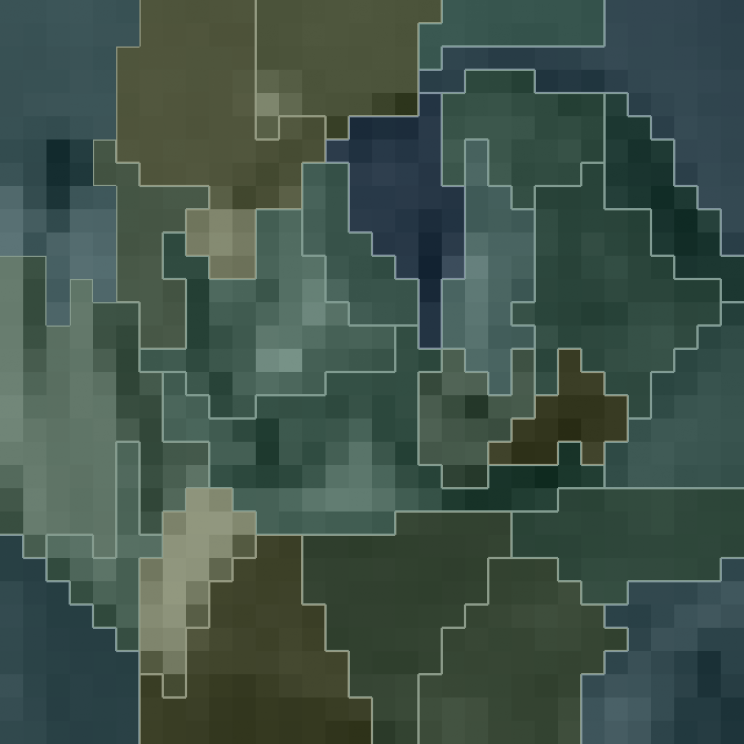}\par $t=65301$ &
        \includegraphics[width=.95\linewidth]{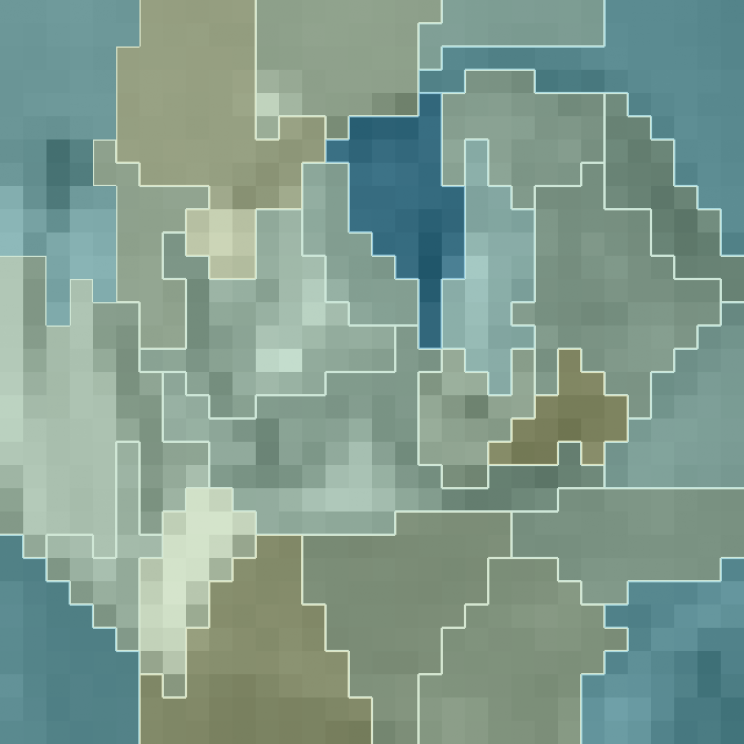}\par $t=65502$ ($10\%$ HR) &
        \includegraphics[width=.95\linewidth]{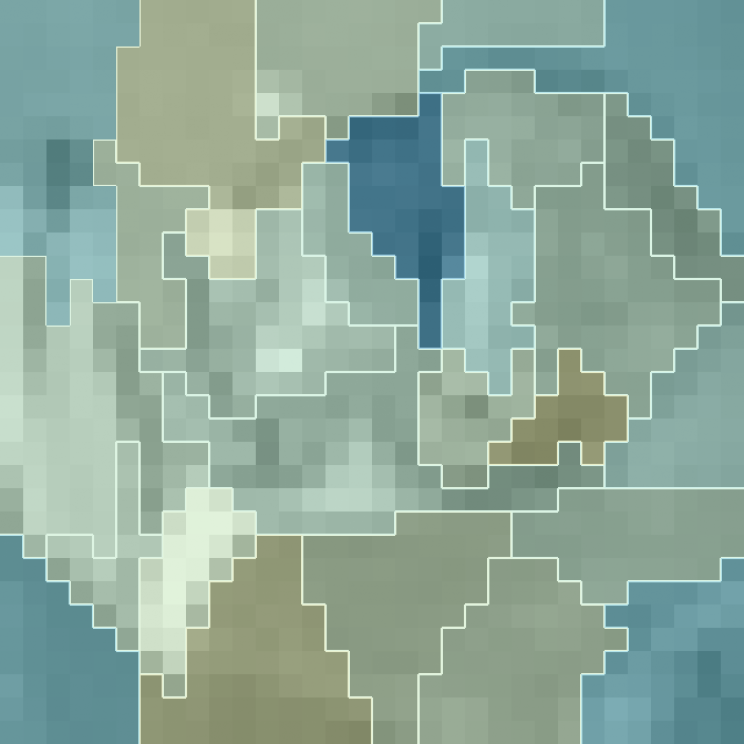}\par $t=65536$ (Exact) &
        \begin{tikzpicture}
            \begin{axis}[
                width=1.18\linewidth, height=1.18\linewidth,
                xlabel={half-range}, ylabel={midpoint},
                xlabel style={yshift=0pt, font=\footnotesize, inner sep=0pt},
                ylabel style={yshift=-8pt, xshift=3pt, font=\footnotesize, inner sep=0pt},
                tick label style={font=\footnotesize},
                enlargelimits=false, axis on top
            ]
                \addplot graphics [
                    xmin=0.0, xmax=1.116895079612732, ymin=-2.233790159225464, ymax=2.233790159225464,
                ] {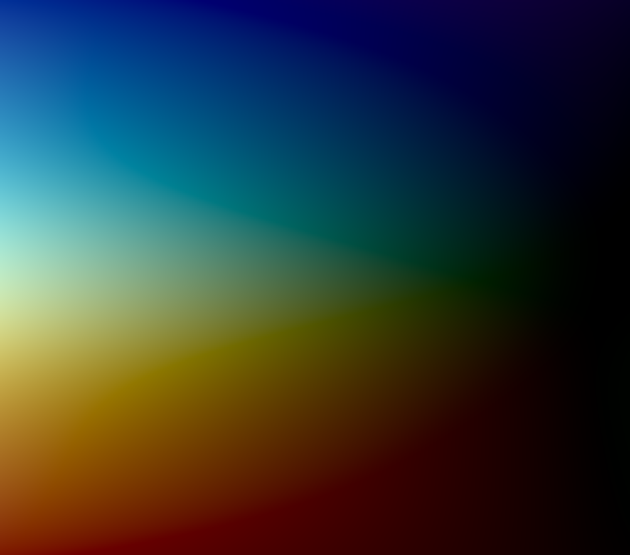};
            \end{axis}
        \end{tikzpicture} \\[1.75cm]
        \rotatebox{90}{\hspace{-1.0cm}\GTSRBDataset{} \(26\) features} &
        \includegraphics[width=.95\linewidth]{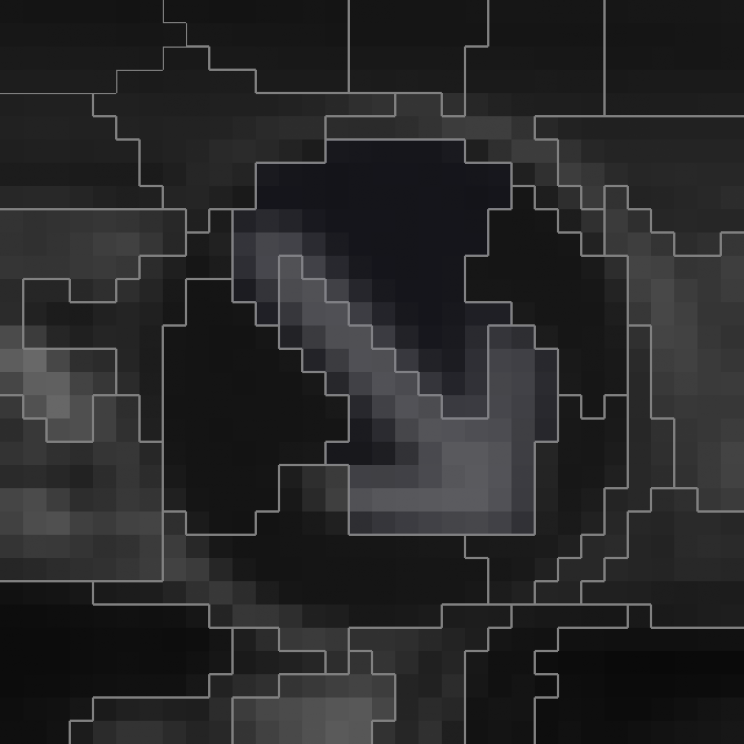}\par $t=1$ &
        \includegraphics[width=.95\linewidth]{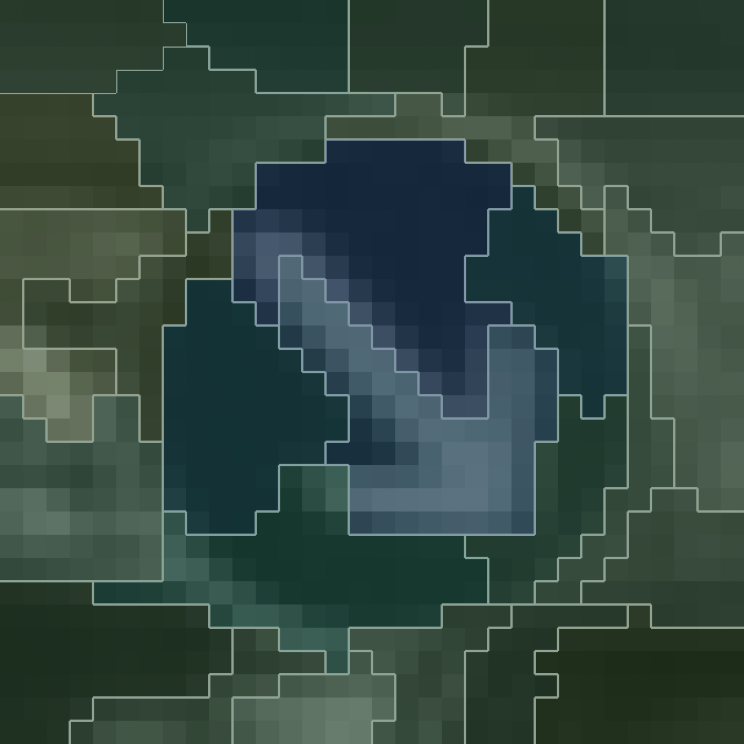}\par $t=63001$ &
        \includegraphics[width=.95\linewidth]{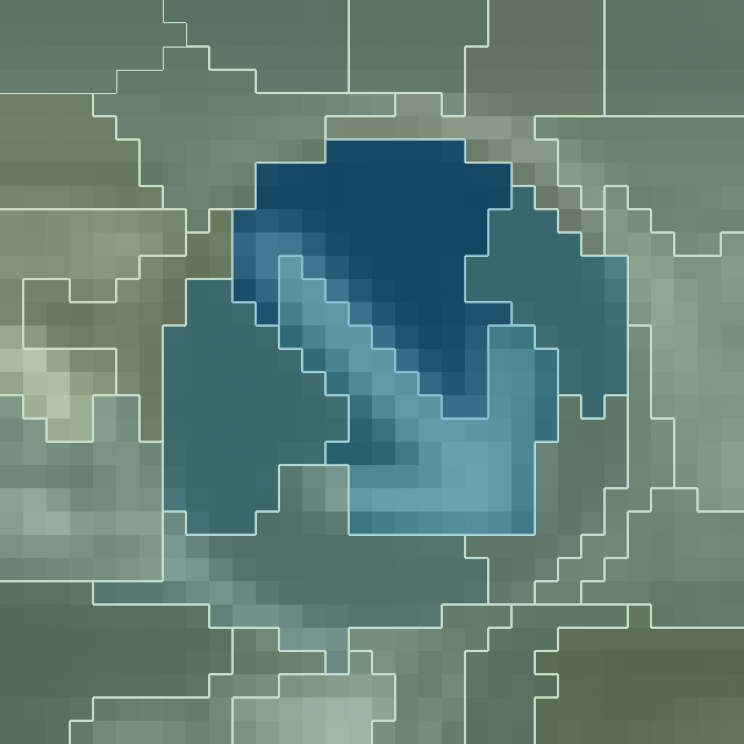}\par $t=64877$ ($10\%$ HR) &
        \includegraphics[width=.95\linewidth]{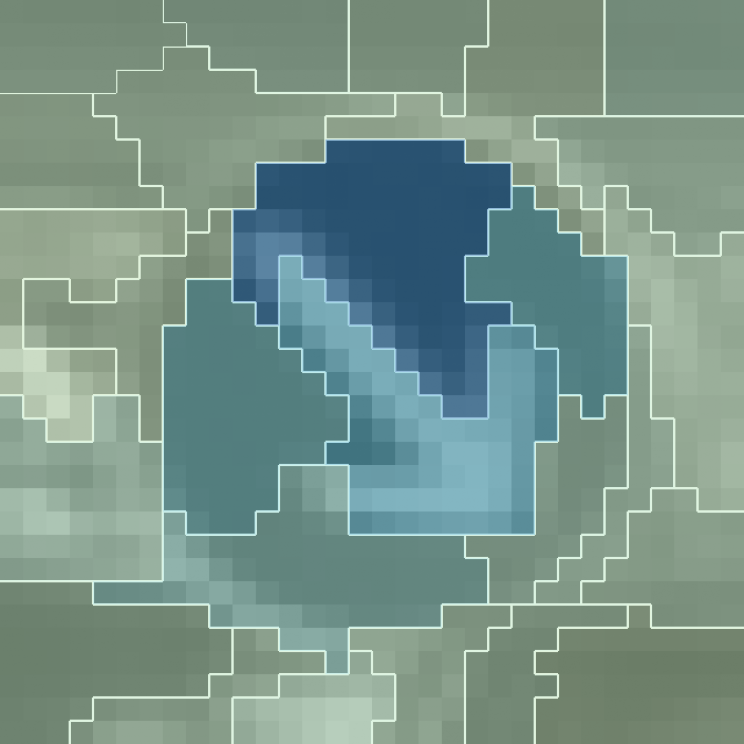}\par $t=65536$ (Exact) &
        \begin{tikzpicture}
            \begin{axis}[
                width=1.18\linewidth, height=1.18\linewidth,
                xlabel={half-range}, ylabel={midpoint},
                xlabel style={yshift=0pt, font=\footnotesize, inner sep=0pt},
                ylabel style={yshift=-8pt, xshift=3pt, font=\footnotesize, inner sep=0pt},
                tick label style={font=\footnotesize},
                enlargelimits=false, axis on top
            ]
                \addplot graphics [
                    xmin=0.0, xmax=24.837082862854004, ymin=-24.837082862854004, ymax=24.837082862854004,
                ] {figures/cifar10_bounds_mean_baseline/replay_colormap};
            \end{axis}
        \end{tikzpicture}
    \end{tabular}
    \caption{\textbf{\OurAlgo{} Results for Mean-Baseline SHAP on \MNISTDataset{}, \FashionMNISTDataset{}, \CIFARTenDataset{}, \& \GTSRBDataset{}.}
        For \MNISTDataset{}, we compute SHAP value bounds for the class \enquote{4} score on a test set image of a \enquote{2}. 
        For \FashionMNISTDataset{}, we attribute the \enquote{Ankle Boot} score on a test set image of an \enquote{Ankle Boot}. 
        For \CIFARTenDataset{}, we attribute the \enquote{Cat} score on a test set image of a \enquote{Cat}. 
        For \GTSRBDataset{}, we attribute the \enquote{Pass on the Right} score on a test set image of a \enquote{Pass on the Right} sign. 
        The SHAP bounds~\(\lb{\shapval}, \ub{\shapval}\) are visualised by their midpoints~\((\ub{\shapval} + \lb{\shapval}) / 2\) and half-ranges (HR)~\((\ub{\shapval} - \lb{\shapval}) / 2\) determining colour hue and lightness, respectively.
        \textbf{Lighter colour means tighter bounds in this figure.}
        We write \enquote{$p\%$ HR} to denote the iteration in which the largest half-range is less than~$p\%$ of the network output for the output to attribute.
    }\label{fig:image-mean-baseline-bounds}
\end{figure*}

\begin{figure*}
    \centering
    \begin{tabular}{@{}cC{.19\linewidth}@{}C{.19\linewidth}@{}C{.19\linewidth}@{}C{.19\linewidth}@{}C{.19\linewidth}@{}}
        & \textbf{First Iteration} & & & \textbf{Exact} & \textbf{Bounds Colourmap} \\
        \rotatebox{90}{\hspace{-1.33cm}\FashionMNISTDataset{} \(6\times6\)} &
        \includegraphics[width=.95\linewidth]{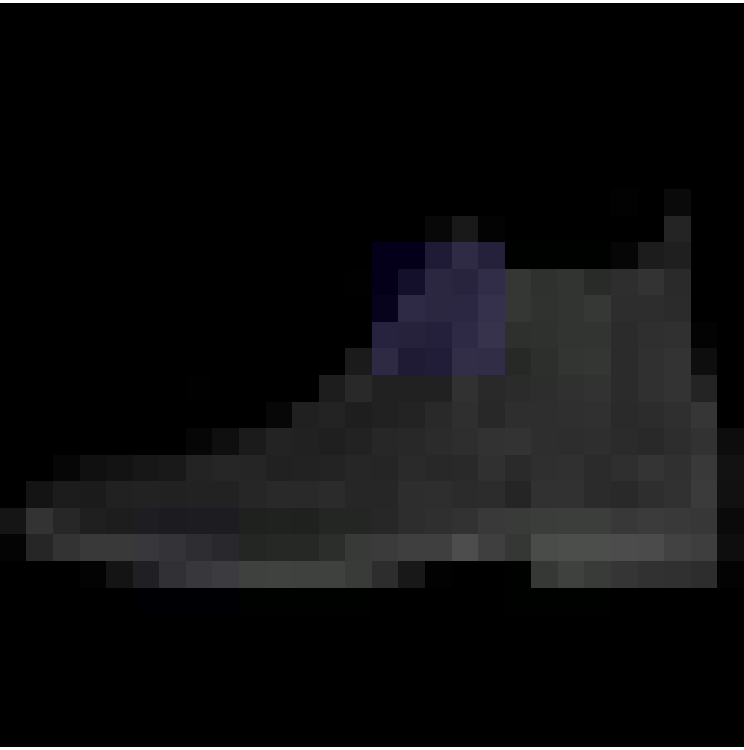}\par $t=1$ &
        \includegraphics[width=.95\linewidth]{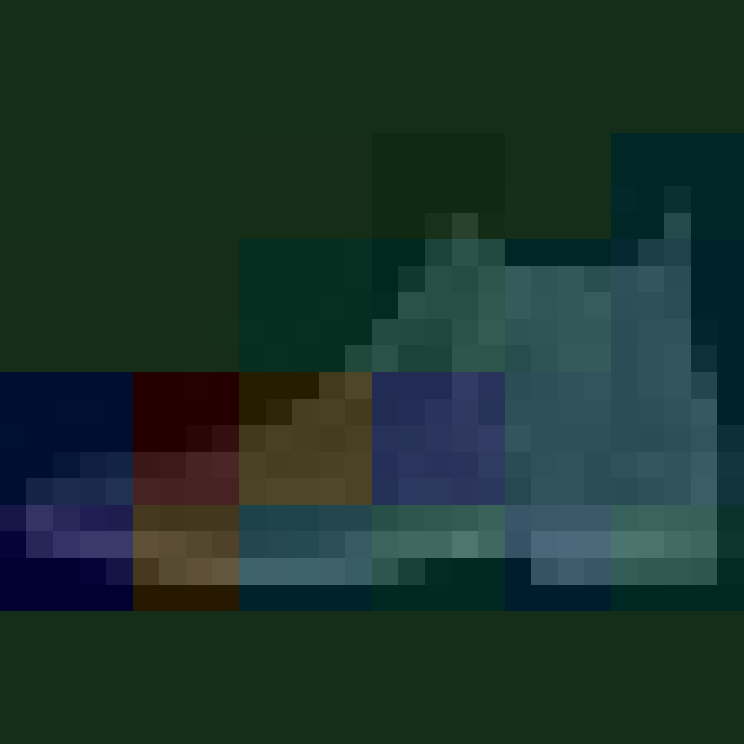}\par $t=49$ &
        \includegraphics[width=.95\linewidth]{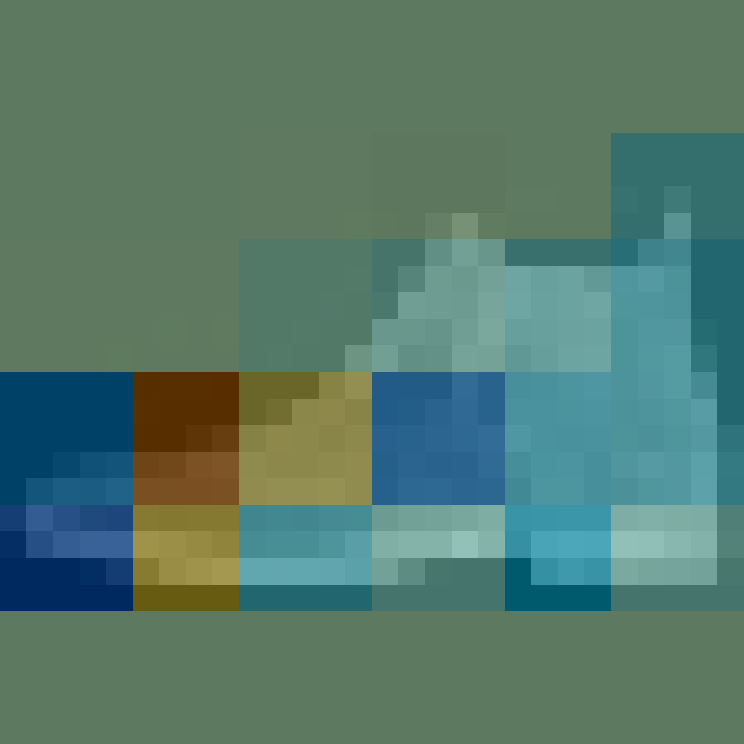}\par $t=59$ ($10\%$ HR) &
        \includegraphics[width=.95\linewidth]{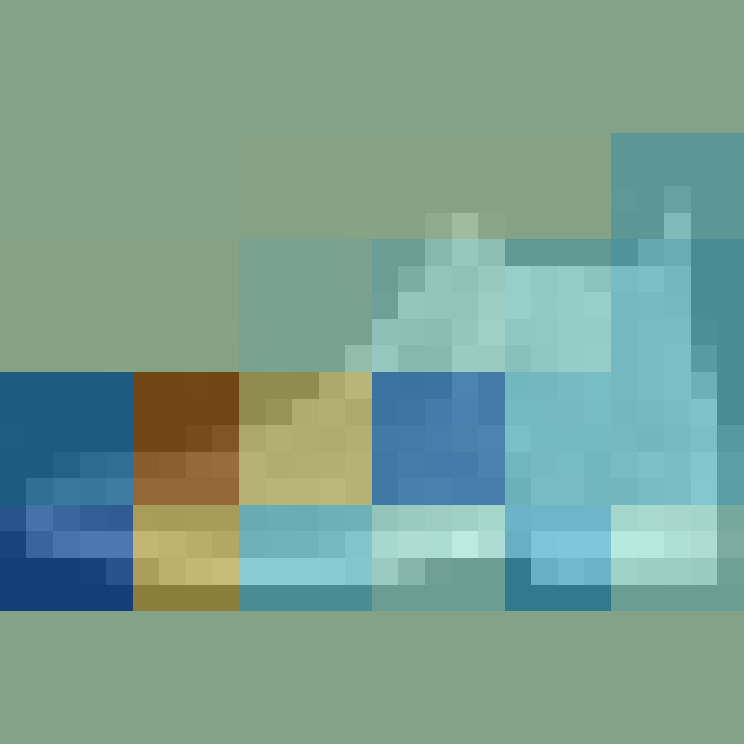}\par $t=1013$ (Exact) &
        \begin{tikzpicture}
            \begin{axis}[
                width=1.18\linewidth, height=1.18\linewidth,
                xlabel={half-range}, ylabel={midpoint},
                xlabel style={yshift=0pt, font=\footnotesize, inner sep=0pt},
                ylabel style={yshift=-8pt, xshift=3pt, font=\footnotesize, inner sep=0pt},
                tick label style={font=\footnotesize},
                enlargelimits=false, axis on top
            ]
                \addplot graphics [
                    xmin=0.0, xmax=5.002566576004028, ymin=-6.1356201171875, ymax=6.1356201171875,
                ] {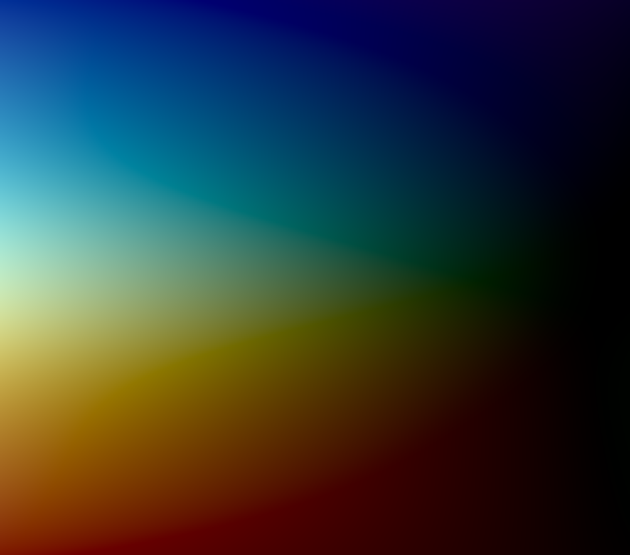};
            \end{axis}
        \end{tikzpicture} 
    \end{tabular}
    \caption{\textbf{\OurAlgo{} Results for Zero-Baseline SHAP on \FashionMNISTDataset{}.}
        We attribute the \enquote{Ankle Boot} score on a test set image of an \enquote{Ankle Boot}. 
        The SHAP bounds~\(\lb{\shapval}, \ub{\shapval}\) are visualised by their midpoints~\((\ub{\shapval} + \lb{\shapval}) / 2\) and half-ranges (HR)~\((\ub{\shapval} - \lb{\shapval}) / 2\) determining colour hue and lightness, respectively.
        \textbf{Lighter colour means tighter bounds in this figure.}
        We write \enquote{$p\%$ HR} to denote the iteration in which the largest half-range is less than~$p\%$ of the network output for the output to attribute.
    }\label{fig:image-zero-baseline-bounds}
\end{figure*}

\subsection{Additional Details on the Comparison to \ExactSHAP{}}\label{appendix:experiments-exactshap-additional}
\Cref{tab:ours-vs-exactshap-extended} complements \cref{tab:ours-vs-exactshap} by providing the interquartile ranges (IQRs) of the runtimes of \OurAlgo{} and \ExactSHAP{} across the~\(10\) test set samples we evaluate the algorithms on in \cref{sec:experiments-compare-to-exactshap}.
For up to~\(23\) input dimensions, the IQRs are within~\(2\)s of the median runtime, with the exception of the \DefaultDataset{} dataset, where the IQR for \OurAlgo{} computing \enquote{10\% HR} bounds is~\(8\)s.
For larger input spaces, the IQRs reach up to~\(191\)s for the \AnnealingDataset{} dataset.

\begin{table}
  \centering
  \caption{%
    \textbf{Extended Version of \cref{tab:ours-vs-exactshap}: \OurAlgo{} vs. \ExactSHAP{}.}
    In addition to the median runtimes provided in \cref{tab:ours-vs-exactshap}, this table also provides the interquartile ranges (\enquote{IQR}) of the runtimes. 
  }\label{tab:ours-vs-exactshap-extended}%
  \begin{tabular}{@{}llllll@{}}
     &  & & \multicolumn{3}{c}{\textbf{\OurAlgo{} (Ours)}} \\\cmidrule(lr){4-6}
     &  & \multicolumn{1}{c}{\textbf{\textsc{Exact SHAP}}}  & \multicolumn{1}{c}{$10$\% HR} & \multicolumn{1}{c}{$1$\% HR} & \multicolumn{1}{c}{Exact} \\
    $|\powerset{\upto{n}}|$ & \textbf{Dataset} & median (IQR) & median (IQR) & median (IQR) & median (IQR) \\
    \midrule
    $2^{16}{\scriptstyle>\!10^4}$ & \ObesityDataset{}       & $4$s ($4$s--$4$s) & $ 18$s ($18$s--$18$s) & $ 18$s ($18$s--$19$s) & $ 18$s ($18$s--$19$s) \\
    $2^{20}{\scriptstyle>\!10^6}$ & \GermanDataset{}        & $9$s ($9$s--$9$s) & $ 16$s ($16$s--$17$s) & $ 19$s ($17$s--$19$s) & $ 20$s ($19$s--$20$s) \\
    $2^{22}{\scriptstyle>\!10^6}$ & \MushroomDataset{}      &   --              & $ 17$s ($16$s--$17$s) & $ 20$s ($20$s--$32$s) & $ 25$s ($25$s--$25$s) \\
    $2^{23}{\scriptstyle>\!10^6}$ & \DefaultDataset{}       &   --              & $127$s ($123$s--$130$s) & $132$s ($131$s--$133$s) & $133$s ($132$s--$133$s) \\
    $2^{25}{\scriptstyle>\!10^7}$ & \AutomobileDataset{}    &   --              & $ 81$s ($73$s--$85$s) & $213$s ($187$s--$233$s) & $316$s ($314$s--$317$s) \\
    $2^{27}{\scriptstyle>\!10^8}$ & \SteelDataset{}         &   --              & $ 37$s ($30$s--$44$s) & $322$s ($271$s--$374$s) &           --  \\
    $2^{30}{\scriptstyle>\!10^9}$ & \BreastCancerDataset{}  &   --              & $ 49$s ($35$s--$67$s) & $322$s ($271$s--$564$s) &    --    \\
    $2^{38}{\scriptstyle>\!10^{11}}$ & \AnnealingDataset{}  &   --              & $269$s ($124$s--$315$s) &    --   &    --   \\
    $2^{60}{\scriptstyle>\!10^{18}}$ & \SonarDataset{}      &   --              & $ 13$s ($13$s--$13$s)  &    --  &    --  
  \end{tabular}%
  \vspace*{-.1cm}
\end{table}

\subsection{Comparison to Additional Statistical Estimators}\label{appendix:experiments-additional-estimators}
\begin{figure*}
  \centering
  \tikzexternaldisable%
  \begin{tikzpicture}
    \begin{groupplot}[
        group style={
            group size=3 by 4,
            xlabels at=edge bottom,
            ylabels at=edge left,
            xticklabels at=edge bottom,
            vertical sep=.75cm,
        },
        height=4.5cm, width=.33\textwidth,
        xlabel={Sample Size},
        ylabel style={align=center},
    ]
      \nextgroupplot[
        xmode=log, ymode=log, 
        title=\MushroomDataset, 
        ylabel={\KernelSHAP{}\\ Max. Error},
        xmin=1000,xmax=10000000, ymin=0.0002, ymax=0.1,
      ]
      \addplot[name path=Zero, forget plot, domain=100:10000000, samples=2] {0.000000001};
      \addplot[name path=OursHigh, FirstSHAPLine, forget plot, domain=100:10000000, samples=2] {0.009887456893920898};  
      \addplot[name path=OursLow, FirstSHAPLine, forget plot, domain=100:10000000, samples=2] {0.0009784698486328125};  
      \addplot[FirstSHAPArea, opacity=0.3, forget plot] fill between [of=Zero and OursHigh];
      \addplot[FirstSHAPArea, opacity=0.3] fill between [of=Zero and OursLow];
      
      \addplot[name path=KernelSHAPMin, SecondSHAPLine, forget plot] table [col sep=comma,x=num_samples,y=KernelSHAP_min_error] {data/bab_vs_estimators/mushroom-mlp-8x1_linf_0.csv};
      \addplot[name path=KernelSHAPMax, SecondSHAPLine, forget plot] table [col sep=comma,x=num_samples,y=KernelSHAP_max_error] {data/bab_vs_estimators/mushroom-mlp-8x1_linf_0.csv};
      \addplot[SecondSHAPArea] fill between [of=KernelSHAPMin and KernelSHAPMax];
      
      \node[FirstTimeMarker] at (axis cs:2500000,0.009887456893920898) {21s};
      \node[FirstTimeMarker] at (axis cs:2500000,0.0009784698486328125) {23s};
      
      \nextgroupplot[
        xmode=log, ymode=log, 
        title=\DefaultDataset,
        xmin=1000,xmax=2000000, ymin=0.002, ymax=0.2,
      ]
      \addplot[name path=Zero, forget plot, domain=100:10000000, samples=2] {0.000000001};
      \addplot[name path=OursHigh, FirstSHAPLine, forget plot, domain=100:10000000, samples=2] {0.019828736782073975};  
      \addplot[name path=OursLow, FirstSHAPLine, forget plot, domain=100:10000000, samples=2] {0.006113827228546143};  
      \addplot[FirstSHAPArea, opacity=0.3, forget plot] fill between [of=Zero and OursHigh];
      \addplot[FirstSHAPArea, opacity=0.3] fill between [of=Zero and OursLow];

      \addplot[name path=KernelSHAPMin, SecondSHAPLine, forget plot] table [col sep=comma,x=num_samples,y=KernelSHAP_min_error] {data/bab_vs_estimators/default-mlp-64x3_linf_0.csv};
      \addplot[name path=KernelSHAPMax, SecondSHAPLine, forget plot] table [col sep=comma,x=num_samples,y=KernelSHAP_max_error] {data/bab_vs_estimators/default-mlp-64x3_linf_0.csv};
      \addplot[SecondSHAPArea] fill between [of=KernelSHAPMin and KernelSHAPMax];

      \node[FirstTimeMarker] at (axis cs:600000,0.019828736782073975) {132s};
      \node[FirstTimeMarker] at (axis cs:600000,0.006113827228546143) {133s};
      
      \nextgroupplot[
        xmode=log, ymode=log, 
        title=\AutomobileDataset{},
        xmin=1000,xmax=6000000, ymin=0.00008, ymax=0.01,
      ]
      \addplot[name path=Zero, forget plot, domain=100:10000000, samples=2] {0.000000001};
      \addplot[name path=OursHigh, FirstSHAPLine, forget plot, domain=100:10000000, samples=2] {0.001021549105644226};  
      \addplot[name path=OursLow, FirstSHAPLine, forget plot, domain=100:10000000, samples=2] {0.00018627941608428955};  
      \addplot[FirstSHAPArea, opacity=0.3, forget plot] fill between [of=Zero and OursHigh];
      \addplot[FirstSHAPArea, opacity=0.3] fill between [of=Zero and OursLow];

      \addplot[name path=KernelSHAPMin, SecondSHAPLine, forget plot] table [col sep=comma,x=num_samples,y=KernelSHAP_min_error] {data/bab_vs_estimators/automobile-mlp-32x2_linf_0.csv};
      \addplot[name path=KernelSHAPMax, SecondSHAPLine, forget plot] table [col sep=comma,x=num_samples,y=KernelSHAP_max_error] {data/bab_vs_estimators/automobile-mlp-32x2_linf_0.csv};
      \addplot[SecondSHAPArea] fill between [of=KernelSHAPMin and KernelSHAPMax];

      \node[FirstTimeMarker] at (axis cs:1600000,0.001021549105644226) {304s};
      \node[FirstTimeMarker] at (axis cs:1600000,0.00018627941608428955) {313s};

      \nextgroupplot[
        xmode=log, ymode=log, 
        ylabel={\LeverageSHAP{}\\ Max. Error},
        xmin=1000,xmax=10000000, ymin=0.0002, ymax=0.1,
      ]
      \addplot[name path=Zero, forget plot, domain=100:10000000, samples=2] {0.000000001};
      \addplot[name path=OursHigh, FirstSHAPLine, forget plot, domain=100:10000000, samples=2] {0.009887456893920898};  
      \addplot[name path=OursLow, FirstSHAPLine, forget plot, domain=100:10000000, samples=2] {0.0009784698486328125};  
      \addplot[FirstSHAPArea, opacity=0.3, forget plot] fill between [of=Zero and OursHigh];
      \addplot[FirstSHAPArea, opacity=0.3] fill between [of=Zero and OursLow];
      
      \addplot[name path=LeverageSHAPMin, FourthSHAPLine, forget plot] table [col sep=comma,x=num_samples,y=LeverageSHAP_min_error] {data/bab_vs_estimators/mushroom-mlp-8x1_linf_0.csv};
      \addplot[name path=LeverageSHAPMax, FourthSHAPLine, forget plot] table [col sep=comma,x=num_samples,y=LeverageSHAP_max_error] {data/bab_vs_estimators/mushroom-mlp-8x1_linf_0.csv};
      \addplot[FourthSHAPArea] fill between [of=LeverageSHAPMin and LeverageSHAPMax];
      
      \node[FirstTimeMarker] at (axis cs:2500000,0.009887456893920898) {21s};
      \node[FirstTimeMarker] at (axis cs:2500000,0.0009784698486328125) {23s};
      
      \nextgroupplot[
        xmode=log, ymode=log, 
        xmin=1000,xmax=2000000, ymin=0.002, ymax=0.2,
      ]
      \addplot[name path=Zero, forget plot, domain=100:10000000, samples=2] {0.000000001};
      \addplot[name path=OursHigh, FirstSHAPLine, forget plot, domain=100:10000000, samples=2] {0.019828736782073975};  
      \addplot[name path=OursLow, FirstSHAPLine, forget plot, domain=100:10000000, samples=2] {0.006113827228546143};  
      \addplot[FirstSHAPArea, opacity=0.3, forget plot] fill between [of=Zero and OursHigh];
      \addplot[FirstSHAPArea, opacity=0.3] fill between [of=Zero and OursLow];

      \addplot[name path=LeverageSHAPMin, FourthSHAPLine, forget plot] table [col sep=comma,x=num_samples,y=LeverageSHAP_min_error] {data/bab_vs_estimators/default-mlp-64x3_linf_0.csv};
      \addplot[name path=LeverageSHAPMax, FourthSHAPLine, forget plot] table [col sep=comma,x=num_samples,y=LeverageSHAP_max_error] {data/bab_vs_estimators/default-mlp-64x3_linf_0.csv};
      \addplot[FourthSHAPArea] fill between [of=LeverageSHAPMin and LeverageSHAPMax];

      \node[FirstTimeMarker] at (axis cs:600000,0.019828736782073975) {132s};
      \node[FirstTimeMarker] at (axis cs:600000,0.006113827228546143) {133s};
      
      \nextgroupplot[
        xmode=log, ymode=log, 
        xmin=1000,xmax=6000000, ymin=0.00008, ymax=0.01,
      ]
      \addplot[name path=Zero, forget plot, domain=100:10000000, samples=2] {0.000000001};
      \addplot[name path=OursHigh, FirstSHAPLine, forget plot, domain=100:10000000, samples=2] {0.001021549105644226};  
      \addplot[name path=OursLow, FirstSHAPLine, forget plot, domain=100:10000000, samples=2] {0.00018627941608428955};  
      \addplot[FirstSHAPArea, opacity=0.3, forget plot] fill between [of=Zero and OursHigh];
      \addplot[FirstSHAPArea, opacity=0.3] fill between [of=Zero and OursLow];

      \addplot[name path=LeverageSHAPMin, FourthSHAPLine, forget plot] table [col sep=comma,x=num_samples,y=LeverageSHAP_min_error] {data/bab_vs_estimators/automobile-mlp-32x2_linf_0.csv};
      \addplot[name path=LeverageSHAPMax, FourthSHAPLine, forget plot] table [col sep=comma,x=num_samples,y=LeverageSHAP_max_error] {data/bab_vs_estimators/automobile-mlp-32x2_linf_0.csv};
      \addplot[FourthSHAPArea] fill between [of=LeverageSHAPMin and LeverageSHAPMax];

      \node[FirstTimeMarker] at (axis cs:1600000,0.001021549105644226) {304s};
      \node[FirstTimeMarker] at (axis cs:1600000,0.00018627941608428955) {313s};

      \nextgroupplot[
        xmode=log, ymode=log, 
        ylabel={\LinearMSR{}\\ Max. Error},
        xmin=1000,xmax=10000000, ymin=0.0002, ymax=0.1,
      ]
      \addplot[name path=Zero, forget plot, domain=100:10000000, samples=2] {0.000000001};
      \addplot[name path=OursHigh, FirstSHAPLine, forget plot, domain=100:10000000, samples=2] {0.009887456893920898};  
      \addplot[name path=OursLow, FirstSHAPLine, forget plot, domain=100:10000000, samples=2] {0.0009784698486328125};  
      \addplot[FirstSHAPArea, opacity=0.3, forget plot] fill between [of=Zero and OursHigh];
      \addplot[FirstSHAPArea, opacity=0.3] fill between [of=Zero and OursLow];
      
      \addplot[name path=LinearMSRMin, FifthSHAPLine, forget plot] table [col sep=comma,x=num_samples,y=LinearMSR_min_error] {data/bab_vs_estimators/mushroom-mlp-8x1_linf_0.csv};
      \addplot[name path=LinearMSRMax, FifthSHAPLine, forget plot] table [col sep=comma,x=num_samples,y=LinearMSR_max_error] {data/bab_vs_estimators/mushroom-mlp-8x1_linf_0.csv};
      \addplot[FifthSHAPArea] fill between [of=LinearMSRMin and LinearMSRMax];
      
      \node[FirstTimeMarker] at (axis cs:2500000,0.009887456893920898) {21s};
      \node[FirstTimeMarker] at (axis cs:2500000,0.0009784698486328125) {23s};
      
      \nextgroupplot[
        xmode=log, ymode=log, 
        xmin=1000,xmax=2000000, ymin=0.002, ymax=0.2,
      ]
      \addplot[name path=Zero, forget plot, domain=100:10000000, samples=2] {0.000000001};
      \addplot[name path=OursHigh, FirstSHAPLine, forget plot, domain=100:10000000, samples=2] {0.019828736782073975};  
      \addplot[name path=OursLow, FirstSHAPLine, forget plot, domain=100:10000000, samples=2] {0.006113827228546143};  
      \addplot[FirstSHAPArea, opacity=0.3, forget plot] fill between [of=Zero and OursHigh];
      \addplot[FirstSHAPArea, opacity=0.3] fill between [of=Zero and OursLow];

      \addplot[name path=LinearMSRMin, FifthSHAPLine, forget plot] table [col sep=comma,x=num_samples,y=LinearMSR_min_error] {data/bab_vs_estimators/default-mlp-64x3_linf_0.csv};
      \addplot[name path=LinearMSRMax, FifthSHAPLine, forget plot] table [col sep=comma,x=num_samples,y=LinearMSR_max_error] {data/bab_vs_estimators/default-mlp-64x3_linf_0.csv};
      \addplot[FifthSHAPArea] fill between [of=LinearMSRMin and LinearMSRMax];

      \node[FirstTimeMarker] at (axis cs:600000,0.019828736782073975) {132s};
      \node[FirstTimeMarker] at (axis cs:600000,0.006113827228546143) {133s};
      
      \nextgroupplot[
        xmode=log, ymode=log, 
        xmin=1000,xmax=6000000, ymin=0.00008, ymax=0.01,
      ]
      \addplot[name path=Zero, forget plot, domain=100:10000000, samples=2] {0.000000001};
      \addplot[name path=OursHigh, FirstSHAPLine, forget plot, domain=100:10000000, samples=2] {0.001021549105644226};  
      \addplot[name path=OursLow, FirstSHAPLine, forget plot, domain=100:10000000, samples=2] {0.00018627941608428955};  
      \addplot[FirstSHAPArea, opacity=0.3, forget plot] fill between [of=Zero and OursHigh];
      \addplot[FirstSHAPArea, opacity=0.3] fill between [of=Zero and OursLow];

      \addplot[name path=LinearMSRMin, FifthSHAPLine, forget plot] table [col sep=comma,x=num_samples,y=LinearMSR_min_error] {data/bab_vs_estimators/automobile-mlp-32x2_linf_0.csv};
      \addplot[name path=LinearMSRMax, FifthSHAPLine, forget plot] table [col sep=comma,x=num_samples,y=LinearMSR_max_error] {data/bab_vs_estimators/automobile-mlp-32x2_linf_0.csv};
      \addplot[FifthSHAPArea] fill between [of=LinearMSRMin and LinearMSRMax];

      \node[FirstTimeMarker] at (axis cs:1600000,0.001021549105644226) {304s};
      \node[FirstTimeMarker] at (axis cs:1600000,0.00018627941608428955) {313s};

      \nextgroupplot[
        xmode=log, ymode=log, 
        ylabel={\TreeMSR{}\\ Max. Error},
        xmin=1000,xmax=10000000, ymin=0.0002, ymax=0.1,
      ]
      \addplot[name path=Zero, forget plot, domain=100:10000000, samples=2] {0.000000001};
      \addplot[name path=OursHigh, FirstSHAPLine, forget plot, domain=100:10000000, samples=2] {0.009887456893920898};  
      \addplot[name path=OursLow, FirstSHAPLine, forget plot, domain=100:10000000, samples=2] {0.0009784698486328125};  
      \addplot[FirstSHAPArea, opacity=0.3, forget plot] fill between [of=Zero and OursHigh];
      \addplot[FirstSHAPArea, opacity=0.3] fill between [of=Zero and OursLow];
      
      \addplot[name path=TreeMSRMin, ThirdSHAPLine, forget plot] table [col sep=comma,x=num_samples,y=TreeMSR_min_error] {data/bab_vs_estimators/mushroom-mlp-8x1_linf_0.csv};
      \addplot[name path=TreeMSRMax, ThirdSHAPLine, forget plot] table [col sep=comma,x=num_samples,y=TreeMSR_max_error] {data/bab_vs_estimators/mushroom-mlp-8x1_linf_0.csv};
      \addplot[ThirdSHAPArea] fill between [of=TreeMSRMin and TreeMSRMax];
      
      \node[FirstTimeMarker] at (axis cs:2500000,0.009887456893920898) {21s};
      \node[FirstTimeMarker] at (axis cs:2500000,0.0009784698486328125) {23s};
      
      \nextgroupplot[
        xmode=log, ymode=log, 
        xmin=1000,xmax=2000000, ymin=0.002, ymax=0.2,
      ]
      \addplot[name path=Zero, forget plot, domain=100:10000000, samples=2] {0.000000001};
      \addplot[name path=OursHigh, FirstSHAPLine, forget plot, domain=100:10000000, samples=2] {0.019828736782073975};  
      \addplot[name path=OursLow, FirstSHAPLine, forget plot, domain=100:10000000, samples=2] {0.006113827228546143};  
      \addplot[FirstSHAPArea, opacity=0.3, forget plot] fill between [of=Zero and OursHigh];
      \addplot[FirstSHAPArea, opacity=0.3] fill between [of=Zero and OursLow];

      \addplot[name path=TreeMSRMin, ThirdSHAPLine, forget plot] table [col sep=comma,x=num_samples,y=TreeMSR_min_error] {data/bab_vs_estimators/default-mlp-64x3_linf_0.csv};
      \addplot[name path=TreeMSRMax, ThirdSHAPLine, forget plot] table [col sep=comma,x=num_samples,y=TreeMSR_max_error] {data/bab_vs_estimators/default-mlp-64x3_linf_0.csv};
      \addplot[ThirdSHAPArea] fill between [of=TreeMSRMin and TreeMSRMax];

      \node[FirstTimeMarker] at (axis cs:600000,0.019828736782073975) {132s};
      \node[FirstTimeMarker] at (axis cs:600000,0.006113827228546143) {133s};
      
      \nextgroupplot[
        xmode=log, ymode=log, 
        xmin=1000,xmax=6000000, ymin=0.00008, ymax=0.01,
      ]
      \addplot[name path=Zero, forget plot, domain=100:10000000, samples=2] {0.000000001};
      \addplot[name path=OursHigh, FirstSHAPLine, forget plot, domain=100:10000000, samples=2] {0.001021549105644226};  
      \addplot[name path=OursLow, FirstSHAPLine, forget plot, domain=100:10000000, samples=2] {0.00018627941608428955};  
      \addplot[FirstSHAPArea, opacity=0.3, forget plot] fill between [of=Zero and OursHigh];
      \addplot[FirstSHAPArea, opacity=0.3] fill between [of=Zero and OursLow];

      \addplot[name path=TreeMSRMin, ThirdSHAPLine, forget plot] table [col sep=comma,x=num_samples,y=TreeMSR_min_error] {data/bab_vs_estimators/automobile-mlp-32x2_linf_0.csv};
      \addplot[name path=TreeMSRMax, ThirdSHAPLine, forget plot] table [col sep=comma,x=num_samples,y=TreeMSR_max_error] {data/bab_vs_estimators/automobile-mlp-32x2_linf_0.csv};
      \addplot[ThirdSHAPArea] fill between [of=TreeMSRMin and TreeMSRMax];

      \node[FirstTimeMarker] at (axis cs:1600000,0.001021549105644226) {304s};
      \node[FirstTimeMarker] at (axis cs:1600000,0.00018627941608428955) {313s};
    \end{groupplot}
  \end{tikzpicture}
  \tikzexternaldisable%
  \caption[VeriSHAP vs. KernelSHAP, LeverageSHAP, LinearMSR, \& TreeMSR]{%
    \textbf{\OurAlgo{}~\legendcolorbox{FirstSHAPArea} vs. \KernelSHAP{}~\legendcolorbox{SecondSHAPArea}, \LeverageSHAP{}~\legendcolorbox{FourthSHAPArea}, \LinearMSR{}~\legendcolorbox{FifthSHAPArea} \& \TreeMSR{}~\legendcolorbox{ThirdSHAPArea}.}
    We run \KernelSHAP{}, \LeverageSHAP{}, \LinearMSR{}, and \TreeMSR{} with increasing sample sizes until each exhausts the available GPU memory, repeating each run~\(100\) times.
    For each sample size, we plot the range across repetitions of the largest error between the estimated SHAP value and the true SHAP value of a feature.
    For \OurAlgo{}, we mark the half-ranges of the SHAP bounds at different~\tikz[baseline={([yshift={-2.7mm}]current bounding box.north)}]{\node[FirstTimeMarker] at (0,0) {runtimes}}.
  }\label{fig:ours-vs-estimators-extended}
\end{figure*}

\begin{table}
    \centering
    \caption{
        \textbf{SHAP Estimator Runtimes}. This table provides the runtime of~\(100\) repetitions of different statistical estimators for the largest number of samples that does not exceed the GPU memory of our hardware, as plotted in \cref{fig:ours-vs-estimators-extended}.
    }\label{tab:estimator-runtimes}
    \begin{tabular}{lrrr}
        & \multicolumn{3}{c}{\textbf{Runtime}} \\
        \textbf{Estimator} & \MushroomDataset{} & \DefaultDataset{} & \AutomobileDataset{} \\ \midrule
        \KernelSHAP{}      & \(346\)s & \(123\)s & \(351\)s \\
        \LeverageSHAP{}    & \(291\)s & \(324\)s & \(314\)s \\
        \LinearMSR{}       & \(267\)s & \(270\)s & \(274\)s \\
        \TreeMSR{}         & \(895\)s & \(132\)s & \(341\)s
    \end{tabular}
\end{table}

\Cref{sec:experiments-compare-to-kernelshap} compares \OurAlgo{} to \KernelSHAP{} and \TreeMSR{}.
This section provides a comparison to the additional statistical estimators \LeverageSHAP{}~\citep{musco2025leverageshap} and \LinearMSR{}~\citep{witter2025regression}.
\Cref{fig:ours-vs-estimators-extended} presents the results of our comparison of \OurAlgo{} to these statistical estimators.
\Cref{tab:estimator-runtimes} provides the runtime required for performing~\(100\) repetitions of the different statistical estimators for the largest feasible sample size on our hardware.
The additional results follow the same trend discussed for \KernelSHAP{} and \TreeMSR{} in \cref{sec:experiments-compare-to-kernelshap}.

\subsection{Evaluating SHAP Estimators using \OurAlgo{}}\label{appendix:additional-compare-estimators}
This section presents additional results on evaluating SHAP estimators using \OurAlgo{}, as discussed in \cref{sec:estimators-comparison}.
\Cref{fig:estimators-comparison-extended} extends \cref{fig:estimators-comparison} by including additional datasets.
In this section, we also include \SteelDataset{} in the comparison for which \OurAlgo{} is unable to compute the exact SHAP values.
Therefore, we report the \emph{optimistic}~mean squared error~\(\frac{1}{n}\|\!\min(|\hat{\boldsymbol{\varphi}} - \lb{\boldsymbol{\varphi}}|, |\hat{\boldsymbol{\varphi}} - \ub{\boldsymbol{\varphi}}|)\|_2^2\) on \SteelDataset{}, where~\(\hat{\boldsymbol{\varphi}}\) are the estimated SHAP values and~\([\lb{\boldsymbol{\varphi}}, \ub{\boldsymbol{\varphi}}]\) are the SHAP value bounds computed by \OurAlgo{}.

The datasets in \cref{fig:estimators-comparison} reflect the different relationships between \TreeMSR{}, \KernelSHAP{}, and \LeverageSHAP{} that we observe in \cref{fig:estimators-comparison-extended}.
In particular, the results on \ObesityDataset{} and \AutomobileDataset{} follow the same trend as on \MushroomDataset{}, and the results on \SteelDataset{} follow the same trend as on \GermanDataset{}.
\Cref{fig:leverageshap-vs-linearmsr} compares \LeverageSHAP{} to \LinearMSR{}~\citep{witter2025regression}, confirming the result of \citet{witter2025regression} that \LinearMSR{} closely matches \LeverageSHAP{}.

\begin{figure}[p]
  \centering
  \tikzexternalenable%
  \begin{tikzpicture}
    \begin{groupplot}[
        group style={
            group size=3 by 2,
            xlabels at=edge bottom,
            ylabels at=edge left,
            xticklabels at=edge bottom,
            vertical sep=1cm,
            horizontal sep=1cm,
        },
        height=5cm, width=.35\linewidth,
        xlabel={Sample Size},
        label style={align=center},
        title style={align=center},
    ]
      \nextgroupplot[
        xmode=log, ymode=log, 
        title={\ObesityDataset{}}, 
        ylabel={MSE},
        ylabel style={yshift=-.1cm},
        xmin=200,xmax=100000,
      ]
      \addplot[SecondSHAP] table [col sep=comma,x=num_samples,y=KernelSHAP_mean_error] {data/estimators_comparison/obesity-mlp-32x2_l2_mean.csv};
      \addplot[FourthSHAP] table [col sep=comma,x=num_samples,y=LeverageSHAP_mean_error] {data/estimators_comparison/obesity-mlp-32x2_l2_mean.csv};
      \addplot[ThirdSHAP] table [col sep=comma,x=num_samples,y=TreeMSR_mean_error] {data/estimators_comparison/obesity-mlp-32x2_l2_mean.csv};
    
      \nextgroupplot[
        xmode=log, ymode=log, 
        title={\GermanDataset{}}, 
        ylabel style={yshift=-.1cm},
        xmin=200,xmax=100000,
      ]
      \addplot[SecondSHAP] table [col sep=comma,x=num_samples,y=KernelSHAP_mean_error] {data/estimators_comparison/german-mlp-8x1_l2_mean.csv};
      \addplot[FourthSHAP] table [col sep=comma,x=num_samples,y=LeverageSHAP_mean_error] {data/estimators_comparison/german-mlp-8x1_l2_mean.csv};
      \addplot[ThirdSHAP] table [col sep=comma,x=num_samples,y=TreeMSR_mean_error] {data/estimators_comparison/german-mlp-8x1_l2_mean.csv};
      
      \nextgroupplot[
        xmode=log, ymode=log, 
        title={\MushroomDataset{}}, 
        ylabel style={yshift=-.1cm, xshift=.15cm},
        xmin=200,xmax=100000,
      ]
      \addplot[SecondSHAP] table [col sep=comma,x=num_samples,y=KernelSHAP_mean_error] {data/estimators_comparison/mushroom-mlp-8x1_l2_mean.csv};
      \addplot[FourthSHAP] table [col sep=comma,x=num_samples,y=LeverageSHAP_mean_error] {data/estimators_comparison/mushroom-mlp-8x1_l2_mean.csv};
      \addplot[ThirdSHAP] table [col sep=comma,x=num_samples,y=TreeMSR_mean_error] {data/estimators_comparison/mushroom-mlp-8x1_l2_mean.csv};
      
      \nextgroupplot[
        xmode=log, ymode=log, 
        title={\DefaultDataset{}}, 
        ylabel={MSE},
        ylabel style={yshift=-.1cm, xshift=.15cm},
        xmin=200,xmax=100000,
      ]
      \addplot[SecondSHAP] table [col sep=comma,x=num_samples,y=KernelSHAP_mean_error] {data/estimators_comparison/default-mlp-64x3_l2_mean.csv};
      \addplot[FourthSHAP] table [col sep=comma,x=num_samples,y=LeverageSHAP_mean_error] {data/estimators_comparison/default-mlp-64x3_l2_mean.csv};
      \addplot[ThirdSHAP] table [col sep=comma,x=num_samples,y=TreeMSR_mean_error] {data/estimators_comparison/default-mlp-64x3_l2_mean.csv};
      
      \nextgroupplot[
        xmode=log, ymode=log, 
        title={\AutomobileDataset{}}, 
        ylabel style={yshift=-.1cm, xshift=.15cm},
        xmin=200,xmax=100000,
      ]
      \addplot[SecondSHAP] table [col sep=comma,x=num_samples,y=KernelSHAP_mean_error] {data/estimators_comparison/automobile-mlp-32x2_l2_mean.csv};
      \addplot[FourthSHAP] table [col sep=comma,x=num_samples,y=LeverageSHAP_mean_error] {data/estimators_comparison/automobile-mlp-32x2_l2_mean.csv};
      \addplot[ThirdSHAP] table [col sep=comma,x=num_samples,y=TreeMSR_mean_error] {data/estimators_comparison/automobile-mlp-32x2_l2_mean.csv};
      
      \nextgroupplot[
        xmode=log, ymode=log, 
        title={\SteelDataset{}}, 
        ylabel style={yshift=-.1cm, xshift=.15cm},
        xmin=200,xmax=25000,
      ]
      \addplot[SecondSHAP] table [col sep=comma,x=num_samples,y=KernelSHAP_mean_error] {data/estimators_comparison/steel-mlp-8x2_l2_mean.csv};
      \addplot[FourthSHAP] table [col sep=comma,x=num_samples,y=LeverageSHAP_mean_error] {data/estimators_comparison/steel-mlp-8x2_l2_mean.csv};
      \addplot[ThirdSHAP] table [col sep=comma,x=num_samples,y=TreeMSR_mean_error] {data/estimators_comparison/steel-mlp-8x2_l2_mean.csv};
      
    \end{groupplot}
  \end{tikzpicture}
  \tikzexternaldisable%
  \caption[KernelSHAP vs. LeverageSHAP vs. TreeMSR (Extended)]{%
      \textbf{\KernelSHAP{}~\legendcolorline{SecondSHAP} vs. \LeverageSHAP{}~\legendcolorline{FourthSHAP} vs. \TreeMSR{}~\legendcolorline{ThirdSHAP} (Extended)}.
      We report the mean squared error (MSE) of statistical SHAP estimators across~\(100\) runs as in \cref{fig:estimators-comparison}.
      Since the exact SHAP values are unknown for the \SteelDataset{}, we compute an optimistic MSE as~\(\frac{1}{n}\|\!\min(|\hat{\boldsymbol{\varphi}} - \lb{\boldsymbol{\varphi}}|, |\hat{\boldsymbol{\varphi}} - \ub{\boldsymbol{\varphi}}|)\|_2^2\), where~\(\hat{\boldsymbol{\varphi}}\) are the estimated SHAP values and~\([\lb{\boldsymbol{\varphi}}, \ub{\boldsymbol{\varphi}}]\) are the bounds on the exact SHAP values computed by \OurAlgo{}.
  }\label{fig:estimators-comparison-extended}
\end{figure}
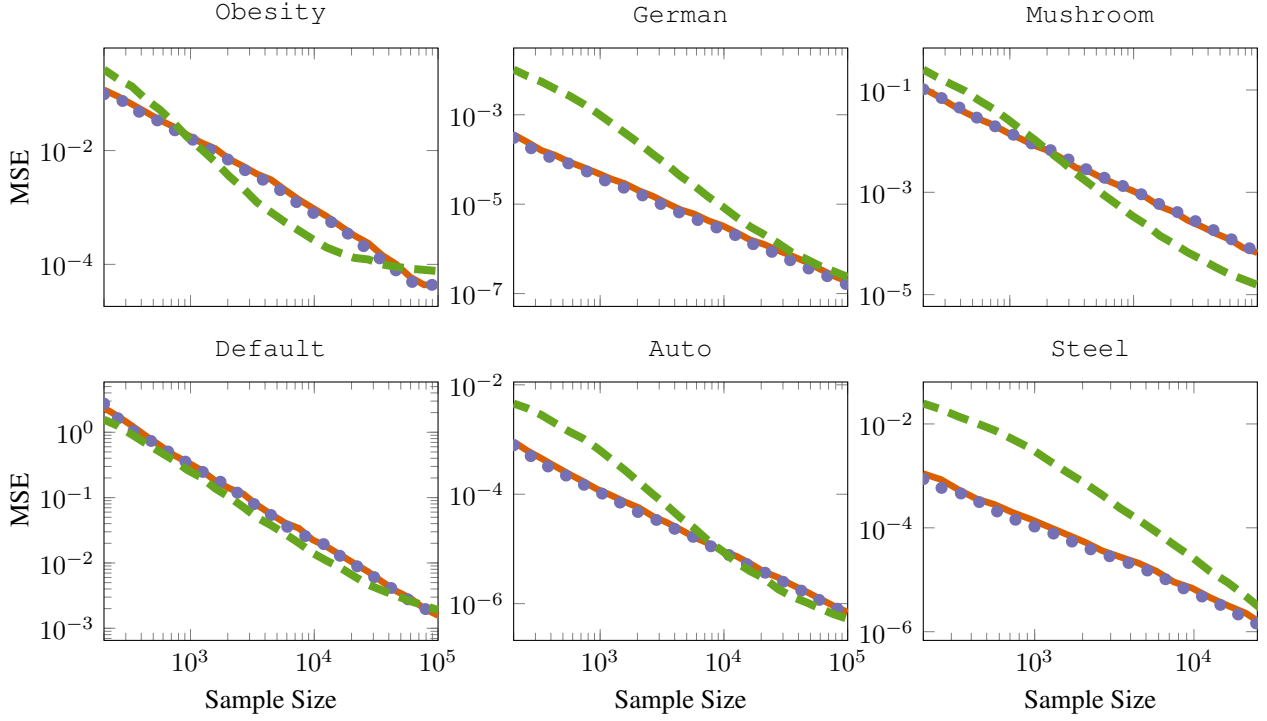

\begin{figure}[p]
  \centering
  \tikzexternalenable%
  \begin{tikzpicture}
    \begin{groupplot}[
        group style={
            group size=3 by 2,
            xlabels at=edge bottom,
            ylabels at=edge left,
            xticklabels at=edge bottom,
            vertical sep=1cm,
            horizontal sep=1cm,
        },
        height=5cm, width=.35\linewidth,
        xlabel={Sample Size},
        label style={align=center},
        title style={align=center},
    ]
      \nextgroupplot[
        xmode=log, ymode=log, 
        title={\ObesityDataset{}}, 
        ylabel={MSE},
        ylabel style={yshift=-.1cm},
        xmin=200,xmax=100000,
      ]
      \addplot[FifthSHAP] table [col sep=comma,x=num_samples,y=LinearMSR_mean_error] {data/estimators_comparison/obesity-mlp-32x2_l2_mean.csv};
      \addplot[FourthSHAP] table [col sep=comma,x=num_samples,y=LeverageSHAP_mean_error] {data/estimators_comparison/obesity-mlp-32x2_l2_mean.csv};
    
      \nextgroupplot[
        xmode=log, ymode=log, 
        title={\GermanDataset{}}, 
        ylabel style={yshift=-.1cm},
        xmin=200,xmax=100000,
      ]
      \addplot[FifthSHAP] table [col sep=comma,x=num_samples,y=LinearMSR_mean_error] {data/estimators_comparison/german-mlp-8x1_l2_mean.csv};
      \addplot[FourthSHAP] table [col sep=comma,x=num_samples,y=LeverageSHAP_mean_error] {data/estimators_comparison/german-mlp-8x1_l2_mean.csv};
      
      \nextgroupplot[
        xmode=log, ymode=log, 
        title={\MushroomDataset{}}, 
        ylabel style={yshift=-.1cm, xshift=.15cm},
        xmin=200,xmax=100000,
      ]
      \addplot[FifthSHAP] table [col sep=comma,x=num_samples,y=LinearMSR_mean_error] {data/estimators_comparison/mushroom-mlp-8x1_l2_mean.csv};
      \addplot[FourthSHAP] table [col sep=comma,x=num_samples,y=LeverageSHAP_mean_error] {data/estimators_comparison/mushroom-mlp-8x1_l2_mean.csv};
      
      \nextgroupplot[
        xmode=log, ymode=log, 
        title={\DefaultDataset{}}, 
        ylabel={MSE},
        ylabel style={yshift=-.1cm, xshift=.15cm},
        xmin=200,xmax=100000,
      ]
      \addplot[FifthSHAP] table [col sep=comma,x=num_samples,y=LinearMSR_mean_error] {data/estimators_comparison/default-mlp-64x3_l2_mean.csv};
      \addplot[FourthSHAP] table [col sep=comma,x=num_samples,y=LeverageSHAP_mean_error] {data/estimators_comparison/default-mlp-64x3_l2_mean.csv};
      
      \nextgroupplot[
        xmode=log, ymode=log, 
        title={\AutomobileDataset{}}, 
        ylabel style={yshift=-.1cm, xshift=.15cm},
        xmin=200,xmax=100000,
      ]
      \addplot[FifthSHAP] table [col sep=comma,x=num_samples,y=LinearMSR_mean_error] {data/estimators_comparison/automobile-mlp-32x2_l2_mean.csv};
      \addplot[FourthSHAP] table [col sep=comma,x=num_samples,y=LeverageSHAP_mean_error] {data/estimators_comparison/automobile-mlp-32x2_l2_mean.csv};
      
      \nextgroupplot[
        xmode=log, ymode=log, 
        title={\SteelDataset{}}, 
        ylabel style={yshift=-.1cm, xshift=.15cm},
        xmin=200,xmax=25000,
      ]
      \addplot[FifthSHAP] table [col sep=comma,x=num_samples,y=LinearMSR_mean_error] {data/estimators_comparison/steel-mlp-8x2_l2_mean.csv};
      \addplot[FourthSHAP] table [col sep=comma,x=num_samples,y=LeverageSHAP_mean_error] {data/estimators_comparison/steel-mlp-8x2_l2_mean.csv};
      
    \end{groupplot}
  \end{tikzpicture}
  \tikzexternaldisable%
  \caption[LeverageSHAP vs. LinearMSR]{%
      \textbf{\LeverageSHAP{}~\legendcolorline{FourthSHAP} vs. \LinearMSR{}~\legendcolorline{FifthSHAP}}.
      We report the mean squared error (MSE) of the statistical SHAP estimators across~\(100\) runs as in \cref{fig:estimators-comparison-extended}.
  }\label{fig:leverageshap-vs-linearmsr}
\end{figure}

\subsection{Predicting \OurAlgo{} Runtime}\label{appendix:predict-runtime}
This section is concerned with predicting the runtime of \OurAlgo{} before \OurAlgo{} is run.
The goal is to judge whether \OurAlgo{} will successfully compute tight bounds for a neural network.
Concretely, we study the number of branches required by \OurAlgo{} until the half-width of the SHAP bounds is less than~\(10\%\) of the network output to attribute (\(10\%\) HR).
The number of branches directly determines the runtime of \OurAlgo{} while discounting for different network architectures, enabling better comparability across datasets.
As \cref{fig:init-margin-vs-branches} shows, there is a strong correlation between the half-range of the initial SHAP bounds and the number of branches required for reaching~\(10\%\) HR tight bounds.
Therefore, the half-range of the initial bounds can be used to predict the time required until \OurAlgo{} computes tight bounds.

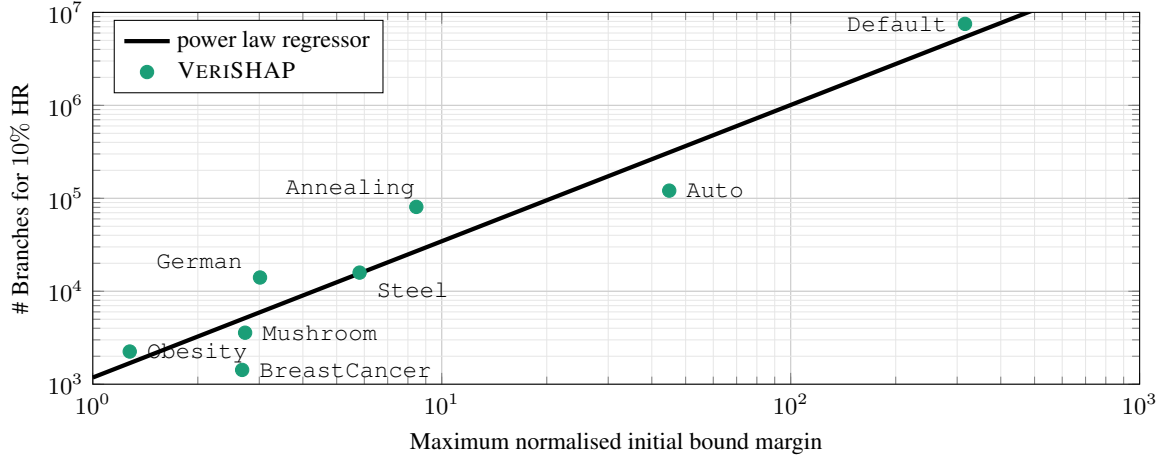
\begin{figure}
  \centering
  \tikzexternaldisable%
  \begin{tikzpicture}
    \begin{axis}[
        height=6.5cm, width=.9\linewidth,
        xmode=log, ymode=log,
        xlabel={Maximum normalised initial bound margin},
        ylabel={\# Branches for $10\%$ HR},
        xmin=1, xmax=1000,
        ymin=1e3, ymax=1e7,
        grid=both,
        grid style={very thin, gray!20},
        major grid style={thin, gray!40},
        label style={font=\footnotesize},
        tick label style={font=\footnotesize},
        legend style={font=\footnotesize, at={(0.02,0.98)}, anchor=north west},
        legend cell align=left,
        nodes near coords align={anchor=west},
        every node near coord/.append style={
          font=\scriptsize, color=gray!70!black,
          xshift=4pt, yshift=-1pt,
        },
        clip mode=individual,
    ]
      \addplot[black, ultra thick, domain=1:1000, samples=2] {1180 * x^1.466};
      \addlegendentry{power law regressor}

      \addplot[
          only marks,
          mark=*, mark size=2.5pt,
          color=CatA,
          mark options={solid, fill=CatA, draw=CatA},
        ] table [col sep=comma, x=margin, y=branches] {
          margin,branches
          8.456598850073956,80808.0
          44.874502400816226,120952.5
          2.6794544899195927,1425.0
          315.973554394813,7524352.0
          3.0148823336262174,14064.0
          2.73247268680048,3582.5
          1.276133438306599,2247.0
          5.816216854110825,15873.5
        };
    \addlegendentry{\OurAlgo{}}
    
    \node[anchor=south east, font=\footnotesize, color=black, xshift=3pt]
      at (axis cs:8.456598850073956,80808.0) {\AnnealingDataset{}};
    \node[anchor=west, font=\footnotesize, color=black, xshift=3pt]
      at (axis cs:44.874502400816226,120952.5) {\AutomobileDataset{}};
    \node[anchor=west, font=\footnotesize, color=black, xshift=3pt]
      at (axis cs:2.6794544899195927,1425.0) {\BreastCancerDataset{}};
    \node[anchor=east, font=\footnotesize, color=black, xshift=-3pt]
      at (axis cs:315.973554394813,7524352.0) {\DefaultDataset{}};
    \node[anchor=south east, font=\footnotesize, color=black, xshift=-3pt]
      at (axis cs:3.0148823336262174,14064.0) {\GermanDataset{}};
    \node[anchor=west, font=\footnotesize, color=black, xshift=3pt]
      at (axis cs:2.73247268680048,3582.5) {\MushroomDataset{}};
    \node[anchor=west, font=\footnotesize, color=black, xshift=3pt]
      at (axis cs:1.276133438306599,2247.0) {\ObesityDataset{}};
    \node[anchor=north west, font=\footnotesize, color=black, xshift=3pt]
      at (axis cs:5.816216854110825,15873.5) {\SteelDataset{}};
    \end{axis}
  \end{tikzpicture}
  \caption{%
    \textbf{Initial Bounds Half-Range vs.\ Number of Branchs.}
    For each tabular network, we plot the largest initial SHAP bounds half-range (HR)~\(\max_{i \in \upto{n}}(\ub{\shapval}^{(1)} - \lb{\shapval}^{(1)}) / 2\) against the number of branches required for \OurAlgo{} to compute SHAP bounds with a maximum half-range of at-most~\(10\%\) of the network output to attribute ($10\%$ HR).
    We normalise the initial bounds margin by the network output to attribute to enable the comparison across datasets.
    A linear regression in log-log space (power law regression) achieves a Pearson correlation coefficient~\(r\) of~\(0.956\), demonstrating a strong correlation between the initial bounds half-range and the number of branches required for computing tight bounds.
  }\label{fig:init-margin-vs-branches}
\end{figure}


\end{document}